\documentclass[nohyperref]{article}

\immediate\write18{cp `icml2022.bst` .}

\usepackage{microtype}
\usepackage{graphicx}
\usepackage{subfigure}
\usepackage{booktabs} 

\usepackage{hyperref}

\usepackage[accepted]{icml2022}

\usepackage{amsmath}
\usepackage{amssymb}
\usepackage{mathtools}
\usepackage{amsthm}

\usepackage[capitalize,noabbrev]{cleveref}

\theoremstyle{plain}
\newtheorem{theorem}{Theorem}[section]

\newtheorem{lemma}[theorem]{Lemma}

\theoremstyle{definition}
\newtheorem{definition}[theorem]{Definition}

\theoremstyle{remark}
\newtheorem{remark}[theorem]{Remark}

\newcommand{\systemnamesecond}{Frugal\textsc{MCT}}
\newcommand{\systemname}{Frugal\textsc{ML}}

\newcommand{\james}[1]{{\color{blue} {\bf James:} #1}}

\newcommand{\lingjiao}[1]{{\color{orange} {\bf Lingjiao:} #1}}

\usepackage{wrapfig}
\usepackage{bbm}
\usepackage{hyperref}       
\usepackage{url}            
\usepackage{booktabs}       
\usepackage{amsfonts}       
\usepackage{microtype}      
\usepackage{microtype}
\usepackage{graphicx}
\usepackage{booktabs} 
\usepackage{amsthm,amsmath,amssymb}
\usepackage{float,url,amsfonts,alltt}
\usepackage{mathtools,rotating}
\usepackage{ifpdf,fancyvrb}
\usepackage{enumitem}

\usepackage{microtype}
\usepackage{graphicx}
\usepackage{booktabs} 
\usepackage{hyperref}

\usepackage{graphicx,xspace,verbatim,comment}
\usepackage{hyperref,array,color,balance,multirow}
\usepackage{balance,float,url,amsfonts,alltt}
\usepackage{mathtools,rotating,amsmath,amssymb}
\usepackage{color,ifpdf,fancyvrb,array}
\usepackage{etoolbox,listings}
\usepackage{bigstrut,morefloats}

\usepackage{pbox}

\usepackage[boxruled,algo2e,linesnumbered]{algorithm2e}
\DeclarePairedDelimiterX{\inp}[2]{\langle}{\rangle}{#1, #2}

\newcommand{\eat}[1]{}

\newcommand{\Exp}{\mathbb{E}}

\newcommand{\R}{\mathbb{R}}
\numberwithin{equation}{section}

\usepackage{amsmath}
\usepackage{accents}
\newlength{\dhatheight}

\usepackage[T1]{fontenc}
\usepackage{upgreek}

\usepackage[greek,english]{babel}
\usepackage{diagbox}

\usepackage{makecell}

\usepackage[textsize=tiny]{todonotes}

\icmltitlerunning{Efficient Online ML API Selection for Multi-Label Classification Tasks}

\begin{document}

\twocolumn[
\icmltitle{Efficient Online ML API Selection for Multi-Label Classification Tasks}



\icmlsetsymbol{equal}{*}

\begin{icmlauthorlist}
\icmlauthor{Lingjiao Chen}{yyy}
\icmlauthor{Matei Zaharia}{yyy}
\icmlauthor{James Zou}{yyy,comp}
\end{icmlauthorlist}

\icmlaffiliation{yyy}{Department of Computer Sciences, Stanford University, Stanford, USA}
\icmlaffiliation{comp}{Department of Biomedical Data Science, Stanford University, Stanford, USA}

\icmlcorrespondingauthor{Lingjiao Chen}{lingjiao@stanford.edu}

\icmlkeywords{Machine Learning, ICML}

\vskip 0.3in
]



\printAffiliationsAndNotice{}  

\begin{abstract}
Multi-label classification tasks such as OCR and multi-object recognition are a major focus of the growing machine learning as a service industry.
While many multi-label APIs are available, it is challenging for users to decide which API to use for their own data and budget, due to the heterogeneity in their prices and performance.    
Recent work has shown how to efficiently select and combine single-label  APIs to optimize performance and cost. 
However, its computation cost is exponential in the number of labels, and is not suitable for settings like OCR.
In this work, we propose \systemnamesecond{}, a principled framework that adaptively selects the APIs to use for different data in an online fashion while respecting the user's budget.
It allows combining ML APIs' predictions for any single data point, and selects the best combination based on an accuracy estimator. 
We run systematic experiments using ML APIs from Google, Microsoft, Amazon, IBM, Tencent, and other providers for tasks including multi-label image classification, scene text recognition and named entity recognition. 
 Across these tasks, \systemnamesecond{} can achieve over 90\% cost reduction while matching the accuracy of the best single API, or up to 8\% better accuracy while matching the best API's cost.
\end{abstract}

\section{Introduction}\label{Sec:SFAME:Intro}


Many machine learning users are starting to adopt machine learning as a service (MLaaS) APIs to obtain high-quality predictions.
One of the most common tasks these APIs target is multi-label classification.
For example, one can use Google's computer vision API~\cite{GoogleAPI}  to tag an image with a wide range of possible labels for \$0.0015, or Microsoft's  API~\cite{MicrosoftAPI} for \$0.0010. 
Another example is to extract all text strings from an image for \$0.005 via iFLYTEK's API~\cite{IflytekAPI} or \$0.021 via Tencent's API~\cite{TencentAPI}.
In practice, these APIs also provide different performance on different types of input data (e.g., English vs Chinese text).
The heterogeneity in APIs' performance and prices makes it hard for users to decide which API, or combination of APIs, to use for their own datasets and budgets. 

Recent work~\cite{FrugalML2020} proposed FrugalML, an algorithmic framework that adaptively decides which APIs to call for a data point to optimize accuracy and cost.
Their approach learns a fast decision rule for each possible output label that can significantly improve cost-performance over the individual APIs. However, FrugalML requires a large amount of training data 
and involves solving a non-convex optimization  problem with complexity exponential in the number of distinct labels.
This prevents it from being used for tasks with large number of  labels, such as multi-label classification.  
Furthermore, FrugalML ignores correlation between different APIs' predictions,  potentially limiting its accuracy.
For example, APIs A and B may output \textit{\{person, car\}} and  \textit{\{car, bike\}} separately for an image whose true keywords are \textit{\{person, car, bike\}}.
FrugalML would select one of the two label sets, but combining them  results in the true label set and thus higher accuracy.
Thus, this paper aims to solve these significant limitations and address the question: \textit{how do we design efficient ML API selection strategies for multi-label classification tasks to maximize accuracy within a budget?} 
\begin{figure*}[t]
	\centering
	\vspace{-1mm}
	\includegraphics[width=1.0\linewidth]{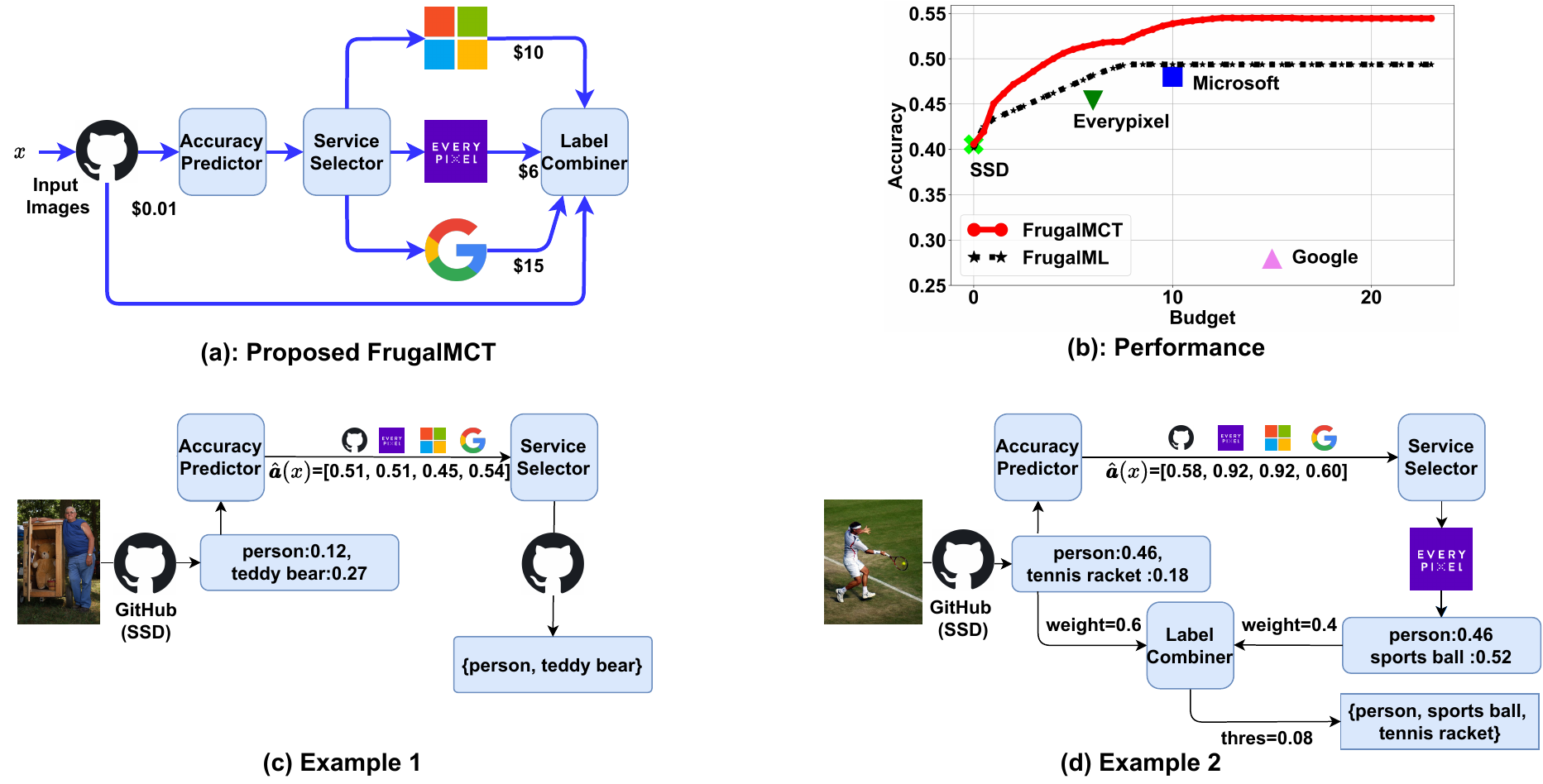}
	\vspace{-4mm}
	\caption{Demonstration of \systemnamesecond{}. \textbf{(a)}: \systemnamesecond{} workflow. \textbf{(b)}: Performance of \systemnamesecond{} on COCO, a multi-label image dataset, using real commercial ML APIs. \textbf{(c), (d)}: Examples of \systemnamesecond{}'s behavior on different inputs. In \textbf{(c)}, \systemnamesecond{} estimates that the accuracy of a cheap open source SSD model from GitHub is high, and thus directly returns its predictions. In \textbf{(d)}, \systemnamesecond{} estimates that combining SSD's results with the Everypixel API has a much higher estimated accuracy, and thus it invokes EveryPixel and combines its results with SSD's results.}
	\label{fig:SFAME:Example}
\end{figure*} 

We propose \systemnamesecond{}, a principled framework that learns the relative strengths of different combinations of multi-label classification APIs and efficiently selects the optimal combinations of APIs to call for different data and budget constraints. 
As shown in Fig.~\ref{fig:SFAME:Example} (a), \systemnamesecond{}  directly estimates the accuracy of each API combination on a particular input based on the features and predicted labels of that input. Then it uses a fast service selector based on the estimated accuracy to balance accuracy and budget.
For example, we might first call API A on an input. 
If A returns \emph{person} and \emph{teddy bear} 
and the accuracy predictor gives  relatively high estimated accuracy (Fig.~\ref{fig:SFAME:Example} (c)), then we stop and report \textit{\{person, teddy bear\}} as the label set.
If A returns \textit{person} and \textit{tennis racket}, and we predict that combining it with API B's output gives a much higher accuracy, then we invoke API B and combine their prediction to obtain \textit{\{person, sports ball, tennis racket\}} (Fig.~\ref{fig:SFAME:Example} (d)).

\paragraph{Contributions.} 
\systemnamesecond{} is an end-to-end approach that integrates the selection of APIs and the combination of their outputs for individual user queries. It leverages our key new finding that \emph{current commercial APIs have complementary strengths and weaknesses, and that we can reliably predict which APIs are likely to work well for a new query based on easy-to-generate metadata about its input.} Based on this API accuracy predictor,  
\systemnamesecond{} then leverages an efficient online algorithm to determine which combination of APIs to call for different user queries. 
We show that the online algorithm enjoys an accuracy provably close to the offline method as well as a small computational cost.
All components in \systemnamesecond{} are trainable, making it easy to customize for different applications.
To our knowledge, \systemnamesecond{} is the first work on how to effectively select and combine multi-label ML APIs. 


Empirically, \systemnamesecond{} produces substantially better prediction performance than  individual APIs and than  FrugalML adapted for multi-label tasks (Fig.~\ref{fig:SFAME:Example} (b)).
Extensive experiments with real commercial APIs on several tasks, including multi-label image classifications, scene text recognition, and named entity recognition, show that \systemnamesecond{} typically provides over 60\% (as high as  98\%) cost reduction when aiming to match the best commercial API's performance.
Also, when targeting the same cost as the best commercial API, \systemnamesecond{} can improve performance up to 8\%.
As a dataset contribution, we have also released~\footnote{\url{https://github.com/lchen001/FrugalMCT}} our dataset of 295,212 samples annotated by commercial multi-label APIs as the largest dataset and resource for studying multi-label ML prediction APIs.

\eat{

\textbf{Contributions.} Our main contributions are:
\begin{enumerate}
    \item We propose \systemnamesecond{}, a principled framework that adaptively exploits ML APIs for multi-label classification to optimize accuracy within a budget constraint. \systemnamesecond{} is  efficient and has provably  performance for both batch and streaming data.  This is the first work on how to select and combine multi-label ML APIs. 
    
    \item We evaluate  \systemnamesecond{}{}  using real-world APIs from diverse providers (e.g., Google, Microsoft, Amazon, and Tencent) for tasks including multi-label image classification, scene text recognition, and named entity recognition.
    \systemnamesecond{} can significantly reduce costs and improve accuracy over the best individual APIs.
     
     \item We release our dataset of 295,212 samples annotated by commercial multi-label APIs as the largest dataset and resource for studying ML APIs.
\end{enumerate}
}

\section{Related Work}
\textbf{MLaaS:} With the growing importance and adoption of MLaaS APIs \cite{AmazonAPI, TencentAPI, GoogleAPI, IBMNLPAPI, MicrosoftAPI}, existing research has largely focused on evaluating individual APIs for their performance  \cite{MLasS_EmpiricalAnalysis2017},  robustness  \cite{MLasS_GoogleNotRobust_2017}, biases \cite{RacialDisparityASR2020}, performance estimation~\cite{chen2021MASA}, pricing~\cite{chen2019nimbus},and applications \cite{ pmlr-v81-buolamwini18a, MLasS_GoogleDigitalMedia_2019, MLasS_Google_Microsoft_Blind18}.
Recent work on \systemname{} \cite{FrugalML2020}  studies API calling strategies for single label classification. 
While their approach's computational complexity is exponential in the number of labels, \systemnamesecond{}'s complexity does not depend on the number of labels, making it suitable for multi-label prediction APIs. 
In addition, \systemname{} selects only one API per user query, while  \systemnamesecond{} considers the combination of multiple APIs' output for each input data. 
This improves the overall accuracy (as shown in Sec \ref{Sec:SFAME:Experiment}), but also creates unique optimization challenges that we solve.

\textbf{Ensembles for multi-label classification:}
Ensemble learning is a natural approach to combine different predictors' output. Several ensemble methods have been developed, 
such as using pruned sets \cite{Ensemble_Multilabel_PrunedSet2008}, classifier chains \cite{Ensemble_Multilabel_chainmethod2011}, and random subsets \cite{Ensemble_Multilabel_2007}, with applications in image annotations~\cite{ensembleimagevideotag2011}, document classification~\cite{Ensemble_TextCategory2017}, and speech categorization~\cite{Ensemble_Multilabel_hatespeech2019}. 
 \citet{Ensemble_Multilabel_Survey2020} provide a detailed survey of this area.
Almost all of these ensemble methods require joint training of the base classifiers, but MLaaS APIs are black box to the users.
Also, while  ensemble methods focus only on improving accuracy, \systemnamesecond{} explicitly considers the cost of each API and enforces a budget constraint.

\textbf{Model cascades:} A series of works 
\cite{Viola01robustreal_time,ModelCascade_NIPS2001,FacePointCascade13,  PedestrianDetectionCascadeCai15, ModelCascade_Linear11, CascadeXu14,Multilabel_Cascade_Segmentation2018,Multilabel_cascade_medical2018,Multilabel_Cascade_Segmentation2018} explores cascades (a sequence of models) to balance the quality and runtime of inference.
Model cascades use a \textit{single} predicted quality score to avoid calling computationally expensive models,
but \systemnamesecond{}' strategies utilize both \emph{quality scores and predicted label sets} to select an  expensive add-on service. While cascades do not explicitly specify inference speed,  \systemnamesecond{} allows users to explicitly incorporate different budget requirements. 
Designing such strategies requires solving a significantly harder optimization problem, e.g., choosing how to divide the available budget between classes (\S\ref{Sec:SFAME:Theory}), but also improves performance substantially over using the quality score alone (\S\ref{Sec:SFAME:Experiment}).

\textbf{AutoML for multi-label classification:} AutoML~\cite{AutoML_earlywork_2013} automates the customization of ML pipelines, including the selection, combination, and parametrization of the learning algorithms. 
There is a rich literature of AutoML techniques for standard single label tasks, and fewer methods on multi-label predictions~\cite{AutoML_Multilabel_Survey2021} (e.g. genetic algorithms~\cite{AutoMLMLC_Genetic_2017} and a neural network-based search scheme ~\cite{AutoML_MLC_NN_2019}).
We refer interested readers to a recent survey~\cite{AutoML_Multilabel_Survey2021} for more details.
Applying AutoML to use multiple ML APIs is underexplored, and \systemnamesecond{} can be viewed as the first AutoML approch designed for automating the selection of multiple mutlti-label ML APIs. 
While most AutoML systems exclusively focus on prediction performance,  \systemnamesecond{}  optimizes accuracy and cost jointly, which is desirable for cost-sensitive API users. 

\textbf{Multiple choice knapsack and integer programming:} Many resource allocation problems can be modeled as 
multiple choice knapsack problem (MCKP)~\cite{knapsackbook1993}, such as keyword bidding \cite{knapsack_application_keywordbid_2008} and quality of service control \cite{Knapsack_application_qos_1999}.
While NP-hard \cite{multipleknapsack1979}, various approximations have been proposed for MCKP, such as branch and bound \cite{knapsackbook1993}, convex hull relaxation \cite{KCknapsackconvexhull2006} and bi-objective transformation \cite{approximateMultichoiceknapsack2018}.
Inherently an integer linear programming (ILP) problem, MCKP can also be tackled by  ILP solvers, motivated by online adwords searching \cite{OnlineAdwordRP}, resource allocation \cite{OnlineRecourseAllocation2019} and general linear programming \cite{FastOnlineLP2020}.  
The service selector of \systemnamesecond{} can be viewed as a MCKP with the same  item cost vector per item group, which we leverage to obtain a customized fast and online solver. Our goal is to not develop novel MCKP solver, but to efficiently adapt ILP methods as a subroutine of our end-to-end \systemnamesecond{} to tackle a practical new application. 

\eat{\textbf{Online resource allocations:}
Online resource allocations have been studied} 

\section{Preliminaries}\label{Sec:SFAME:Preli}

\paragraph{Notation.}
We denote matrices and vectors in bold, and scalars, sets,  and functions in standard script. 
Given a matrix $\mathbf{A} \in \mathbb{R}^{n\times m}$, we let $\mathbf{A}_{i,j}$ denote its entry at location $(i,j)$.
$\mathbbm{1}(\cdot)$ represents the indicator function.


\paragraph{Multi-label classification Tasks.}
Throughout this paper, we focus on multi-label classification  tasks: assigning a label set $Y \subseteq \mathcal{Y}$ to any data point $x \in \mathcal{X}$. 
In contrast to basic supervised learning, in multi-label
learning each data point is  associated with a set of labels instead of a single label.
Many MLaaS APIs target such tasks.
Consider, for example, image tagging, where $\mathcal{X}$ is a set of images and $\mathcal{Y}$ is the set of all tags.
Example label sets could be \{\textit{person}, \textit{car}\} or \{\textit{bag}, \textit{train}, \textit{sky}\}. 

\paragraph{MLaaS Market.}
Consider a MLaaS market consisting of $K$ different ML services for some multi-label  tasks.
For a data point $x$, the $k$th service returns to the user a set of labels with their quality scores, denoted by $Y_k(x)\subseteq \mathcal{Y} \times [0,1]$.
For example, one API for multi-label image classification might produce $Y_k(x)=\{(\textit{person}, 0.8), (\textit{car},0.7)\}$, indicating the label \textit{person} with confidence $0.8$ and \textit{car} with confidence $0.7$.
Let the vector $\pmb{c} \in \R^{K}$
denote the unit cost of all services.
E.g., $\mathbf{c}_k=0.01$ means that  users need to pay \$$0.01$ every time they call the $k$th service. 
\begin{figure*}[t]
	\centering
	\includegraphics[width=1.0\linewidth]{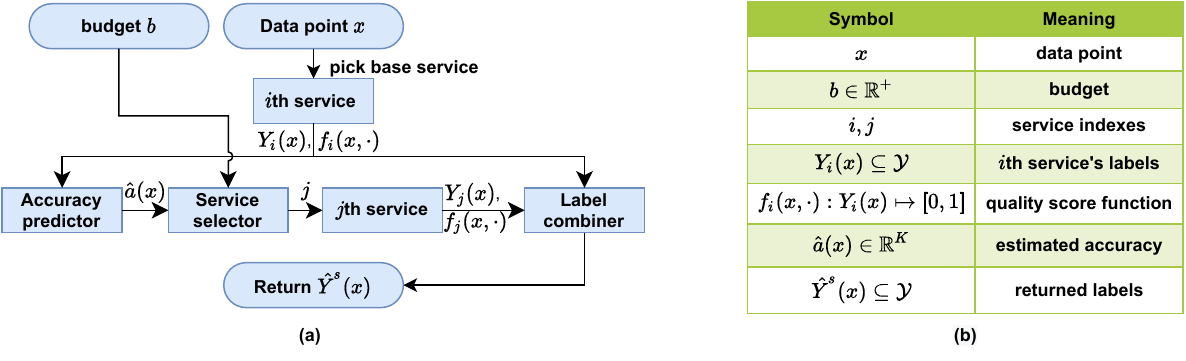}
	\vspace{-8mm}
	\caption{Overview of \systemnamesecond{}. \textbf{(a)} shows how it works: Given a data point, \systemnamesecond{} first invokes a base service. An accuracy predictor then estimates the performance of different APIs. Next, an add-on service is selected based on the predicted accuracy and budget. Finally, the add-on and base services' predictions are combined to return \systemnamesecond{}'s prediction. \textbf{(b)} lists notation. \eat{\james{The font here for symbols look thicker than usual.} \lingjiao{The bold operation on $\hat{a}$ has been removed. Does it look better now?}}
	}
	\label{fig:SFAME:Framework}
\end{figure*}

\section{\systemnamesecond{} Framework}\label{Sec:SFAME:Theory}

In this section, we present \systemnamesecond{}, a framework to adaptively select ML APIs for  multi-label classification tasks within a budget. All proofs are left to the  appendix.
We generalize the scheme in Figure  \ref{fig:SFAME:Example} (a) to $K$ ML services.
As shown in Figure \ref{fig:SFAME:Framework}, \systemnamesecond{} contains three main components: an accuracy estimator, a service selector, and a label combiner.


 
Given a data point $x$, it first calls some base service, denoted by $\textit{base}$, which is one of the $K$ APIs,  and obtains $Y_{\textit{base}}(x)$. Often, \textit{base} is a cheap or free service, such as an inexpensive open source model; we discuss how to choose \textit{base} out of multiple options in Section \ref{sec:SFAME:theory:baseandcombiner}.  Next, an accuracy predictor produces a vector $\hat{\pmb{a}}(x) \in [0,1]^{K}$, whose $k$th value estimates the accuracy of the label set produced by the label combiner using base's and $k$th API's outputs.
The service selector $s(\cdot): \mathcal{X} \mapsto [K]$ then decides if and which \textit{add-on service} needs to be invoked.
Finally, a label combiner generates a label set by combining the predictions from  the base and add-on APIs. 
Take Figure \ref{fig:SFAME:Example} (d) as an example. 
The image is first passed to the GitHub model, which produces \{(\textit{person}, 0.46),(\textit{tennis racket},0.18)\}, by which the accuracy predictor predicts the accuracy of the label set generated by combining each API's output with GitHub model's. 
The service selector then decides to further invoke   Everypixel, which gives \{(\textit{person}, 0.46), (\textit{sports ball}, 0.52)\}. 
Finally, the label combiner uses both APIs' output for the final prediction.

\systemnamesecond{} allows users to customize the accuracy predictor and the label combiner, depending on the  applications.
For example, for the image tagging problem, one might use image features (e.g., brightness and contrast) to build the accuracy predictor, while word embeddings can be more useful for named entity recognition. In the following sections, we explain the key  of accuracy predictor, API selector and the label combiner in more detail.

\subsection{Accuracy prediction} 
The accuracy predictor $\hat{\pmb{a}}(\cdot)$ can be obtained by two steps. The first step is to generate a feature vector for every data point in the training dataset $\mathbb{X}^{Tr} \triangleq \{x_1^{Tr}, x_2^{Tr}, \cdots,  x^{Tr}_{N^{Tr}}\}$. Generally the feature vector can be any embedding of the data point $x$ and base service prediction $Y_{\textit{base}}(x)$.
In this paper we adopt a simple approach: if the label set $\mathcal{Y}$ is bounded, a $|\mathcal{Y}|$ dimensional vector is generated using one hot encoding on $Y_{\textit{base}}(x)$. For example, given $\mathcal{Y}=\{\textit{person}, \textit{car}, \textit{bike}\}$ and $Y_{\textit{base}}(x)=\{(\textit{person}, 0.8), (\textit{car}, 0.7)\}$, the generated feature vector is $[0.8,0.7,0]$.
For unbounded $\mathcal{Y}$, word embedding is used to generate a vector for every predicted label, and the sum of them (weighted by their quality values) becomes the corresponding feature vector.

The next step is to train the accuracy predictor. 
For each $x_n^{Tr} \in \mathbb{X}^{Tr}$, as its true label sets and prediction from each API are available, we can construct its true accuracy vector $\pmb a(x_n^{Tr}) \in [0,1]^K$, whose $k$th element is the accuracy of the label produced by the label combiner using base and $k$th service predictions. Then we can train some regressor (e.g., random forest) to map the feature vector to the accuracy vector.  We use standard multi-label accuracy\footnote{$\frac{\|Y \cap Y'\|}{ \|Y \cup Y' \|}$ where $Y$/$Y'$ is the  true/predicted label set. } \cite{Review_Ensemble_Zhang2014} as a concrete metric. 
 \systemnamesecond{} can as easily use another metric  such as F1-score, precision or subset accuracy.  
\vspace{-2mm}
\subsection{The API selection problem}
A core subroutine of \systemnamesecond{} is  the API selector $s$: given a budget $b$ and the estimated accuracy $\hat{\pmb a}(x)$, which service should be invoked?
Let $\mathbb{X}\triangleq\{x_1,x_2,\cdots, x_N\}$ be the entire unlabeled dataset to be classified, and
$S\triangleq\{1,2,\cdots, K\}^{\mathbb{X}}$ be the set of all functions mapping each data point in $\mathbb{X}$ to an API. Let \textit{base} be the index of the base service.
 For any $s\in S$, $s(x)=\textit{base}$ implies no add-on API is needed, and $s(x)=k\not=\textit{base}$ implies $k$th API is invoked.
Our goal is to find some  $s \in S$  to maximize the estimated accuracy while satisfying the budget constraint,  formally stated as below. 
\begin{definition}
Let $\pmb Z^*_{n,k}$ be the optimal solution to the budget aware API selection problem 
\begin{equation}\label{prob:SFAME:optimaldefinition}
    \begin{split}
        \max_{\pmb Z \in \R^{N\times K}:} &  \frac{1}{N} \sum_{n=1}^{N} \sum_{k=1}^{K} \pmb{Z}_{n,k} \pmb {\hat{a}}_{k}(x_n)\\
        s.t. \textit{ } & \frac{1}{N} \sum_{n=1}^{N} \sum_{k=1,k\not=\textit{base}}^{K} \pmb Z_{n,k} {\pmb c}_{k} + \pmb c_{\textit{base}} \leq  b;\\  
        & \sum_{k=1}^{K} \pmb{Z}_{n,k} = 1,  \forall n;      \pmb{Z}_{n,k} \in \{0,1\}, \forall n, k
    \end{split}
\end{equation}
Then the optimal \systemnamesecond{} strategy is given by $s^*(x_n) \triangleq \arg \max_k \pmb Z^*_{n,k}$.

\end{definition}
Here, the objective quantifies the average accuracy, the first constraint models the budget requirement, and the last two constraints enforces only one add-on API is picked for each data point.  
Base service is needed for every data point and thus its cost
$\pmb c_{base}$ appears for every $n$ in the budget constraint.
Note that Problem \ref{prob:SFAME:optimaldefinition} is a MCKP (and thus integer linear program) and  NP-hard in general.

\subsection{An online algorithm for \systemnamesecond{}}
In many time-sensitive applications, the input data $x_n$ (as well as the accuracy vector $\hat{\pmb a}(x_n)$) comes sequentially, and the API needs to be selected before observing the future data.
The selection process also needs to be fast.

To tackle this challenge, we present an efficient online algorithm, which requires $O(K)$ computations per round and gives a provably near-optimal solution.
The key idea is to explicitly balance between accuracy and cost  at every iteration.
Specifically, for a given data point $x_n$ and $p\in \R$, let us define a strategy 
${s}^p(x_n) \triangleq {\arg \max_k \hat{\pmb a}_k(x_n) - p  {\pmb c}_k}\mathbbm{1}_{k\not=\textit{base}}$ and break ties by picking $k$ with smallest cost.
Here, $p$ is a parameter to balance between accuracy $\hat{\pmb a}(x_n)$ and cost ${\pmb c}$. When $p=0$, ${s}^{p}(x_n)$ selects  the API with highest estimated accuracy. When $p$ is large enough 
$ {s}^p(x_n)$ enforces to pick the base API. 
In fact, larger value of $p$ implies more weights on cost and smaller $p$ favors more the accuracy.
Let $r(s)\triangleq \frac{1}{N}\sum_{n=1}^{N} \hat{\pmb a}_{s(x_n)}(x_n)$ denote the average accuracy achieved by a strategy $s$. 
We can show, interestingly, an appropriate choice of $p$ leads to small average accuracy loss. 
\begin{theorem}\label{thm:SFAME:dualbound}
Assume the probability density of $\hat{\pmb a}(x)$ is a continuous function on $[0,1]^K$.  Then with probability 1,
 there exists $p^*$ such that $ {s}^{p^*}$ satisfies budget constraint, and $ {r}({{s}^{p^*}}) \geq  {r}(s^*) - \frac{1}{N}$.
\end{theorem}

In words, ${s}^{p^*}(x_n)$ gives a solution to the API selection problem with accuracy loss at most $\frac{1}{N}$.
In practice, $\hat{\pmb a}(x)$ is continuous for standard ML models of accuracy predictors (e.g., logistic regressors) and thus the assumption holds.  
In addition, it is computationally efficient: at iteration $n$, it only requires computing $ \hat{\pmb a}_k(x_n) - p  \pmb c_k \mathbbm{1}_{k\not=\textit{base}}$ for $k=1,2,\cdots, K$, which takes only $O(K)$ computations.

The remaining question is  how to obtain $p^*$. 
As we cannot see the future data to compute $p^*$, a natural idea is to estimate it using the training dataset. More precisely, given the training dataset $\{x^{Tr}_1, x^{Tr}_2,\cdots, x^{Tr}_{N^{Tr}}\}$, 
let $\hat{p}, \hat{\pmb{q}}$ be the optimal solution to the following problem 
\begin{equation}\label{prob:SFAME:APISelectionTrainingDual}
    \begin{split}
        \min_{p, \pmb q } &  (1-\delta) (b-\pmb c_{\textit{base}}) p +   \sum_{n=1}^{N^{\mathit{Tr}}} \pmb q_n,\\
        \textit{ } s.t. \textit{ } &  \frac{ {\pmb c}_k \cdot \mathbbm{1}_{k\not=\textit{base}} \cdot p}{N^{\mathit{Tr}}} + \pmb q_n \geq  \frac{\hat{\pmb a}_k(x_n^{\mathit{Tr}})}{N^{\mathit{Tr}}}, \forall n, k         \\
&        p\geq 0, \pmb q\in \R^{N^{Tr}}, \pmb q\geq 0 \\
    \end{split}
\end{equation}
where $\delta \in (0,1)$ is a small buffer to ensure that we don't exceed the budget (in practice we set $\delta \leq 0.01$).
Technically, Problem \ref{prob:SFAME:APISelectionTrainingDual} is the dual problem to the linear programming by relaxing the integer constraint in  Problem \ref{prob:SFAME:optimaldefinition} on the training dataset with budget $(1-\delta)b$, and $\hat{p}$ corresponds to the near-optimal strategy for the training dataset.
If the training and testing datasets are from the same distribution, then a small $\delta$ can  ensure with high probability, $\hat{p}$ is slightly less than $p^*$ and thus $s^{\hat{p}}$ satisfies the budget constraint.
Given $\hat{p}$, one can use ${s}^{\hat{p}}$ to select the APIs in an online fashion.
The details are given in Algorithm \ref{Alg:SFAME:OnlineAlg}. 

\begin{algorithm}
\caption{\systemnamesecond{}{} Online API Selection Algorithm.}
	\label{Alg:SFAME:OnlineAlg}
	\SetKwInOut{Input}{Input}
	\SetKwInOut{Output}{Output}
	\Input{${\pmb c}, b, \{x_1^{Tr},x_2^{Tr},\cdots, x_{N^{Tr}}^{Tr}\}, \{x_1, x_2,\cdots, x_N\}$}
	\Output{\systemnamesecond{} online API selector $s^o(\cdot)$}  
	
  \begin{algorithmic}[1]
  	\STATE Compute $\hat{p}$ by solving Problem \ref{prob:SFAME:APISelectionTrainingDual} and set $b^r=N (b-\pmb c_{base})$.
 
\STATE At iteration $n=1,2,\cdots, N$:

\STATE \textit{  } $s^o(x_n) =  \begin{cases}  { s}^{\hat{p}}(x_n) &\mbox{if } b^r - {\pmb c}_{{ s}^{\hat{p}}(x_n)} \geq 0 \\
\textit{base} & o/w \end{cases} $

\STATE \textit{  }  $b^r = b^r - {\pmb c}_{{ s}^{\hat{p}}(x_n)} \mathbbm{1}_{{ s}^{\hat{p}}(x_n)\not=\textit{base}} $ 

  \end{algorithmic}
\end{algorithm}

Here, $b^r$ is used to ensure the generated solution is always feasible. 
The following theorem gives the performance guarantee of the online solution.

\begin{theorem}\label{thm:SFAME:onlinebound}
If $\delta = \Theta\left(\sqrt{\frac{\log N/ \epsilon}{N}} + \sqrt{\frac{\log N^{Tr}/ \epsilon}{N^{Tr}}}\right)$ and the probability density of $\hat{\pmb a}(x)$ is a continuous function on $[0,1]^K$, then $s^o$ satisfies the budget constraint, and  with probability at least $1-\epsilon$, 
$  {r}({s}^{o}) \geq r(s^*) - O\left(\sqrt{\frac{\log N^{}/\epsilon }{N^{}}} + \sqrt{\frac{\log N^{Tr}/ \epsilon }{N^{Tr}}} \right)$. 
\end{theorem}
Roughly speaking,  ${s}^{o}$ leads to an accuracy loss at most $O\left(\sqrt{\frac{\log N}{N}}+\sqrt{\frac{\log N^{Tr}}{N^{Tr}}} \right)$ compared to the optimal offline strategy.
For large training and testing datasets, such an accuracy loss is often negligible, which is also verified by our experiments on real world datasets. 



\subsection{Base service selection and label combination}\label{sec:SFAME:theory:baseandcombiner}
Now we describe how the base service is selected and how the label combiner works. 
The base service can be picked by an offline searching process. 
More precisely, for each possible base service, we train a \systemnamesecond{} strategy and evaluate its performance on a validation dataset, and pick the base service corresponding to the highest performance.

The label combiner contains two phases. First, a new label set associated with its quality function is produced. The label set is simply the union of that from the base service and add-on service. The quality score is a weighted sum of the score from both APIs, controlled by  a hyperparameter  $w$.
For example, suppose the base predicts $\{(\textit{person},0.8), (\textit{car}, 0.7)\}$ and the add-on predicts $\{(\textit{car}, 0.5), (\textit{bike},0.4)\}$. Given $w=0.3$,  new confidence for \textit{person} is $0.3\times 0.8=0.24$, for \textit{car} is $0.3 \times 0.7 + 0.7\times 0.5=0.46$, and for \textit{bike} is $0.7 \times 0.4=0.28$.
Thus the combined set is $\{(\textit{person},0.24), (\textit{car},0.46), (\textit{bike},0.28)\}$.
Next, a threshold $\theta$ is applied to remove labels with low confidence.
For example, given $\theta=0.25$, the label $person$ would be removed, and the final predicted label set becomes $\{car, bike\}$.
The parameters $w$ and $\theta$ are global hyperparameters for each dataset, and can be obtained by an efficient searching algorithm to maximize the overall performance. 
The details are left to Appendix \ref{sec:SFAME:techdetails}.  

\section{Experiments}\label{Sec:SFAME:Experiment}
We compare the accuracy and incurred costs of \systemnamesecond{} to that of real world ML services for various tasks.
Our goal is to (i) understand when and why \systemnamesecond{} can reduce cost without hurting accuracy,  (ii) investigate the trade-offs between accuracy and cost achieved by \systemnamesecond{}, and (iii) assess the effect of training data size and accuracy predictors on \systemnamesecond{}'s performance. 

\paragraph{Tasks,  ML Services, and Datasets.} 
We focus on three common ML tasks in different domains: multi-label image classification (\textit{MIC}), scene text recognition   (\textit{STR}), and named entity recognition (\textit{NER}). 
\textit{MIC} aims at obtaining all keywords associated with an image, \textit{STR} seeks to recognize all texts in an image, and \textit{NER} desires to extract all entities in a text paragraph. 
The ML services used for each task and  their prices are summarized in Table \ref{tab:SFAME:MLservice}. 
For each task we use three datasets, summarized in Table \ref{tab:SFAME:DatasetStats}. More details can be found in Appendix \ref{sec:SFAME:experimentdetails}.

\begin{table}[t]
  \centering
  \small
  \caption{ML services used for each task. Price unit: USD/10,000 queries. A publicly available (and thus free) GitHub model is also used per task: a single shot detector (SSD)
	\cite{SSD_MIC_github} pretrained  on Open Images V4 \cite{OpenImagesV4IJCV2020} for \textit{MIC},  a  convolutional recurrent neural network (PP-OCR) \cite{PaddleOCR_github} pretrained on an industrial dataset \cite{PaddleOCRArxiv2020} for \textit{STR}, and a convolutional neural network (spaCy \cite{Spacy_github}) pretrained on OntoNotes \cite{Dataset_OntoNotes} for \textit{NER}. 
	}
    \begin{tabular}{|c||c|c|c|c|}
    \hline
    Task  & ML Service & Price & ML Service & Price \bigstrut\\
    \hline
    \hline    
    \multirow{2}[4]{*}{MIC} & SSD \cite{SSD_MIC_github} & <0.01 & Everypixel  \cite{EverypixelAPI} & 6 \bigstrut\\
\cline{2-5}          & Microsoft \cite{MicrosoftAPI}& 10    & Google \cite{GoogleAPI} & 15 \bigstrut\\
    \hline
    \multirow{2}[4]{*}{STR} & PP-OCR  \cite{PaddleOCR_github} & <0.01  & Google \cite{GoogleAPI} &  15 \bigstrut\\
\cline{2-5}          & iFLYTEK \cite{IflytekAPI} & 50    & Tencent \cite{TencentAPI} &  210 \bigstrut\\
    \hline
    \multirow{2}[4]{*}{NER} & spaCy \cite{Spacy_github} & <0.01 & Amazon \cite{AmazonAPI}& 3 \bigstrut\\
\cline{2-5}          & Google \cite{GoNLPAPI} & 10    & IBM \cite{IBMNLPAPI}  & 30 \bigstrut\\
    \hline
    \end{tabular}
\label{tab:SFAME:MLservice}
\end{table}%

\begin{table}[htbp]
  \centering
  \small
  \caption{Dataset Statistics.}
    \begin{tabular}{|c||c|c|c|c|}
    \hline
    Task  & Dataset & Size  & \# Labels & Dist Labels \bigstrut\\
    \hline
    \hline
    \multirow{3}[6]{*}{MIC} & PASCAL  & 11540 & 16682 & 20 \bigstrut\\
\cline{2-5}          & MIR  & 25000 & 92909 & 24 \bigstrut\\
\cline{2-5}          & COCO  & 123287 & 357662 & 80 \bigstrut\\
    \hline
    \multirow{3}[6]{*}{STR} & MTWI  & 9742  & 867727 & 4404 \bigstrut\\
\cline{2-5}          & ReCTS & 20000 & 555286 & 4134 \bigstrut\\
\cline{2-5}          & LSVT   & 30000 & 1878682 & 4852 \bigstrut\\
    \hline
    \multirow{3}[6]{*}{NER} & CONLL & 10898 & 43968 & 9910 \bigstrut\\
\cline{2-5}          & ZHNER & 16915 & 147164 & 4375 \bigstrut\\
\cline{2-5}          & GMB  & 47830 & 116225 & 14376 \bigstrut\\
    \hline
    \end{tabular}%
  \label{tab:SFAME:DatasetStats}%
\end{table}%

\paragraph{Accuracy Predictors.} Except when explicitly noted, we use a random forest regressor as the accuracy predictor for all the datasets. 
For \textit{MIC} and \textit{STR} datasets, we map each possible label to an index, and create a feature vector whose $k$th element is base service's  quality score for the label corresponding to $k$. If a label is not predicted,  the corresponding value is 0.  
For \textit{NER} datasets, we map each predicted label to a 96-dimensional vector using a word embedding from spaCy \cite{Spacy_github}, and then use the sum weighted by their corresponding quality scores as the feature vector. 
The accuracy predictor is then trained on half of the datasets using the feature vectors generated as above. Interestingly, we found we are able to accurately predict which commercial API is best for each instance using relatively simple features. This makes the approach more broadly applicable. 
We will study the effects of accuracy predictors later in this section. 

\eat{\james{Flesh out how good the accuracy predictor is. We can emphasize that while the method is simple, this is actually an interesting finding that: 1) commercial APIs have very heterogeneous performance and 2) we can accurately predict which commercial API is best for each instance using easy features.}
\lingjiao{Done. Also added a sub-title \textbf{effects of accuracy predictors} for a later discussion.}
}

	\paragraph{Multi-label Image Classification: A Case Study.}
Let us start with multi-label image classification on the COCO dataset~\cite{Dataset_COCO_2014}.
We set budget $b=6$, the price of Everypixel, the cheapest commercial API (except the open source model from GitHub). 
For comparison, we also use the average quality score over all predicted labels  as the confidence score and adapt FrugalML \cite{FrugalML2020} with the same budget ($=6$) as another baseline .

\begin{figure}[t]
	\centering
	\includegraphics[width=1.0\linewidth]{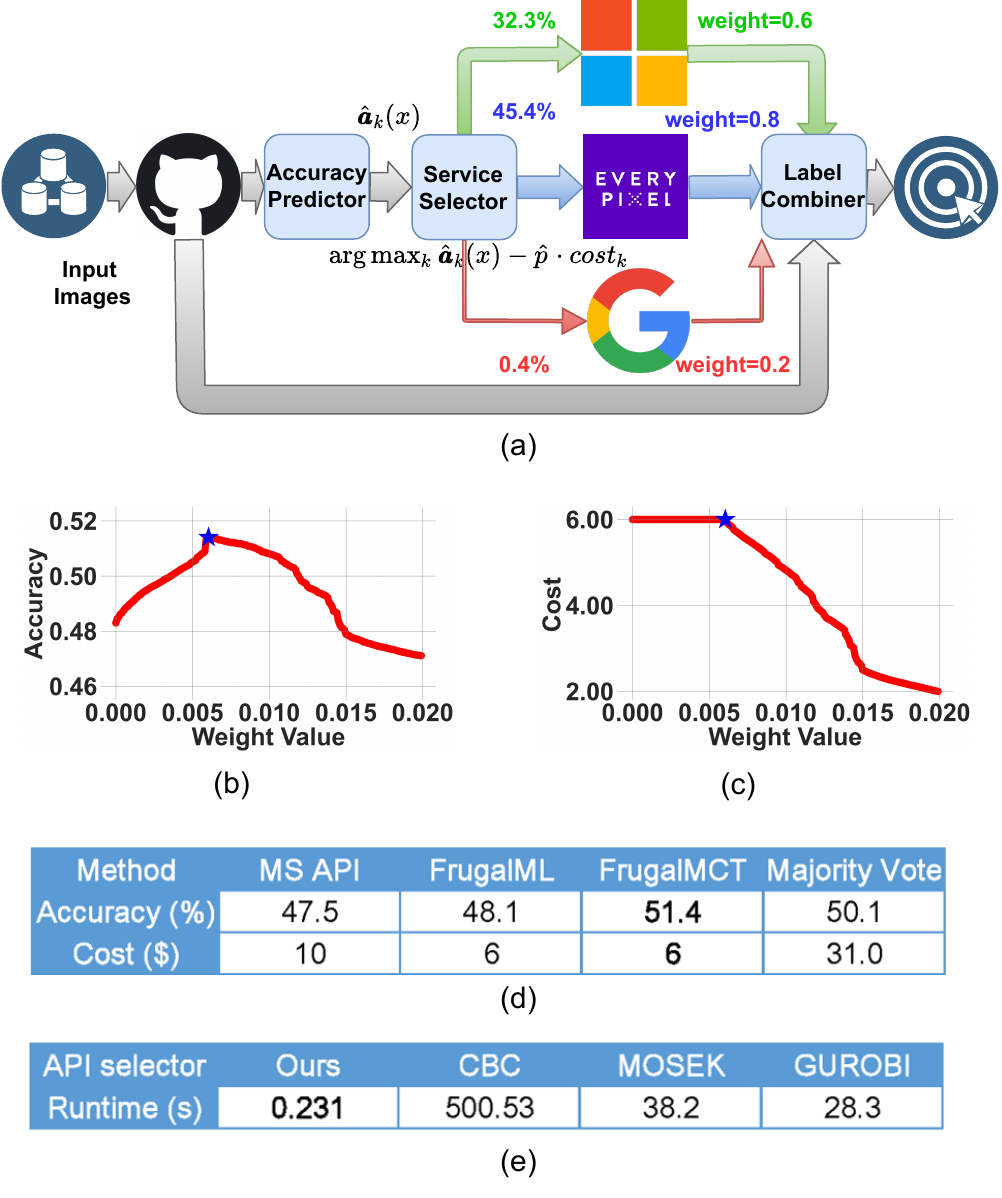}
	\vspace{-4mm}
	\small
	\caption{A \systemnamesecond{} strategy learned  on the dataset COCO. \textbf{(a)} shows that \systemnamesecond{} reduces cost by mostly calling the Everypixel API (45.4\%) or the GitHub API (22.1\%) only. \textbf{(b)} and \textbf{(c)} show how the accuracy and cost vary with weight  $p$. The blue point corresponds to $0.006$, the learned $\hat{p}$. (d) shows the accuracy and cost of \systemnamesecond{}, \systemname{}, Microsoft API, and majority vote. (e) gives the runtime performance of our (online) API selector and three commercial ILP solvers.}
	\label{fig:SFAME:casestudyflow}
	\end{figure}

Figure \ref{fig:SFAME:casestudyflow} demonstrates the learned  \systemnamesecond{} strategy.
As shown in Figure \ref{fig:SFAME:casestudyflow} (a), the learned \systemnamesecond{} reduces the cost by mostly using the Everypixel API (45\%, 6\$) and occasionally calling Microsoft API (32\%, 10\$), and rarely invoking the Google API (0.4\%, 15\$). 
 Note that its performance depends on the threshold value $\hat{p}$.
As shown in Figure  \ref{fig:SFAME:casestudyflow} (b) and (c), for small thresholds, \systemnamesecond{} tends to call the more accurate and expensive APIs.
However, it runs out of budget quickly, and for many data points only base service can be used, leading to low accuracy.
For large thresholds, \systemnamesecond{} tends to call cheaper but less accurate APIs, failing to fully use the budget and thus causing low accuracy too.
The $\hat{p}$ value learned by \systemnamesecond{} (blue point in Figure \ref{fig:SFAME:casestudyflow} (b) and (c)) produces the optimal accuracy given the budget.
\eat{
\begin{table}[t]
\begin{minipage}{.45\linewidth}
    \centering
    \small
    \caption{Accuracy and cost comparison on coco.}
    \label{tab:SFAME:casestudycompare}
    \medskip
    \small
    \begin{tabular}{|c||c|c|c|c|c|}
    \hline
          & MS  & FrugalML & Ours & Maj \bigstrut\\
    \hline
    \hline
    Acc & 0.475 & 0.481 & \textbf{0.514} & 0.501   \bigstrut\\
    \hline
    Cost  & 10    & 6     & 6     & 31.01 \bigstrut\\
    \hline
    \end{tabular}%
\end{minipage}\hfill
\begin{minipage}{.52\linewidth}
    \centering

    \caption{Runtime comparison of FrugalMCT and three ILP solvers. CBC, MOSEK and GUROBI.}
    \label{tab:second_table}

    \medskip
    \small
    \begin{tabular}{|c||c|c|c|c|}
    \hline
          & Ours  & CBC   & MOSEK & GUROBI \bigstrut\\
    \hline
    \hline
    Acc   & 0.514 & 0.514 & 0.514 & 0.511 \bigstrut\\
    \hline
    Time (s) & 0.231 & 500.53 & 38.2  & 28.3 \bigstrut\\
    \hline
    \end{tabular}%
\end{minipage}
\end{table}
}
Figure \ref{fig:SFAME:casestudyflow} (d) shows that \systemnamesecond{}'s accuracy (0.514) is higher than that of the best ML service (MS, 0.475) and majority vote (Maj 0.501), while its cost is much lower.
This is primarily because \systemnamesecond{} learns when the cheaper APIs perform better and call them aptly.
\systemnamesecond{} also outperforms FrugalML by exploiting the label combination. 
This is due to (i) that \systemname{} cannot utilize the label information due to explosion of complexity, and (ii) that the label combiner in \systemnamesecond{} gives higher accuracy than both the base and add-on APIs. 

To understand the efficiency of \systemnamesecond{}'s API selector, we compare it with three ILP solvers,  namely, CBC, MOSEK, and GUROBI.  
CBC~\cite{forrest2005cbc} is an integer linear programming package developed based on cutting and branch. 
MOSEK~\cite{andersen2000mosek} was originally developed for sparse programming and then extended for general mixed integer programming. 
On the other hand, the focus of GUROBI~\cite{bixby2007gurobi} is parallelism optimization in integer programming. 
As shown in Figure \ref{fig:SFAME:casestudyflow} (e), the API selector of \systemnamesecond{} (Alg.~\ref{Alg:SFAME:OnlineAlg}) is several orders of magnitude faster than those commercial ILP solvers. This is beause it leverages the specific structure of Problem \ref{prob:SFAME:optimaldefinition}.
\begin{table}[htbp]
  \centering
  \small
  \caption{End-to-end runtime comparison on COCO.}
    \begin{tabular}{|c||c|c|c|}
    \hline
    Runtime & \systemnamesecond{} & \systemname{} & Majority Vote \bigstrut\\
    \hline
    \hline
    Training & 60s   & 6627s & N/A \bigstrut\\
    \hline
    Inference & 1.25s & 1.24s & 1.92s \bigstrut\\
    \hline
    \end{tabular}%
  \label{tab:SFAME:Runtime}%
\end{table}%

The end-to-end runtime comparison of \systemnamesecond{} with \systemname{} and an ensemble approach (majority vote) is given in Table \ref{tab:SFAME:Runtime} . 
Majority vote does not need training, but its inference time is high due to calling all ML APIs.
\systemnamesecond{} enjoys a similar inference time with \systemname{} but a 100x smaller training time.

\begin{table}[htbp]
  \centering
  \small
  \caption{Cost savings achieved by \systemnamesecond{} that reaches same accuracy as the best commercial API. On average the cost saving across the evaluated datasets is 73\%.}
    \begin{tabular}{|c|c||c|c|c|c|}
    \hline
    Task  & Dataset & Acc (\%) & Best API \$ & Our \$ & \multicolumn{1}{c|}{Save} \bigstrut\\
    \hline
    \hline
    \multirow{3}[6]{*}{MIC} & PASCAL & 74.8  & 10    & 1.4   & 86\% \bigstrut\\
\cline{2-6}          & MIR   & 41.2  & 10    & 4.2   & 58\% \bigstrut\\
\cline{2-6}          & COCO  & 47.5  & 10    & 3     & 70\% \bigstrut\\
    \hline
    \multirow{3}[6]{*}{STR} & MTWI  & 67.9  & 210   & 30    & 86\% \bigstrut\\
\cline{2-6}          & ReCTS & 61.3  & 210   & 78    & 63\% \bigstrut\\
\cline{2-6}          & LSVT  & 53.8  & 210   & 67    & 68\% \bigstrut\\
    \hline
    \multirow{3}[6]{*}{NER} & CONLL & 52.6  & 3     & 1.5   & 50\% \bigstrut\\
\cline{2-6}          & ZHNER & 61.3  & 30    & 0.7   & 98\% \bigstrut\\
\cline{2-6}          & GMB   & 50.1  & 30    & 4.1   & 80\% \bigstrut\\
    \hline
    \end{tabular}%
  \label{tab:SFAME:costsaving}%
\end{table}%

\begin{figure*}[t]
\centering     
\begin{subfigure}[PASCAL]{\label{fig:f}\includegraphics[width=0.32\linewidth]{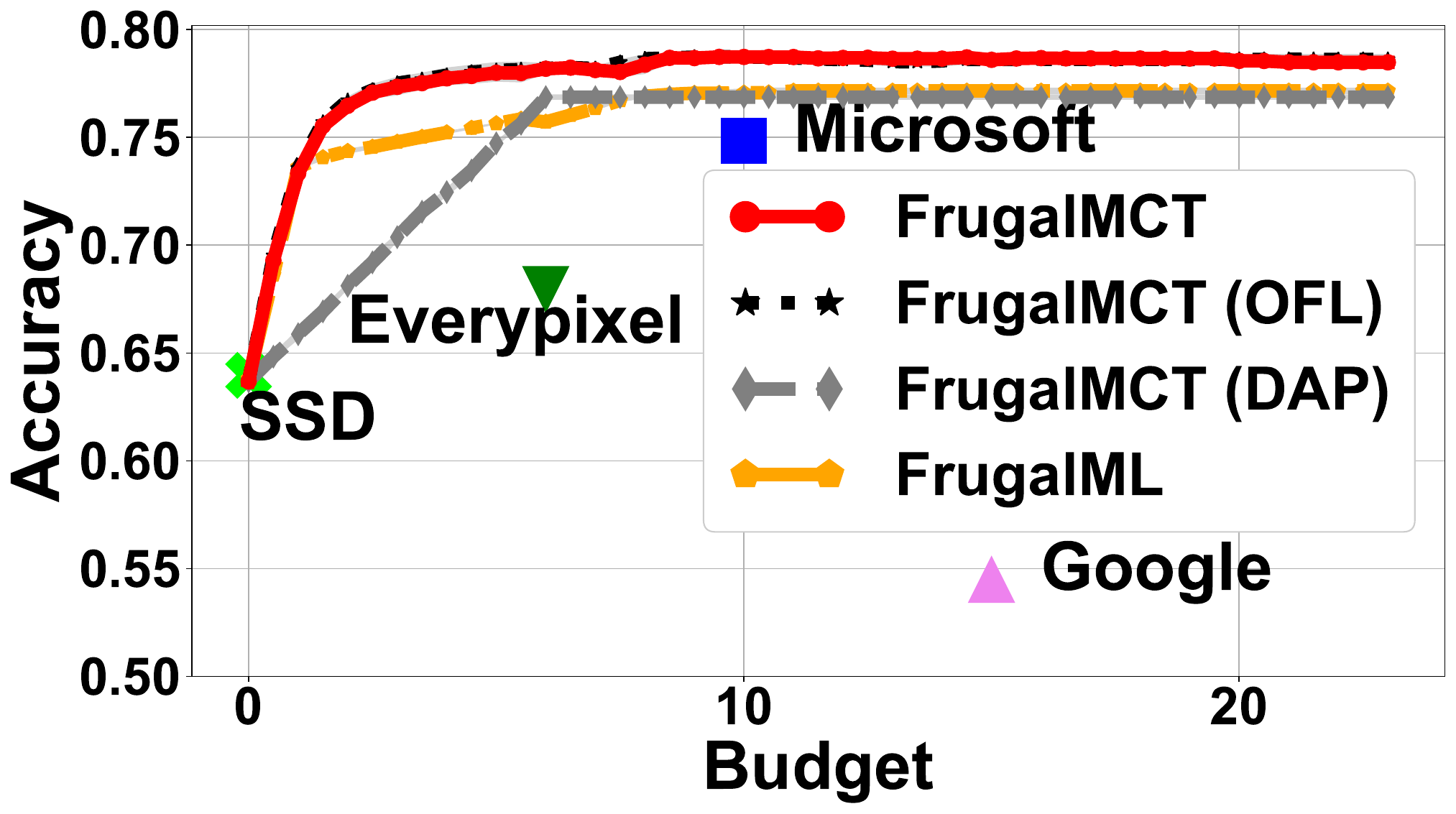}}
\end{subfigure}
\begin{subfigure}[MIR]{\label{fig:g}\includegraphics[width=0.32\linewidth]{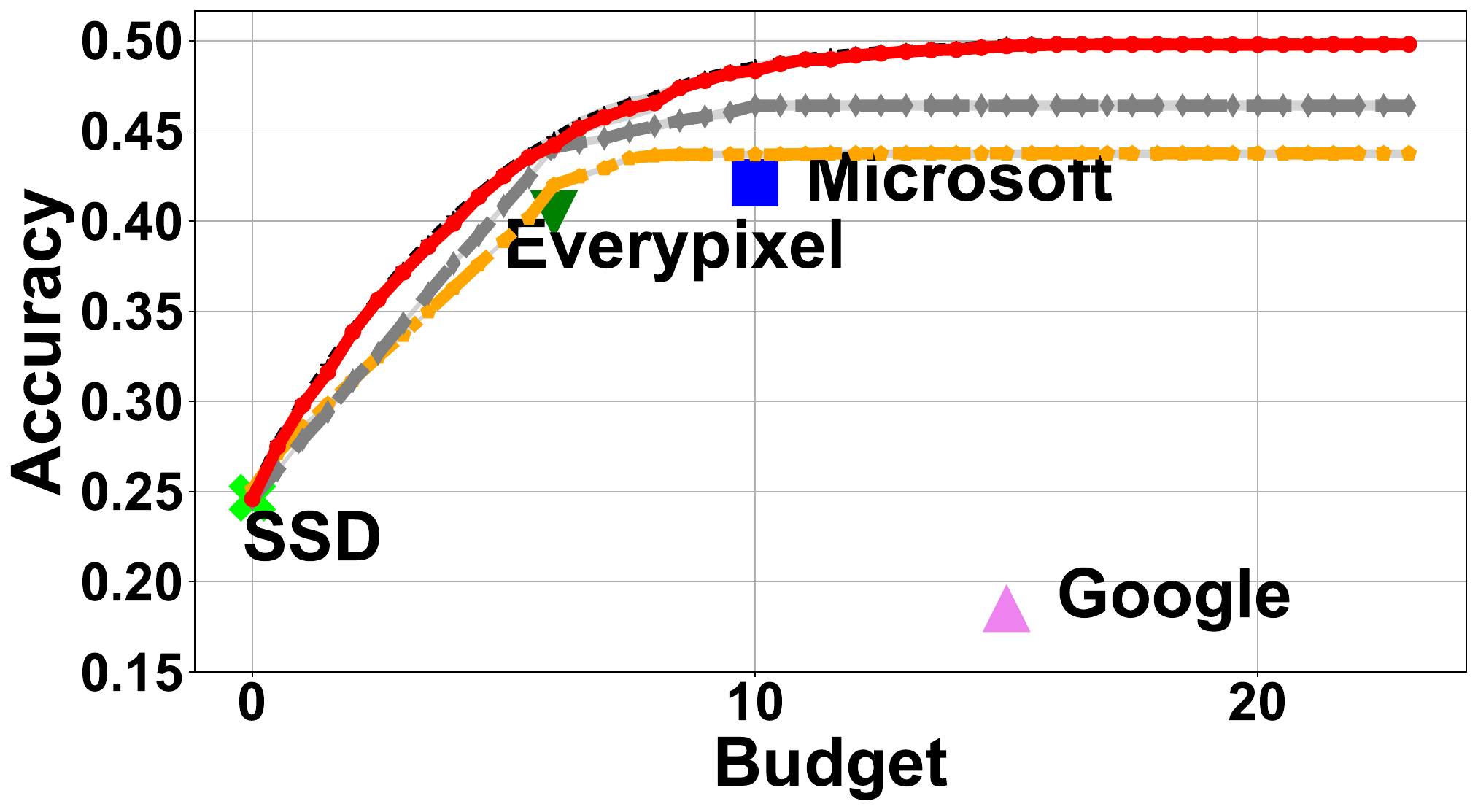}}
\end{subfigure}
\begin{subfigure}[COCO]{\label{fig:e}\includegraphics[width=0.32\linewidth]{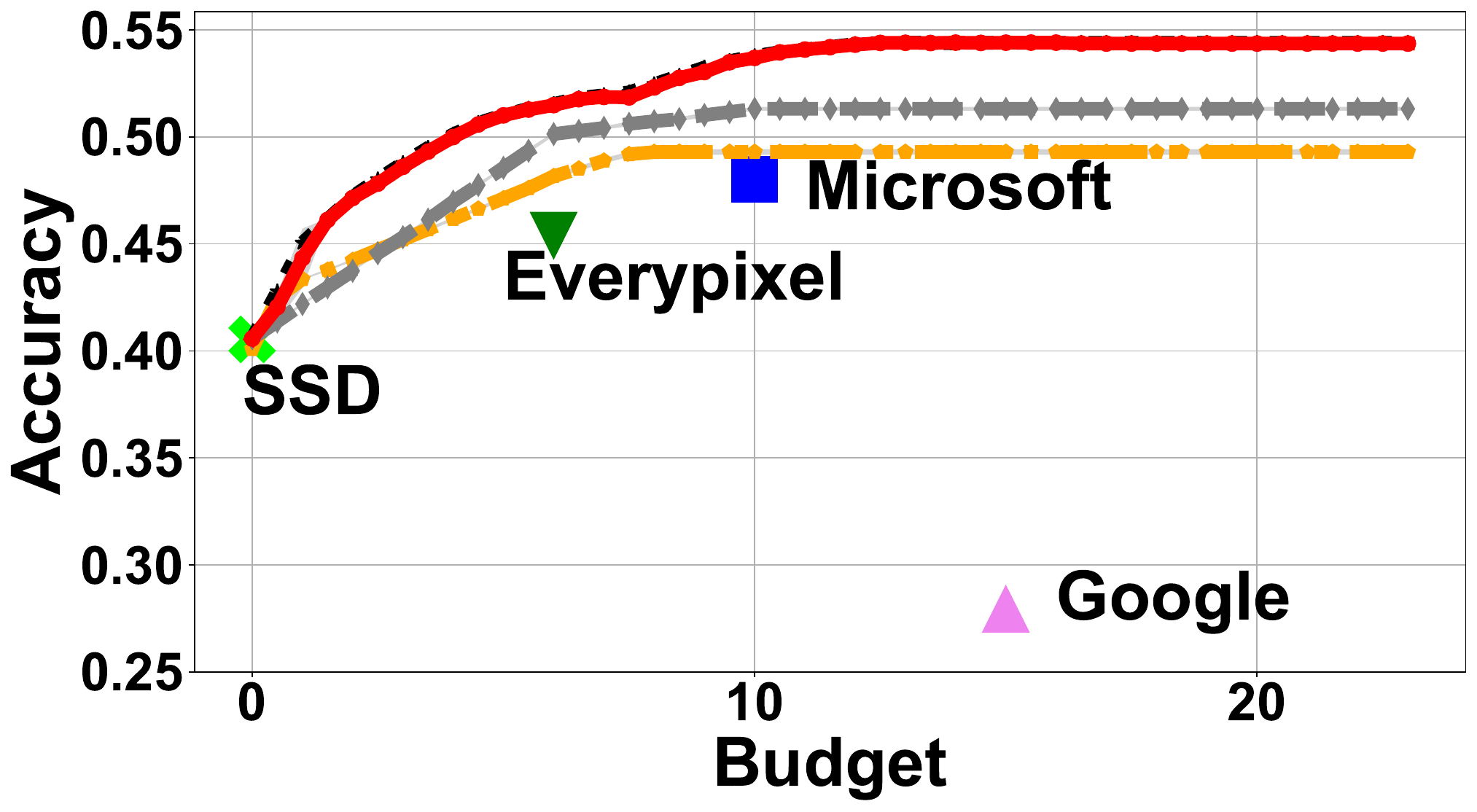}}
\end{subfigure}

\begin{subfigure}[MTWI]{\label{fig:d}\includegraphics[width=0.32\linewidth]{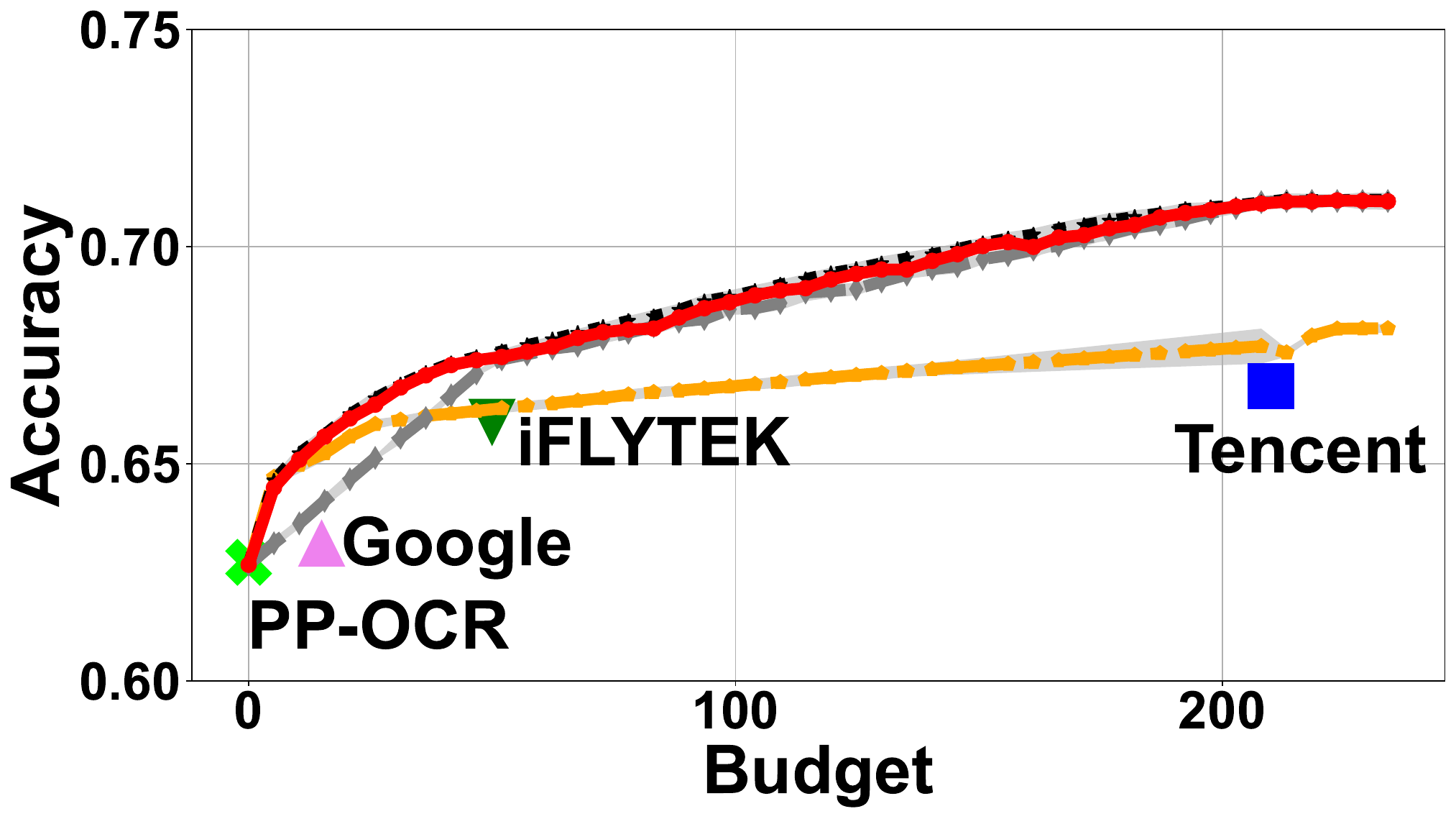}}
\end{subfigure}
\begin{subfigure}[ReCTS]{\label{fig:b}\includegraphics[width=0.32\linewidth]{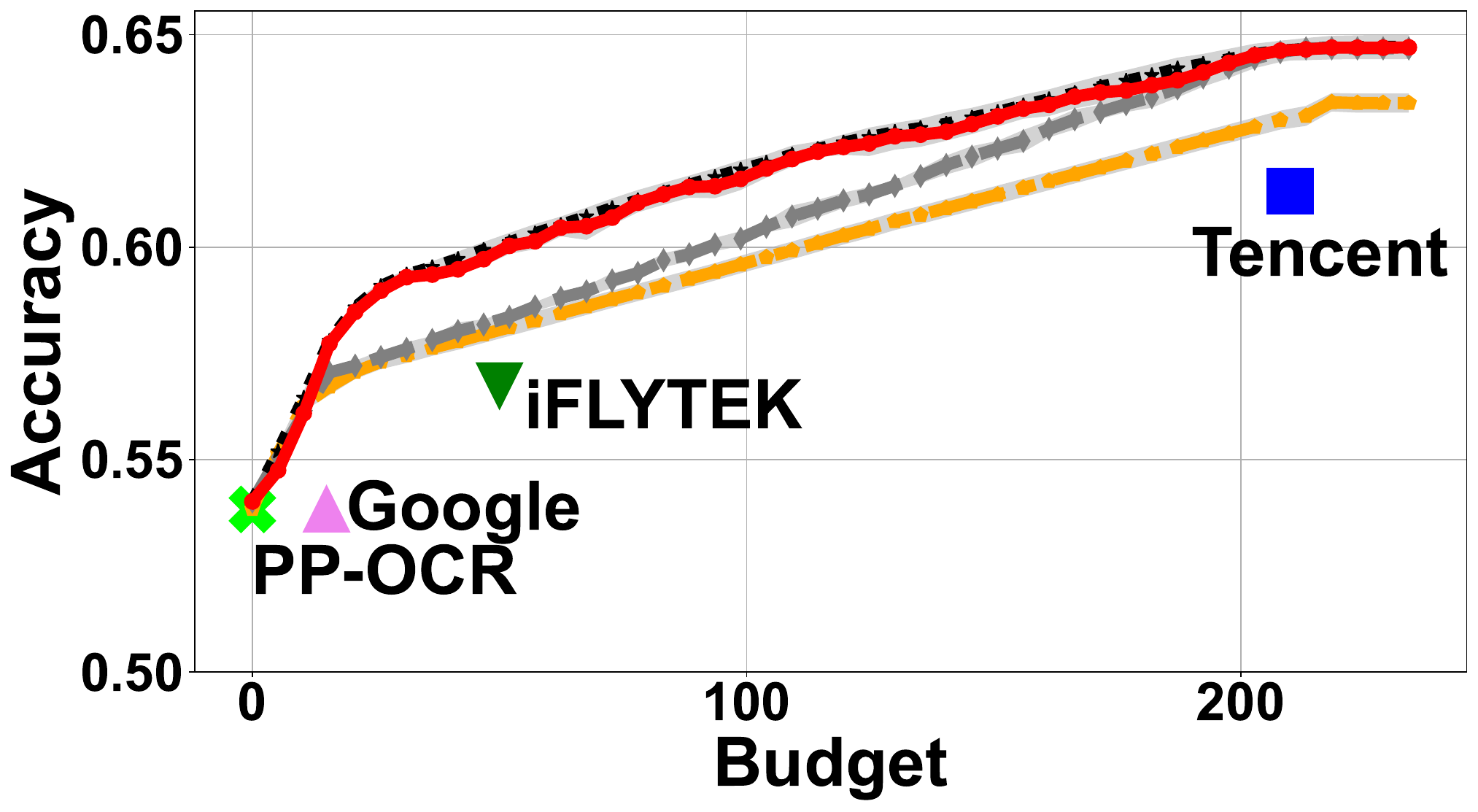}}
\end{subfigure}
\begin{subfigure}[LSVT]{\label{fig:a}\includegraphics[width=0.32\linewidth]{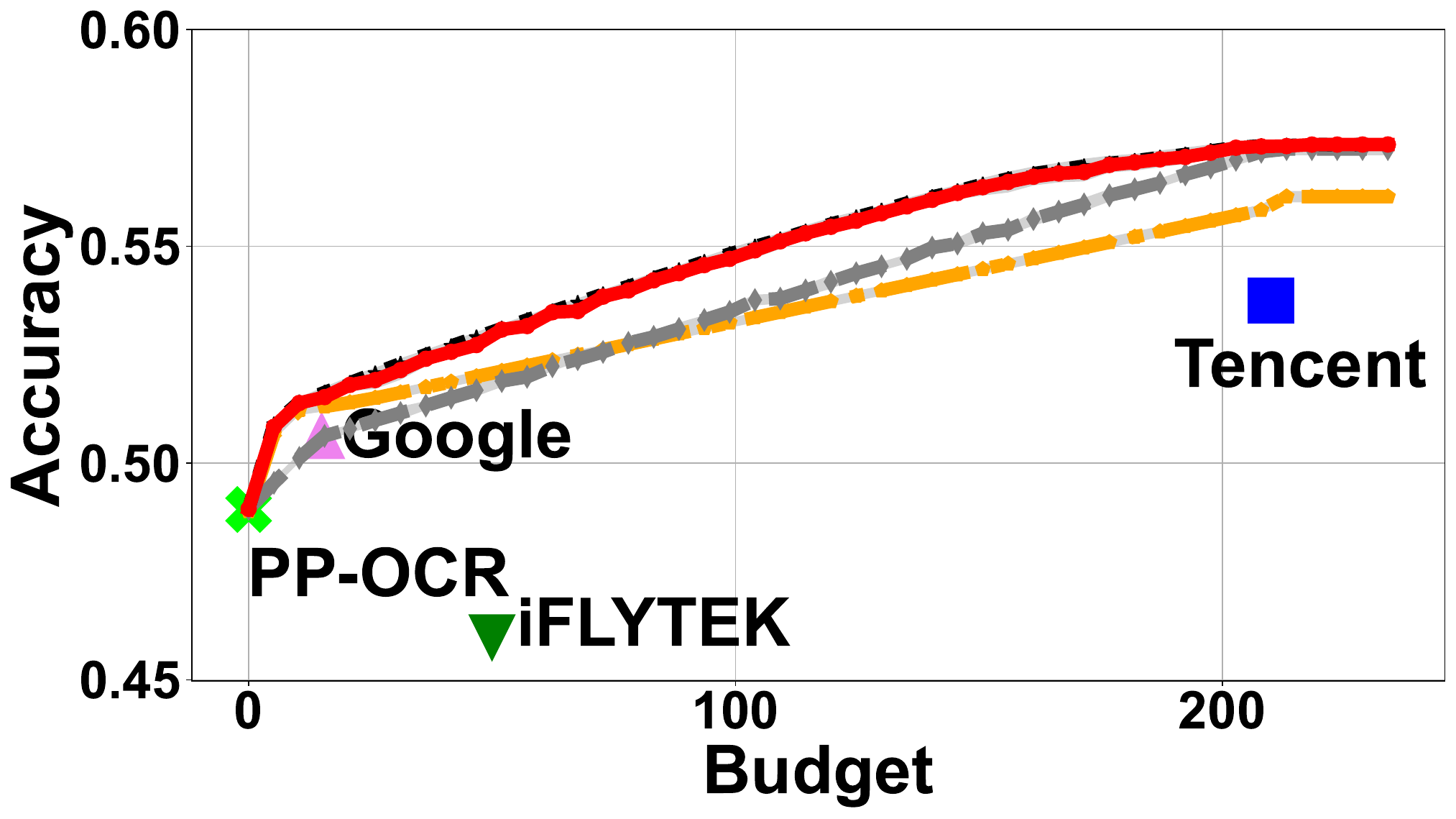}}
\end{subfigure}

\begin{subfigure}[CONLL]{\label{fig:j}\includegraphics[width=0.32\linewidth]{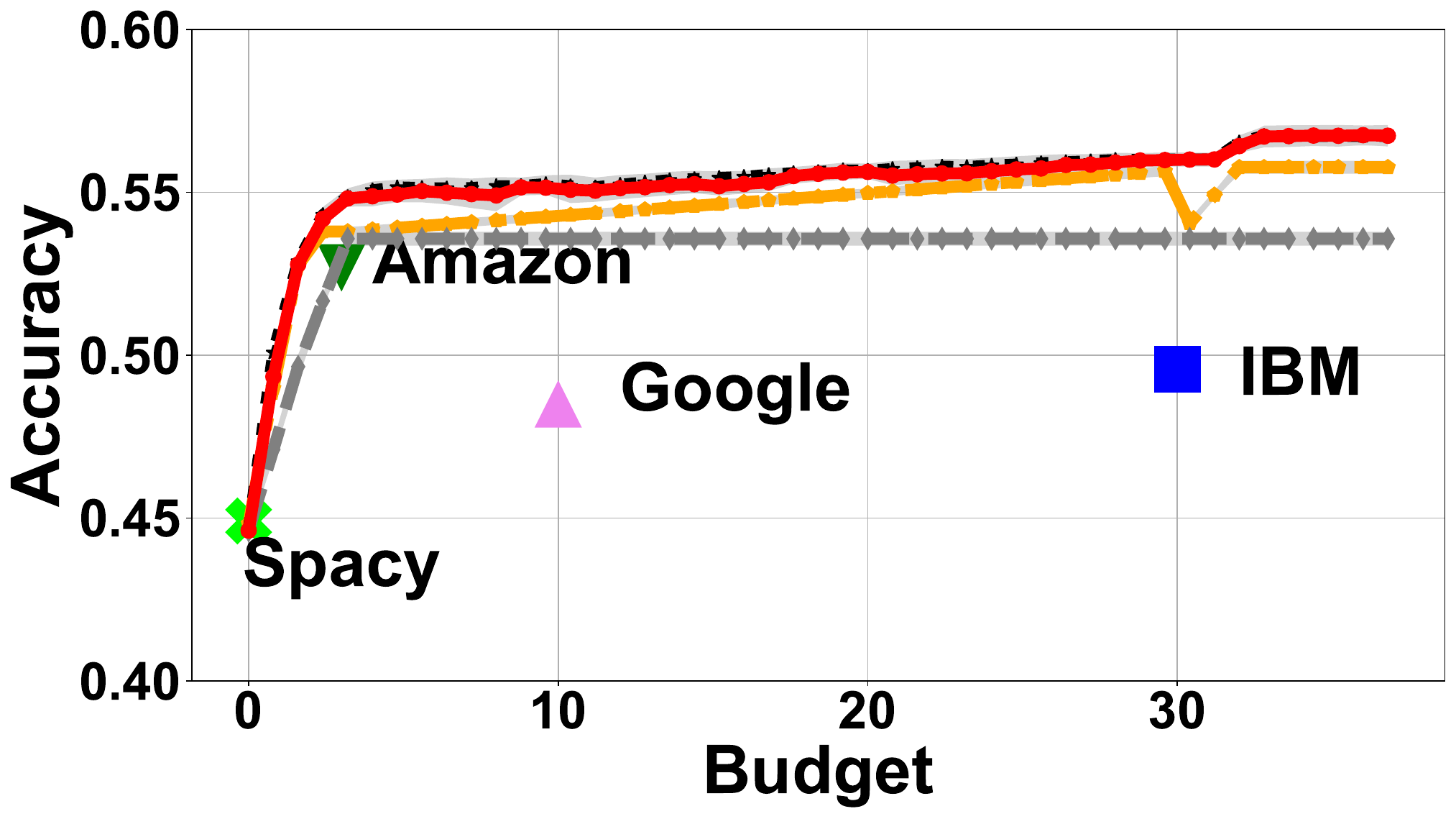}}
\end{subfigure}
\begin{subfigure}[ZHNER]{\label{fig:k}\includegraphics[width=0.32\linewidth]{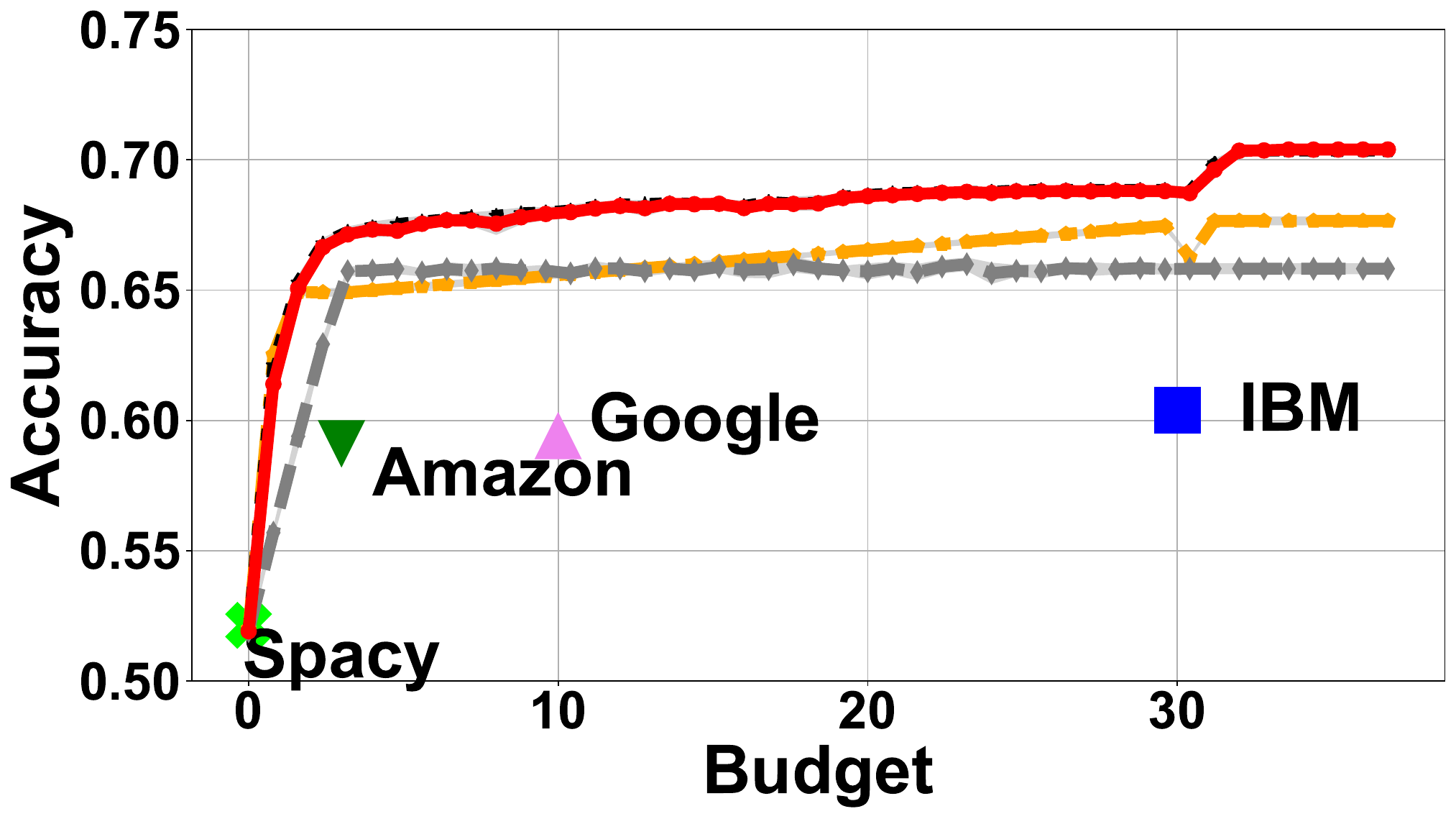}}
\end{subfigure}
\begin{subfigure}[GMB]{\label{fig:i}\includegraphics[width=0.32\linewidth]{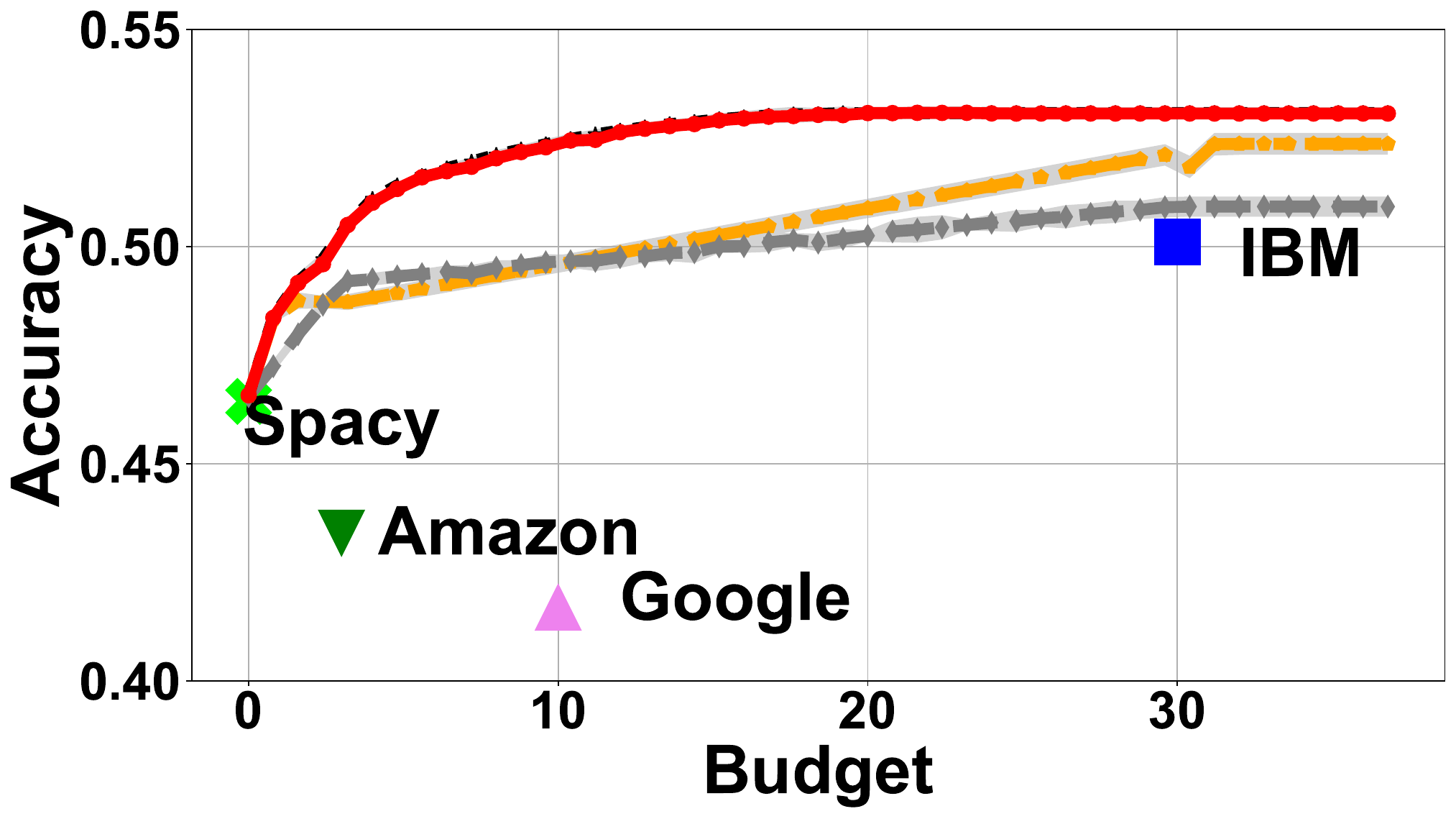}}
\end{subfigure}
\vspace{-0mm}
\caption{Accuracy cost trade-offs.	The offline \systemnamesecond{} (black) observes the full data and then make decisions. The online \systemnamesecond{} (red) matches the offline performance in all the experiments. DAP (grey) is an oblation of \systemnamesecond{} where a dummy accuracy predictor is used. 
\systemname{} (orange) is the previous state-of-the-art method. 
The task of row 1, 2, 3 is \textit{MIC}, \textit{STR}, and \textit{NER}.}\label{fig:SFAME:AccCostTradeoff}
\end{figure*}

\paragraph{Analysis of Cost Savings.} Next, we evaluate how much cost can be saved by \systemnamesecond{} to reach the highest accuracy produced by a single API on different tasks.
As shown in Table \ref{tab:SFAME:costsaving},
\systemnamesecond{} can typically save more than 60\% of the cost.
Interestingly, the cost saving can be up to 98\% on the dataset ZHNER. 
This is probably because (i) the accuracy estimator enables the API selector to identify when the base service's prediction is reliable and to avoid unnecessarily calling add-on services, and (ii) when add-on API is invoked,  the apt combination of the base and add-on services leads to a high accuracy improvement.

\paragraph{Accuracy and Cost Trade-offs.}
Now we dive deeply into the accuracy and cost trade-offs achieved by \systemnamesecond{}, shown in Figure \ref{fig:SFAME:AccCostTradeoff}.
We compare with two ablations: ``Offline'', where the full data is observed before making decision, ``DAP'', where a dummy accuracy predictor is used, which, for each API, always returns its mean accuracy on the training dataset. We also compared with  an adapted version of the previous state-of-the-art for single label task, \systemname{}. To adapt it to multi-label tasks, we use the average quality score over all predicted labels as a single score, and cluster all labels into a ``superclass''. 

\begin{figure*}[t] \centering
\begin{subfigure}[PASCAL]{\label{fig:sample_a}\includegraphics[width=0.33\linewidth]{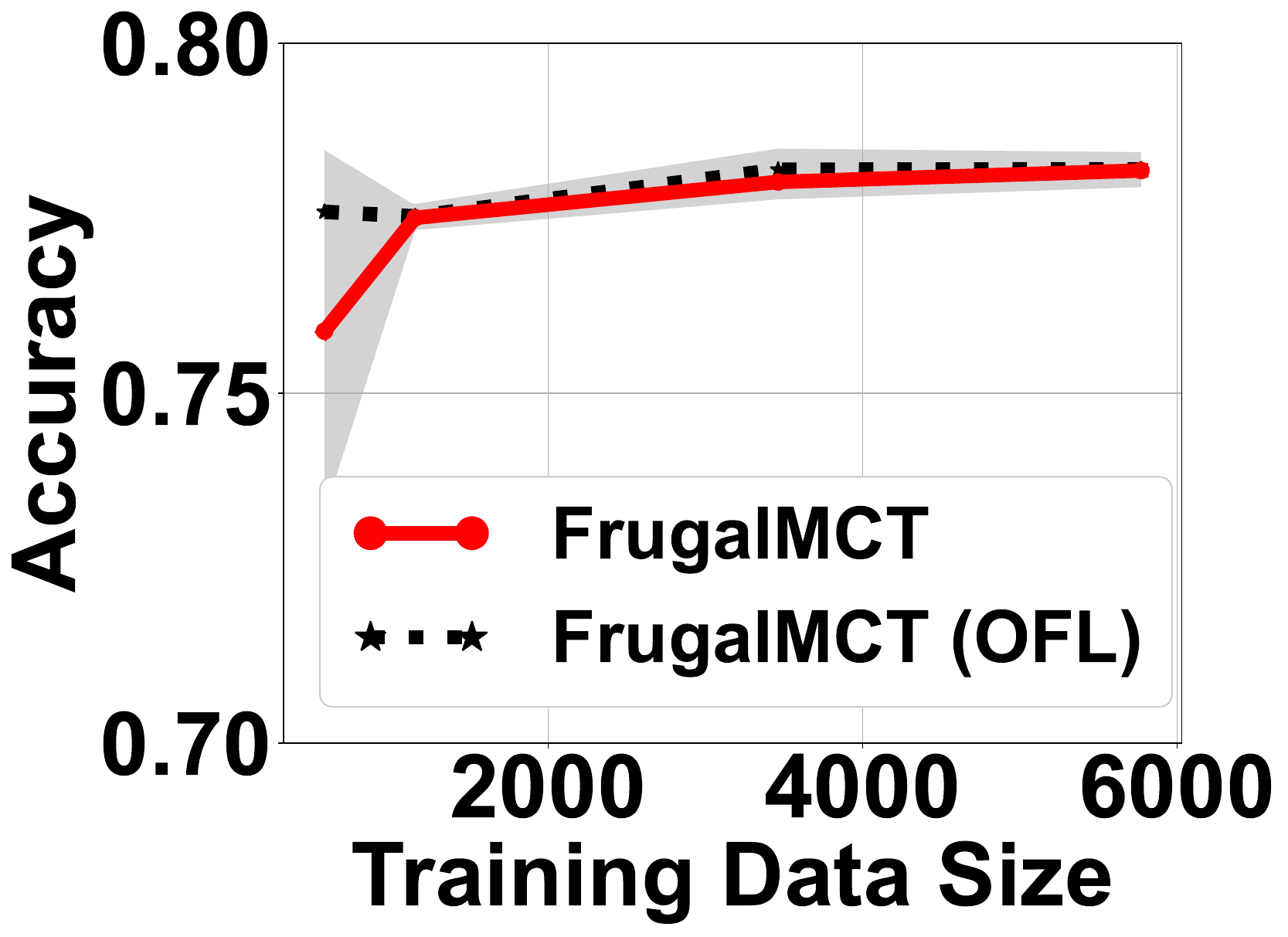}}
\end{subfigure}
\begin{subfigure}[MTWI]{\label{sample_b}\includegraphics[width=0.31\linewidth]{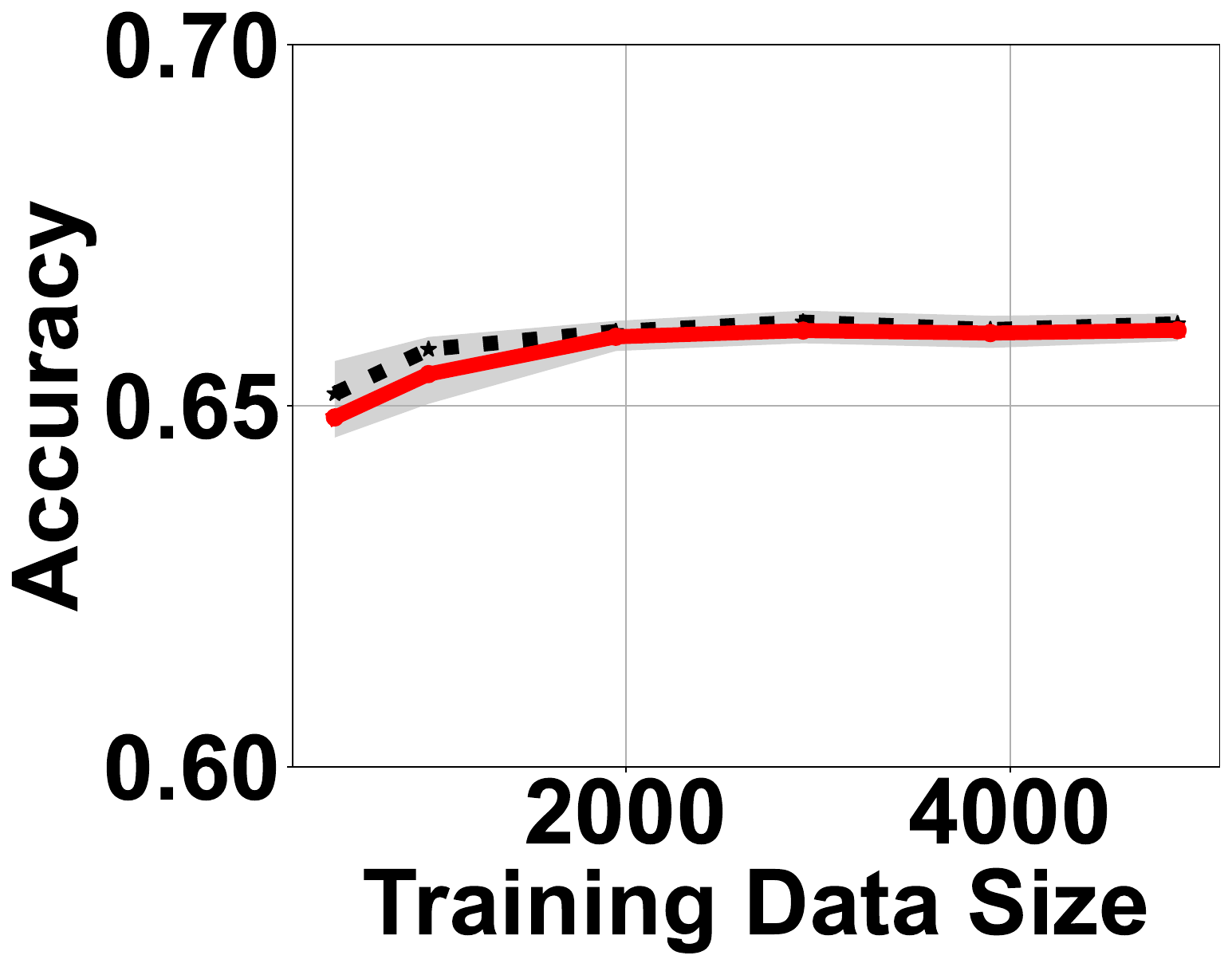}}
\end{subfigure}
\begin{subfigure}[CONLL]{\label{sample_c}\includegraphics[width=0.31\linewidth]{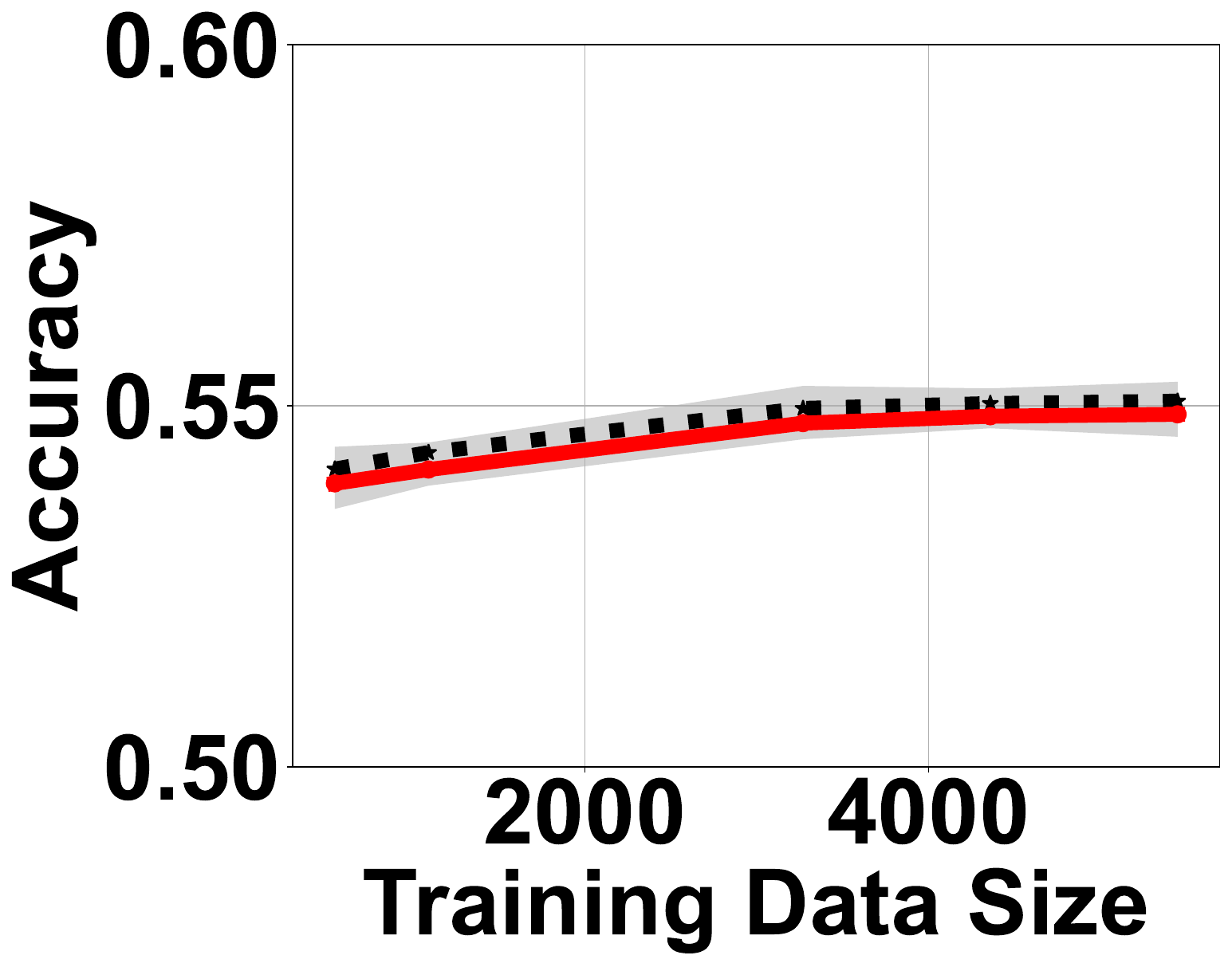}}
\end{subfigure}
\vspace{-0mm}
	\caption{{Testing accuracy v.s.training data size. The fixed budget is 6, 15, 3, respectively.}}\label{fig:SFAME:samplesize}
	\vspace{-0mm}
\end{figure*}

\eat{There are several interesting observations on the commercial APIs. 
First, Microsoft API consistently performs the best across different MIC datasets, while Tencent API consistently beats the other API's performance on the STR task. 
Though, for the NER datasets, the best API can be either IBM or Amazon API. 
Second, most expensive APIs are not always the most accurate. For example, Google API is the most expensive for MIC, but its accuracy is worse than the cheap GitHub API. 
Those suggest the necessity of using \systemnamesecond{} 
for different datasets. 
}
Compared to any single API,
\systemnamesecond{} allows users to pick any point in its trade-off curve and  offers substantial more flexibility.
In addition, \systemnamesecond{} often achieves higher accuracy than any ML services it calls.
For example, on COCO and ZHNER, more than 5\% accuracy improvement can be reached with the same cost of the best API.
Note that \systemnamesecond{} also outperforms  \systemname{} with the same budget.
This is primarily because  \systemnamesecond{} (i) utilizes a more principled way to use the features (learning an accuracy estimator) than FrugalML (directly using the label info), and (ii) adopts a label combiner designed for multi-label tasks. 
Ensemble methods such as majority votes (in the appendix \ref{sec:SFAME:experimentdetails}) produce accuracy similar to  \systemnamesecond{}, but their cost is much higher. 
Note that there is little performance difference between the online \systemnamesecond{} strategy and the offline approach, due to the carefully designed online algorithm. This directly supports our theory.

\begin{table}[htbp]
  \centering
  \small
  \caption{Performance of \systemnamesecond{}'s accuracy predictor. Root mean square error (RMSE) quantifies the standard deviation of the differences between the predicted and the true accuracy. }
    \begin{tabular}{|c|c|c|c|c|c|}
    \hline
    Data  & RMSE  & Data  & RMSE  & Data  & RMSE \bigstrut\\
    \hline
    \hline
    PASCAL & 0.28  & MIR   & 0.22  & COCO  & 0.24 \bigstrut\\
    \hline
    MTWI  & 0.17  & ReCTS & 0.22  & LSVT  & 0.19 \bigstrut\\
    \hline
    CONLL & 0.29  & ZHNER & 0.31  & GMB   & 0.28 \bigstrut\\
    \hline
    \end{tabular}%
  \label{tab:SFAME:AccPredMeasure}%
\end{table}%

\paragraph{Effects of Accuracy Predictors.} The accuracy predictors play an important role in \systemnamesecond{}'s performance.
As Table \ref{tab:SFAME:AccPredMeasure} shows, \systemnamesecond{} provides nontrivial accuracy estimates which enables its success. It's interesting to note that the accuracy predictor doesn't need to be perfect for \systemnamesecond{} to do well; for example, the root mean square error (RMSE) of the accuracy predictor is 0.28 on PASCAL (and 0.29 on CONLL), but \systemnamesecond{} still produces consistently better accuracy than FrugalML. 
We also evaluated \systemnamesecond{}'s performance when the accuracy predictors are obtained via two AutoML toolkits, auto-sklearn~\cite{AutoSKlearn_NIPS2015} and Auto-PyTorch~\cite{AutoPytorch2019} instead of random forest, and observe a similar performance. 

\paragraph{Effects of Training Sample Size.}

Finally we study how the training dataset size affects \systemnamesecond{}'s performance. As shown in Figure \ref{fig:SFAME:samplesize}, across different tasks, a few thousand training samples are typically sufficient to learn the optimal \systemnamesecond{} strategy. This is usually more efficient than training a customized ML model from scratch. 
It also only takes a few minutes to train those  \systemnamesecond{} strategies, which is much faster than training a model from scratch. 
This is useful in latency-critical applications.

\paragraph{Training cost of \systemnamesecond{}.} Both dollar cost and computation time of training are often much smaller than ML APIs’ inference cost.This is because (i) training is a one-time cost and (ii) \systemnamesecond{} requires a small number of label annotations (a few thousands see Figure \ref{fig:SFAME:samplesize}). Consider the image tagging task as an example:  the dollar cost of calling all APIs is $\$0.0006+\$0.001+\$0.0015=\$0.0031$ per image. Labeling for (say) five thousands images takes $\$0.0031\times 5000=\$15.5$. Training a FrugalMCT strategy on half of the COCO dataset takes $59.5$s on the experiment machine. 
This is much cheaper than calling the selected APIs after at large scale (e.g., millions of images).

\section{Conclusion}\label{Sec:SFAME:Conclusion}
In this paper, we presented \systemnamesecond{}, an algorithmic framework to adaptively select and combine ML APIs for multi-label classification tasks within a budget constraint. 
\systemnamesecond{} integrates forecasts of API's accuracy with online constrained optimization to create an end-to-end algorithm with strong empirical performance and theoretical guarantees.
How to efficiently use multi-label APIs is an important problem in practice for the large number of ML users who have chosen to rely on commercial prediction APIs, and has not been studied heavily in the ML literature.
This work helps MLaaS users improve  the overall accuracy and cost of their applications.
Extensive empirical evaluation using real commercial APIs shows that \systemnamesecond{} significantly improves both cost and accuracy. 

To encourage more research on MLaaS, we  also release the dataset used to develop \systemnamesecond{}, consisting of 295,212 samples annotated by commercial multi-label prediction APIs. 
The dataset and our code can be accessed from  \url{https://github.com/lchen001/FrugalMCT}.

\section*{Acknowledgement}
This work was supported in part by a Google PhD Fellowship, a Sloan Fellowship, NSF CCF 1763191, NSF CAREER AWARD
1651570 and 1942926, NIH P30AG059307, NIH U01MH098953, grants from the Chan-Zuckerberg
Initiative, Sutherland, and affiliate members and other supporters of the Stanford DAWN project, including Meta, Google,  and VMware. We also thank anonymous reviewers for helpful discussion and feedback.

\newpage
\newpage
{
\bibliography{MLService}
\bibliographystyle{icml2022}
}


\appendix
\newpage
\onecolumn
\paragraph{Outline.} The appendix is organized as follows. We present missing technical details in Section \ref{sec:SFAME:techdetails}. 
The proofs are provided in Section \ref{Sec:SFAME:proofs}.
Finally, Section \ref{sec:SFAME:experimentdetails} gives detailed experiment setups and additional empirical results.

\section{Technical Details}\label{sec:SFAME:techdetails}
Additional  technical details are presented here.

\paragraph{Label combiner parameter search.} 
Recall that the label combiner requires two parameters: the combining weight $w\in[0,1]$ and the quality score threshold $\theta\in[0,1]$.
We adopt a simple grid search approach to select $w$ and $\theta$.
More precisely, we first create a parameter candidate set $PCS\triangleq \{w_0,w_1,w_2,\cdots, w_M\}\times \{\theta_0, \theta_1,\theta_2,\cdots, \theta_M\}$, where $w_m=\frac{m}{M}$ and $\theta_i=\frac{m}{M}$.  
Next, for each $(w,\theta) \in PCS$, we  evaluate the performance of combining the base service and the $k$th service using $(w,\theta)$, and select the parameter that gives the highest accuracy. 
Note that this involves $M^2$ number of label combinations for each $k\in[K]$. 
In practice, we have found that $M=10$ is sufficient to obtain a good combiner.

\paragraph{$\delta$ selection in Algorithm \ref{Alg:SFAME:OnlineAlg}.}
A naive approach is to set a small constant value, say, $\delta=0.01$.
To obtain a more accurate strategy, we can adopt a search algorithm to select the best $\delta$ value based on the evaluation the performance on a validation dataset.
More precisely, we first create a constant set $\textit{CS}$. Then for each $\alpha \in \textit{CS}$, let $\delta=\alpha \frac{\log N}{N}$, and then solve Problem \ref{prob:SFAME:APISelectionTrainingDual} to obtain the parameter $\hat{p}$, evaluate the performance on a validation dataset.
Finally, we select the $\alpha \in \textit{CS}$ that achieves the highest accuracy on the validation dataset.
In practice, we have found that $\textit{CS}=\{-10,-9,-8,\cdots, 0, 1, 2,\cdots, 10 \}$ is sufficient to obtain a highly accurate solution.

\section{Proofs}\label{Sec:SFAME:proofs}
For ease of notations, let us introduce 
 $\hat{b} \triangleq b - \pmb c_{base}$
 and $\hat{\pmb c}_k \triangleq \pmb c_{k} \cdot \mathbbm{1}_{k\not=\textit{base}}$ first.
 Then we can rewrite the API selection problem (Problem \ref{prob:SFAME:optimaldefinition}) as
\begin{equation}\label{Prob:SFAME:ILPTransform}
    \begin{split}
        \max_{\pmb Z \in \R^{N\times K}:} \textit{ }&  \frac{1}{N} \sum_{n=1}^{N} \sum_{k=1}^{K} \pmb{Z}_{n,k} \pmb {\hat{a}}_{k}(x_n)\\
        s.t. & \frac{1}{N} \sum_{n=1}^{N} \sum_{k=1}^{K} \pmb Z_{n,k} \hat{\pmb c}_{k} \leq  \hat{b} \\
        & \sum_{k=1}^{K} \pmb{Z}_{n,k} = 1,  \forall n;      \pmb{Z}_{n,k} \in \{0,1\}, \forall n, k
    \end{split}
\end{equation} 
Its corresponding linear programming simply becomes
\begin{equation}\label{prob:SFAME:APISelectionLP}
    \begin{split}
        \max_{\pmb Z \in \R^{N\times K}:} \textit{ }&  \frac{1}{N} \sum_{n=1}^{N} \pmb{Z}_{n,k} \pmb {\hat{a}}_{k}(x_n)\\
        s.t. & \frac{1}{N} \sum_{n=1}^{N} \sum_{k=1}^{K} \pmb Z_{n,k} \hat{\pmb c}_{k} \leq  \hat{b} \\
        & \sum_{k=1}^{K} \pmb{Z}_{n,k} = 1, \forall n;       \pmb{Z}_{n,k} \in [0,1], \forall n, k
    \end{split}
\end{equation}
We will analyze some useful properties for those two problems first, and then prove the desired results for the original API selection problem on top of those properties.

\subsection{Helpful Lemmas}
Before proving the desired results, let us also provide a few generic lemmas.

\begin{lemma}\label{lemma:SFAME:ZeroNormBound}
    Let $\pmb A \in \R^{N_1\times N_2}$ be a fixed matrix and $\pmb \beta \in \R^{N_1}$ be a random vector.
    If $\pmb \beta$ is supported on $[0,1]^{N_1}$ with a continuous density function, then with probability 1, \begin{equation*}
        \min_{\pmb x} \| \pmb A \pmb x - \pmb \beta\|_0 \geq N_1 -  N_2.
    \end{equation*} 
    \end{lemma}

\begin{proof}
If $N_1\leq N_2$ then the above inequality obviously holds. Suppose $N_1>N_2$.  We prove this by contradiction.
Assume the inequality does not hold.
Then there exists some $\pmb x'$, such that with probability larger than 0,
\begin{equation*}
    \| \pmb A \pmb x' - \pmb \beta \|_0 < N_1 - N_2.
\end{equation*}
That is to say, at least $N_1-(N_1-N_2)+1=N_2+1$ many equations in $\pmb A \pmb x' = \pmb \beta$ can be forced to $0$.
Let $U$ be the set of those $N_2+1$ indexes.
Then formally we have 
\begin{equation*}
    \pmb A_{U} \pmb x' = \pmb \beta_{U}.
\end{equation*}
That is to say, with probability larger than 0, $\pmb \beta_U$ is in the subspace formed by the columns of $\pmb A_U$.

On the other hand, we can show that for any set of indexes $V$ with $|V|=N_2+1$, $\pmb \beta_V$ lies in the subspace formed by the columns of $\pmb A_V$ with probability 0, which gives a contradiction. 
To see this, let us start by considering a fixed set of indexes $V$. 
Let $\pmb \Omega_{V}$ denote the subspace formed by $\pmb A_V$ and $p_{\pmb \beta_V}(\cdot)$ be the density function of $\pmb \beta_V$.
The density function of $\pmb \beta$ is continuous and thus $p_{\pmb \beta_V}(\cdot)$ is also continuous. 
The support of $\pmb \beta$ is in $[0,1]^{N_1}$, and thus the support of $\pmb\beta_V$ is in $[0,1]^{N_2+1}$ (since $|V|=N_2+1$ by definition).
That is to say, $p_{\pmb \beta_V}(\cdot)$ is a continuous function on a compact set.
Therefore, $p_{\pmb \beta_V}(\cdot)$ must be bounded, i.e., there exists a constant $p^{\sup}$ such that $p_{\pmb \beta_V}(\cdot) \leq p^{\sup}$.
Hence we have 
\begin{equation*}  \Pr[\pmb \beta_{V} \in  \pmb \Omega_{V}] =   \int_{\pmb x\in \pmb \Omega_V}  p_{\pmb \beta_V}(\pmb x)d \pmb x  \leq \int_{\pmb x_V \in \pmb \Omega_V} p^{\sup} d \pmb x = p^{\sup} \int_{\pmb x_V \in \pmb \Omega_V} 1 d \pmb x 
\end{equation*}
where the first equation is by definition of the random variable $\pmb \beta_V$, the inequality is by increasing the density function to its upper bound $p^{\sup}$, and the last equation simply moves the constant out of the integral.
In addition, $\pmb \Omega_V$ is a $N_2+1$ dimensional space spanned by $N_2$ vectors, which implies that its measure in  $\R^{N_2+1}$ is 0, i.e, 
$\int_{\pmb x_V \in \pmb \Omega_V} 1 d \pmb x = 0$.
Thus, we have just shown that 
\begin{equation*}  \Pr[\pmb \beta_{V} \in  \pmb \Omega_{V}] \leq p^{\sup} \int_{\pmb x_V \in \pmb \Omega_V} 1 d \pmb x = 0 
\end{equation*}
Probability is non-negative, and thus $ \Pr[\pmb \beta_{V} \in  \pmb \Omega_{V}]=0$ (for a fixed $V$). Note that the size of $V$ is $N_2+1$ and there are in total $N_1$ possible indexes. Thus, there are $N_1\choose{N_2+1}$ many possible choices of $V$.
Applying union bound, we have 
for any $V$, $ \Pr[\pmb \beta_{V} \in  \pmb \Omega_{V}]=0$. 
A contradiction. 
The assumption is incorrect, and thus we must have \begin{equation*}
        \min_{\pmb x} \| \pmb A \pmb x - \pmb \beta\|_0 \geq N_1 -  N_2.
    \end{equation*} 
\end{proof}

\begin{lemma}\label{lemma:SFAME:optimizationdiffdomain}
Let $f$ be a function defined on $\Omega_{\pmb{z}}$.
Assume there exists a set $\Omega_{\pmb z,1} \subseteq \Omega_{\pmb{z}}$, such that for any $\pmb z\in \Omega_{\pmb z}$, there exists $\pmb z' \in \Omega_{\pmb z, 1}$, such that $\|f(\pmb z) - f(\pmb z')\|\leq \Delta$.
Then we have
\begin{equation*}
\|    f(\pmb z^*) - f(\pmb z^*_1) \| \leq \Delta,
\end{equation*}
where $\pmb z^* = \arg \max_{\pmb z \in \Omega_{\pmb z}} f(\pmb z), \pmb z^*_1 = \arg \max_{\pmb z \in \Omega_{\pmb z,1}} f(\pmb z)$.
\end{lemma}
\begin{proof}
By assumption, there exists a $\pmb z'\in \Omega_{\pmb z,1}$, such that 
\begin{equation*}
    \|f(\pmb z^*) - f(\pmb z')\| \leq \Delta
\end{equation*}
which implies
\begin{equation*}
    f(\pmb z') \geq f(\pmb z^*) - \Delta
\end{equation*}
Noting that $\pmb z^*_1$ is the optimal solution on $\Omega_{\pmb z,1}$ and $\pmb z'$ is a feasible solution, we have 
\begin{equation*}
    f(\pmb z^*_1) \geq f(\pmb z') 
\end{equation*}
Combining the above two inequalities, we have 
\begin{equation*}
    f(\pmb z^*_1) \geq f(\pmb z^*) - \Delta 
\end{equation*}
On the other hand, since $\Omega_{\pmb z,1} \subseteq \Omega_{\pmb z}$, $\pmb z^{*}_{1}$ is a feasible solution on $\Omega_{\pmb z}$, and thus we have
\begin{equation*}
    f(\pmb z^*_1) \leq f(\pmb z^*)  \leq f(\pmb z^*) +\Delta 
\end{equation*}
Combing those two inequalities we have 
\begin{equation*}
\|    f(\pmb z^*_1)- f(\pmb z^*)\|  \leq \Delta
\end{equation*}
which completes the proof.
\end{proof}
\begin{lemma}\label{Lemma:SFAME:traintestmatch}
Let $X_1, X_2,\cdots, X_{N_1}$ and $X_1', X_2',  \cdots X_{N_2}$ be two i.i.d. samples from the same distribution which lies in $[x_{\inf},x_{\sup}]$.
Then we have with probability $1-\epsilon$, 
\begin{equation*}
    \| \frac{1}{N_2}\sum_{n=1}^{N_2} X_{n}' - \frac{1}{N_1}\sum_{n=1}^{N_1} X_{n}\| \leq (x_{\sup}-x_{\inf})
\left[\sqrt{\frac{\log 4 - \log \epsilon }{2 N_2}} + \sqrt{\frac{\log 4 - \log \epsilon }{2 N_1}}\right].
\end{equation*}
\end{lemma}
\begin{proof}
We can apply the Hoeffding's inequality for both sequences separately, and we can obtain with probability $1-\epsilon$,
\begin{equation*}
    \| \frac{1}{N_1}\sum_{n=1}^{N_1} X_{n} - \Exp[X_1]\| \leq (x_{\sup}-x_{\inf})
\sqrt{\frac{\log 2 - \log \epsilon }{2 N_1}}
\end{equation*}
and with probability $1-\epsilon$
\begin{equation*}
    \| \frac{1}{N_2}\sum_{n=1}^{N_2} X_{n}' - \Exp[X_1]\| \leq (x_{\sup}-x_{\inf})
\sqrt{\frac{\log 2 - \log \epsilon }{2 N_2}}
\end{equation*}
Now applying union bound, we have 
with probability $1-\epsilon$,
\begin{equation*}
    \| \frac{1}{N_1}\sum_{n=1}^{N_1} X_{n} - \Exp[X_1]\| \leq (x_{\sup}-x_{\inf})
\sqrt{\frac{\log 4 - \log \epsilon }{2 N_1}}
\end{equation*}
and  $1-\epsilon$
\begin{equation*}
    \| \frac{1}{N_2}\sum_{n=1}^{N_2} X_{n}' - \Exp[X_1]\| \leq (x_{\sup}-x_{\inf})
\sqrt{\frac{\log 4 - \log \epsilon }{2 N_2}}
\end{equation*}
Now applying the triangle inequality, we have
\begin{equation*}
    \| \frac{1}{N_2}\sum_{n=1}^{N_2} X_{n}' - \frac{1}{N_1}\sum_{n=1}^{N_1} X_{n}\| \leq (x_{\sup}-x_{\inf})
\left[\sqrt{\frac{\log 4 - \log \epsilon }{2 N_2}} + \sqrt{\frac{\log 4 - \log \epsilon }{2 N_1}}\right]
\end{equation*}
which completes the proof.
\end{proof}

\begin{lemma}\label{Lemma:SFAME:OptimiziationDiff}
Let $f_1, f_2, g_1, g_2$ be functions defined on $\Omega_{\pmb{z}}$, such that $\max_{\pmb{z}\in \Omega_{\pmb{z}}}|(f_1\pmb{z}) - f_2(\pmb{z})|\leq \Delta_1$
and $\max_{\pmb{z}\in \Omega_{\pmb{z}}} \|g_2(\pmb{z}) - g_1(\pmb{z}) \|\leq  \Delta_2$.
Suppose
\begin{equation*}
    \begin{split}
        \pmb{z}_1^* = \arg \max_{\pmb{z}\in \Omega_{\pmb{z}}} & f_1(\pmb{z})\\
        s.t. &  g_1(\pmb{z}) \leq 0
    \end{split}
\end{equation*}
and
\begin{equation*}
    \begin{split}
        \pmb{z}^*_2 = \arg \max_{\pmb{z}\in \Omega_{\pmb{z}}} & f_2(\pmb{z})\\
        s.t. &  g_2(\pmb{z}) \leq  \Delta_2,
    \end{split}
\end{equation*}
then we must have 
\begin{equation*}
    \begin{split}
        f_1(\pmb{z}^*_2) & \geq f_1(\mathbf{z}_1^*) -2\Delta_1\\
        g_1(\pmb{z}^*_2) & \leq 2 \Delta_2.
    \end{split}
\end{equation*}
\end{lemma}	
\begin{proof}
Note that  $\max_{\pmb{z}\in \Omega_{\pmb{z}}}|(f_1(\pmb{z}) - f_2(\pmb{z})|\leq \Delta_1$ implies $f_1(\pmb{z})\geq f_2(\pmb{z})-\Delta_1$ for any $\pmb{z}\in\Omega_{\pmb{z}}$.
Specifically, 
\begin{equation*}
f_1(\pmb{z}^*_2)\geq f_2(\pmb{z}^*_2)-\Delta_1
\end{equation*}
Noting  $\max_{\pmb{z}\in \Omega_{\pmb{z}}} \|g_2(\pmb{z}) - g_1(\pmb{z}) \|\leq  \Delta_2$,
we have $g_2(\pmb{z}^*_1) \leq g_1(\pmb{z}^*_1)-\Delta_2 \leq -\Delta_2$, where the last inequality is due to $g_1(\pmb{z}^*_1)\leq 0$ by definition.
Since, $\pmb{z}^*_1$ is a feasible solution to the second optimization problem, and the optimal value must be no smaller than the value at $\pmb{z}^*_1$.
That is to say,
\begin{equation*}
    \begin{split}
        f_2(\pmb{z}^*_2) \geq f_2(\pmb{z}^*_1)
    \end{split}
\end{equation*}
Hence we have 
\begin{equation*}
f_1(\pmb{z}^*_2)\geq f_2(\pmb{z}^*_2)-\Delta_1 \geq  f_2(\pmb{z}^*_1) - \Delta_1
\end{equation*}
In addition, $\max_{\pmb{z}\in \Omega_{\mathbf{z}}}|(f_1(\pmb{z}) - f_2(\pmb{z})|\leq \Delta_1$ implies $f_2(\pmb{z})\geq f_1(\pmb{z})-\Delta_1$ for any $\pmb{z}\in\Omega_{\pmb{z}}$.
Thus, we have $f_2(\pmb{z}_1^*)\geq f_1(\pmb{z})^*_1-\Delta_1$ and thus
\begin{equation*}
f_1(\pmb{z}^*_2)\geq   f_2(\pmb{z}^*_1) - \Delta_1 \geq 
f_1(\pmb{z}^*_1) - 2 \Delta_1
\end{equation*}

By  $\max_{\pmb{z}\in \Omega_{\pmb{z}}}|g_1(\pmb{z}) - g_2(\pmb{z})|\leq  \Delta_2$,
we must have $g_1(\pmb{z}^*_2) \leq g_2(\pmb{z}^*_2)+\Delta_2 \leq 2 \Delta_2$, where the last inequality is by definition of $\mathbf{z}'$, which completes the proof.
\end{proof}

\subsection{Proof of Theorem \ref{thm:SFAME:dualbound}}
\begin{proof}
We give a constructive proof via  explicitly giving the value of $p^*$. In fact, let $p^*$ and $\pmb q^*$ be the optimal solution to  
\begin{equation}\label{prob:SFAME:onlineproofdual}
    \begin{split}
        \min_{p, \pmb q:} \textit{ }&  \hat{b} p +   \sum_{n=1}^{N} \pmb q_n\\
        s.t. & \frac{1}{N} \hat{\pmb c}_k p + \pmb q_n \geq \frac{1}{N} \hat{\pmb a}_k(x_n) \\
        & p, \pmb q  \geq 0 \\
    \end{split}
\end{equation}
Then our goal is to show that for this constructed $p^*$,  $s^{p^*}$ is a feasible solution to Problem \ref{prob:SFAME:optimaldefinition} and $r(s^{p^*})\geq r(s^*)-\frac{1}{N}$ with probability 1 (Since probabilistic statement is only introduced in Lemma \ref{lemma:SFAME:ZeroNormBound} whose result holds with probability 1, and we only apply it finite times, we will omit the probabilistic statement for the rest of the proof for simplicity). 
To achieve this, let us construct a $N\times K$ matrix
\begin{equation*}
    \tilde{\pmb Z}^{p^*}_{n,k} \triangleq \mathbbm{1}_{  {{s}}^{p^*}(x_n)=k}
\end{equation*}
It is not hard to see that $    {{s}}^{p^*}(x_n) = \arg \max_k \tilde{\pmb Z}^{p^*}_{n,k}$.
By construction of $s^{p^*}$, feasibility of $s^{p^*}$  to Problem \ref{prob:SFAME:optimaldefinition} is equivalent to feasibility of 
$\tilde{\pmb Z}^{p^*}$ to Problem \ref{Prob:SFAME:ILPTransform}.
By construction of $s^*$ and $s^{p^*}$, $r(s^{p^*})\geq r(s^*)-\frac{1}{N}$ is  equivalent to
 \begin{equation*}
    \begin{split}
       \frac{1}{N} \sum_{n=1}^{N} \sum_{k=1}^{K} \tilde{\pmb{Z}}^{p^*}_{n,k} \pmb {\hat{a}}_{k}(x_n) \geq \frac{1}{N} \sum_{n=1}^{N} \sum_{k=1}^{K} \pmb{Z}^{*}_{n,k} \pmb {\hat{a}}_{k}(x_n) -    \frac{1}{N}.
    \end{split}
\end{equation*}
Therefore, our goal becomes showing the feasibility of $\tilde{\pmb Z}^{p^*}$ and the above inequality. 
By construction of $\tilde{\pmb Z}^{p^*}$, the natural constraints ($\tilde{\pmb Z}^{p^*}_{n,k}\in \{0,1\}$  and $\sum_{k=1}^{K} \tilde{\pmb Z}^{p^*}_{n,k}=1,\forall n$) are obviously satisfied.
Thus, we only need to show $\tilde{\pmb Z}^{p^*}$ satisfies the budget constraint and the above inequality. 
To show those two results, let us  introduce another variable $\pmb Z^{*,LP}$, which represents a sparse optimal solution to the relaxed version of Problem \ref{Prob:SFAME:ILPTransform} (i.e., Problem \ref{prob:SFAME:APISelectionLP}).
The proof idea is then (roughly) to show (i) that  $\tilde{\pmb Z}^{p^*}$ is actually close to ${\pmb Z}^{*,LP}$, (ii) that ${\pmb Z}^{*,LP}$ satisfies the budget constraint and gives an estimated accuracy as high as that of the optimal solution $\pmb Z^*$, and (iii) that the difference between $\tilde{\pmb Z}^{p^*}$ and ${\pmb Z}^{*,LP}$ does not break the budget constraints and only decreases the estimated accuracy by $1/N$. 
Combining the three points finishes the proof.
Now we formalize this idea. 

Step 1: We first show that $\tilde{\pmb Z}^{p^*}$ and ${\pmb Z}^{*,LP}$ are close to each other. 
\begin{lemma}\label{SFAME:lemma:closetolinear}
Let $\pmb Z^{*,LP}$ be an optimal solution to Problem \ref{prob:SFAME:APISelectionLP}.
Then  there exists some constant $ n'$, such that $\tilde{\pmb Z}_{n,\cdot}^{p^*} = \pmb Z^{*,LP}_{n,\cdot}, \forall n\not=n'$.
\end{lemma}
\begin{proof}
Note that Problem \ref{prob:SFAME:onlineproofdual} is the dual problem to Problem \ref{prob:SFAME:APISelectionLP}. We can write the complementary slackness constraints as follows 
\begin{equation*}
    \begin{split}
        \pmb Z_{n,k}^{*LP} (\frac{1}{N} \hat{\pmb c}_{k} p^* + \pmb q_{n}^* - \frac{1}{N}\hat{\pmb a}_{k}(x_n)) & = 0, \forall n,k
    \end{split}
\end{equation*}

Now let us construct the matrix 
\begin{equation*}
\pmb A = \begin{bmatrix}
\frac{1}{N}\hat{\pmb c}, & \pmb 1, &  \pmb 0, & \cdots , & \pmb 0 \\
\frac{1}{N}\hat{\pmb c}, & \pmb 0, &  \pmb 1, & \cdots , & \pmb 0 \\
\vdots, & ,\vdots, & \cdots, & \ddots, & \vdots\\
\frac{1}{N}\hat{\pmb c}, & \pmb 0, &  \pmb 0, & \cdots , & \pmb 1  
\end{bmatrix}
\in \R^{NK\times (N+1)} 
\end{equation*} and the vector
\begin{equation*}
\pmb \beta = \frac{1}{N} \begin{bmatrix}
\hat{\pmb a}(x_1) \\
\hat{\pmb a}(x_2)\\
\vdots\\
\hat{\pmb a}(x_N)
\end{bmatrix}
\in \R^{NK}. 
\end{equation*}
Then by Lemma
\ref{lemma:SFAME:ZeroNormBound}, $\min_{\pmb x} \|\pmb A \pmb x -  \pmb \beta\|_0 \geq NK-N-1$.
Specifically, if $\pmb x = [p^*, \pmb q^{*T} ]^T$,
then we should have $\|\pmb A \pmb x -  \pmb \beta\|_0 \geq NK-N-1$.
Note that each row of $\pmb A \pmb x -  \pmb \beta$ corresponds to 
$\frac{1}{N} \hat{\pmb c}_{k} p^* + \pmb q_{n}^* - \frac{1}{N}\hat{\pmb a}_{k}(x_n)$, and thus we effectively have \begin{equation*}
    \frac{1}{N} \hat{\pmb c}_{k} p^* + \pmb q_{n}^* - \frac{1}{N}\hat{\pmb a}_{k}(x_n) \not=0
\end{equation*}
for at least $NK-N-1$ choices of $n,k$.
In other words, among all possible choices of $n,k$, at most $N+1$ many of them satisfies
\begin{equation*}
    \frac{1}{N} \hat{\pmb c}_{k} p^* + \pmb q_{n}^* - \frac{1}{N}\hat{\pmb a}_{k}(x_n) =0
\end{equation*}
Furthermore, note that the constraint 
 $\sum_{k=1}^{K}\pmb z^{*LP}_k(x_n) =1$ ensures that for any $n$, there must exist at least one $k'$ such that $\pmb Z^{*,LP}_{n,k} \not=0$ and thus $\frac{1}{N}\pmb c_{k'} p^* + \pmb q_{n}^* - \frac{1}{N}\hat{\pmb a}_{k'}(x_n)=0$.
By the pigeonhole principle, we can conclude that for all $n$ except one (denoted by $n'$), exactly one equation in  $\{\frac{1}{N}\pmb c_{k} p^* + \pmb q_{n}^* - \frac{1}{N}\hat{\pmb a}_{k}(x_n)=0\}_{k}$ can be satisfied. 

Now let us fix any $n\not=n'$.
Then there exists some $k'$, such that $\frac{1}{N}\pmb c_{k'} p^* + \pmb q_{n}^* - \frac{1}{N}\hat{\pmb a}_{k'}(x_n)=0$, and for any $k\not=k'$, $\frac{1}{N}\pmb c_{k} p^* + \pmb q_{n}^* - \frac{1}{N}\hat{\pmb a}_{k}(x_n)>0$ (due to the natural constraint in Problem \ref{prob:SFAME:onlineproofdual}). 
That is to say, for any $k\not=k'$,
\begin{equation*}
    \frac{1}{N}\pmb c_{k} p^* + \pmb q_{n}^* - \frac{1}{N}\hat{\pmb a}_{k}(x_n)> 0 = \frac{1}{N}\pmb c_{k'} p^* + \pmb q_{n}^* - \frac{1}{N}\hat{\pmb a}_{k'}(x_n) 
\end{equation*}
Multiplying $N$ and rearranging the terms gives 
\begin{equation*}
\hat{\pmb a}_{k'}(x_n) -\pmb c_{k'} p^* > \hat{\pmb a}_{k}(x_n) - \pmb c_{k} p^*
\end{equation*}
That is to say, $k'$ is the unique solution to $\max_{k} \hat{\pmb a}_{k}(x_n) - \pmb c_{k} p^*$.
By definition of $\tilde{\pmb Z}^{p^*}$, we have $\tilde{\pmb Z}^{p^*}_{n,k'}=1$ and $\tilde{\pmb Z}^{p^*}_{n,k}=0, \forall k\not=k'$. 
Meanwhile, for any $k\not=k'$, by the slackness constraint, since  , $\frac{1}{N}\pmb c_{k} p^* + \pmb q_{n}^* - \frac{1}{N}\hat{\pmb a}_{k}(x_n)>0$,  we must have $\pmb Z^{*,LP}_{n,k} = 0$.  The natural constraint in Problem \ref{prob:SFAME:APISelectionLP} requires $\sum_{k=1}^{K} \pmb Z^{*,LP}_{n,k}=1$.
Thus, we have $\pmb Z^{*,LP}_{n,k'} = \sum_{k=1}^{K} \pmb Z^{*,LP}_{n,k} - \sum_{k\not=k'} \pmb Z^{*,LP}_{n,k} = 1$.

That is to say, for any $n\not=n'$, we always have $\tilde{\pmb Z}_{n,\cdot}^{p^*} = \pmb Z^{*,LP}_{n,\cdot}$, which completes the proof.
\end{proof}

Step 2: Now we can show 
$\frac{1}{N} \sum_{n=1}^{N} \sum_{k=1}^{K} \tilde{\pmb{Z}}^{p^*}_{n,k} \pmb {\hat{a}}_{k}(x_n) \geq \frac{1}{N} \sum_{n=1}^{N} \sum_{k=1}^{K} \pmb{Z}^{*}_{n,k} \pmb {\hat{a}}_{k}(x_n) -  \frac{1}{N}$.
To see this, by Lemma \ref{SFAME:lemma:closetolinear}, $\tilde{\pmb Z}_{n,\cdot}^{p^*} = \pmb Z^{*,LP}_{n,\cdot}, \forall n\not=n'$, we must have 
\begin{equation*}
\frac{1}{N} \sum_{n=1}^{N} \sum_{k=1}^{K} {\pmb{Z}}^{*,LP}_{n,k} \pmb {\hat{a}}_{k}(x_n) - \frac{1}{N} \sum_{n=1}^{N} \sum_{k=1}^{K} \tilde{ \pmb{Z}}^{p^*}_{n,k} \pmb {\hat{a}}_{k}(x_n) = \frac{1}{N} \sum_{k=1}^{K} {\pmb{Z}}^{*,LP}_{n',k} \pmb {\hat{a}}_{k}(x_n') - \frac{1}{N}  \sum_{k=1}^{K} \tilde{ \pmb{Z}}^{p^*}_{n',k} \pmb {\hat{a}}_{k}(x_n')
\end{equation*}
As ${\hat{a}}_{k}(x_n')$ is bounded in $[0,1]$, we have 
\begin{equation*}
    \frac{1}{N} \sum_{k=1}^{K} {\pmb{Z}}^{*,LP}_{n',k} \pmb {\hat{a}}_{k}(x_n') - \frac{1}{N}  \sum_{k=1}^{K} \tilde{ \pmb{Z}}^{p^*}_{n',k} \pmb {\hat{a}}_{k}(x_n') \leq     \frac{1}{N} \sum_{k=1}^{K} {\pmb{Z}}^{*,LP}_{n',k} \cdot 1 - \frac{1}{N}  \sum_{k=1}^{K} \tilde{ \pmb{Z}}^{p^*}_{n',k} \cdot 0 =  \frac{1}{N} \sum_{k=1}^{K} {\pmb{Z}}^{*,LP}_{n',k} 
\end{equation*}
By natural constraint in Problem \ref{prob:SFAME:APISelectionLP}, 
$\sum_{k=1}^{K} {\pmb{Z}}^{*,LP}_{n',k}=1$. Thus, we have 
\begin{equation*}
 \frac{1}{N} \sum_{n=1}^{N} \sum_{k=1}^{K} {\pmb{Z}}^{*,LP}_{n,k} \pmb {\hat{a}}_{k}(x_n) - \frac{1}{N} \sum_{n=1}^{N} \sum_{k=1}^{K} \tilde{ \pmb{Z}}^{p^*}_{n,k} \pmb {\hat{a}}_{k}(x_n) \leq  \frac{1}{N} \sum_{k=1}^{K} {\pmb{Z}}^{*,LP}_{n',k}  =      \frac{1}{N} 
\end{equation*}
 On the other hand, $\pmb Z^{*,LP}$ is the optimal solution to Problem \ref{prob:SFAME:APISelectionLP} and $\pmb Z^{*}$ is a feasible solution.
Thus we have 
\begin{equation*}
 \frac{1}{N} \sum_{n=1}^{N} \sum_{k=1}^{K} \pmb{Z}^{*}_{n,k} \pmb {\hat{a}}_{k}(x_n) \leq \frac{1}{N} \sum_{n=1}^{N} \sum_{k=1}^{K} {\pmb{Z}}^{*,LP}_{n,k} \pmb {\hat{a}}_{k}(x_n)
\end{equation*}
Combining the two inequalities leads to \begin{equation*}
    \frac{1}{N} \sum_{n=1}^{N} \sum_{k=1}^{K} \tilde{\pmb{Z}}^{p^*}_{n,k} \pmb {\hat{a}}_{k}(x_n) \geq \frac{1}{N} \sum_{n=1}^{N} \sum_{k=1}^{K} \pmb{Z}^{*}_{n,k} \pmb {\hat{a}}_{k}(x_n) -  \frac{1}{N}
\end{equation*}

Step 3: Finally, we are ready to show the budget constraint is satisfied. 
By Lemma \ref{SFAME:lemma:closetolinear}, $\tilde{\pmb Z}_{n,\cdot}^{p^*} = \pmb Z^{*,LP}_{n,\cdot}, \forall n\not=n'$, we have
\begin{equation*}
    \begin{split}
\frac{1}{N} \sum_{n=1}^{N} \sum_{k=1}^{K}     \tilde{\pmb Z}_{n,k}^{p^*} \hat{\pmb c}_k - \frac{1}{N} \sum_{n=1}^{N} \sum_{k=1}^{K}     {\pmb Z}_{n,k}^{*,LP} \hat{\pmb c}_k = \frac{1}{N} \sum_{k=1}^{K}     \tilde{\pmb Z}_{n',k}^{p^*} \hat{\pmb c}_k - \frac{1}{N} \sum_{k=1}^{K}     {\pmb Z}_{n',k}^{*,LP} \hat{\pmb c}_k
    \end{split}
\end{equation*}
Denote $s^{p^*}(x_{n'})$ by $k_1$. By construction, we have 
\begin{equation*}
    \begin{split}
\sum_{k=1}^{K}     \tilde{\pmb Z}_{n',k}^{p^*} \hat{\pmb c}_k - \sum_{k=1}^{K}     {\pmb Z}_{n',k}^{*,LP} \hat{\pmb c}_k =  \hat{\pmb c}_{k_1} - \sum_{k=1}^{K}     {\pmb Z}_{n',k}^{*,LP} \hat{\pmb c}_k
    \end{split}
\end{equation*}

Let $S$ be the set of any $k$ such that $\pmb Z^{*,LP}_{n',k}\not=0$. 
Then we can further write
\begin{equation*}
    \begin{split}
\sum_{k=1}^{K}     \tilde{\pmb Z}_{n',k}^{p^*} \hat{\pmb c}_k - \sum_{k=1}^{K}     {\pmb Z}_{n',k}^{*,LP} \hat{\pmb c}_k =  \hat{\pmb c}_{k_1} - \sum_{k\in S}     {\pmb Z}_{n',k}^{*,LP} \hat{\pmb c}_k
    \end{split}
\end{equation*}
Note that $k\in S$ implies $k\in \arg \max \hat{\pmb a}_k(x_{n'})-\pmb {\hat{c}}_k p^*$ (Suppose not. Then there exists some $k'$, such that $ \hat{\pmb a}_{k'}(x_{n'})-\pmb {\hat{c}}_{k'} p^*>\hat{\pmb a}_k(x_{n'})-\pmb {\hat{c}}_k p^*$. Multiplying both sides by $-\frac{1}{N}$ and then adding $\pmb q^*_n$ gives  $ - \frac{1}{N} \hat{\pmb a}_{k'}(x_{n'}) + \frac{1}{N} \pmb {\hat{c}}_{k'} p^* + \pmb q_n^* < - \frac{1}{N} \hat{\pmb a}_k(x_{n'})+\frac{1}{N} \pmb {\hat{c}}_k p^*+ \pmb q_n^*$. By complementary slackness of Problem \ref{prob:SFAME:APISelectionLP}, $\pmb Z^{*,LP}_{n,k} (- \frac{1}{N} \hat{\pmb a}_k(x_{n'})+\frac{1}{N} \pmb {\hat{c}}_k p^*+ \pmb q_n^*) = 0 $. $k\in S$ implies $\pmb Z^{*,LP}_{n,k}\not=0$ and thus $- \frac{1}{N} \hat{\pmb a}_k(x_{n'})+\frac{1}{N} \pmb {\hat{c}}_k p^*+ \pmb q_n^*=0$. Thus,  $ - \frac{1}{N} \hat{\pmb a}_{k'}(x_{n'}) + \frac{1}{N} \pmb {\hat{c}}_{k'} p^* + \pmb q_n^* < 0$, which contradicts with the feasibility constraint in the dual problem.).
Recall that $k_1$ is determined by $\arg \max_k \hat{\pmb a}_k(x_{n'})-\pmb {\hat{c}}_k p^*$ and we break ties by picking $k$ with smallest cost.
Thus, for any $k\in S$, $\hat{\pmb c}_k \geq \hat{\pmb c}_{k_1}$. 
Therefore, 
\begin{equation*}
    \begin{split}
\sum_{k=1}^{K}     \tilde{\pmb Z}_{n',k}^{p^*} \hat{\pmb c}_k - \sum_{k=1}^{K}     {\pmb Z}_{n',k}^{*,LP} \hat{\pmb c}_k \leq   \hat{\pmb c}_{k_1} - \sum_{k\in S}     {\pmb Z}_{n',k}^{*,LP} \hat{\pmb c}_{k_1} = (1 - \sum_{k\in S}     {\pmb Z}_{n',k}^{*,LP})  \hat{\pmb c}_{k_1}
    \end{split}
\end{equation*}
By feasibility constraint in Problem \ref{prob:SFAME:APISelectionLP}, $\sum_{k\in S}     {\pmb Z}_{n',k}^{*,LP}= \sum_{k=1}^{K}     {\pmb Z}_{n',k}^{*,LP}=1$.
Thus, the above inequality becomes $\sum_{k=1}^{K}     \tilde{\pmb Z}_{n',k}^{p^*} \hat{\pmb c}_k - \sum_{k=1}^{K}     {\pmb Z}_{n',k}^{*,LP} \hat{\pmb c}_k \leq 0$. 
Thus,
\begin{equation*}
    \begin{split}
\frac{1}{N} \sum_{n=1}^{N} \sum_{k=1}^{K}     \tilde{\pmb Z}_{n,k}^{p^*} \hat{\pmb c}_k - \frac{1}{N} \sum_{n=1}^{N} \sum_{k=1}^{K}     {\pmb Z}_{n,k}^{*,LP} \hat{\pmb c}_k = \frac{1}{N} \sum_{k=1}^{K}     \tilde{\pmb Z}_{n',k}^{p^*} \hat{\pmb c}_k - \frac{1}{N} \sum_{k=1}^{K}     {\pmb Z}_{n',k}^{*,LP} \hat{\pmb c}_k \leq 0
\end{split}
\end{equation*}
$\pmb Z^{*,LP}$ is a feasible solution to Problem \ref{prob:SFAME:APISelectionLP}, so it must satisfy the budget constraint and thus $\frac{1}{N} \sum_{n=1}^{N} \sum_{k=1}^{K}     {\pmb Z}_{n,k}^{*,LP} \hat{\pmb c}_k\leq b$.
Hence, we must have
\begin{equation*}
    \begin{split}
\frac{1}{N} \sum_{n=1}^{N} \sum_{k=1}^{K}     \tilde{\pmb Z}_{n,k}^{p^*} \hat{\pmb c}_k  \leq \frac{1}{N} \sum_{n=1}^{N} \sum_{k=1}^{K}     {\pmb Z}_{n,k}^{*,LP} \hat{\pmb c}_k\leq b \leq b
\end{split}
\end{equation*}
i.e., $\tilde{\pmb Z}^{p^*}$ satisfies the budget constraint in Problem \ref{Prob:SFAME:ILPTransform}.

Finally, combining step 2 and step 3 finishes the proof.

\end{proof}

\subsection{Proof of Theorem \ref{thm:SFAME:onlinebound}}
\begin{proof}
Let us first establish a few lemmas consisting of the main components of the proof.
\begin{lemma}\label{lemma:SFAME:convergencefeasibility}
Suppose $\delta \geq \frac{\|\pmb c \|_{\infty}}{b} \left[\sqrt{\frac{\log 4-\log \epsilon}{N}} + \sqrt{\frac{\log 4-\log \epsilon}{N^{Tr}}} \right]$. Then with probability at least $1-\epsilon$, ${{s}}^{\hat{p}}$ is a feasible solution to Problem \ref{prob:SFAME:optimaldefinition}.
\end{lemma}
\begin{proof}

We first note that Problem \ref{prob:SFAME:APISelectionTrainingDual} is  a linear programming, and its dual problem is
\begin{equation}\label{prob:SFAME:APISelectionLPTraining}
    \begin{split}
        \max_{\pmb Z \in \R^{N\times K}:} \textit{ }&  \frac{1}{N^{Tr}} \sum_{n=1}^{N^{Tr}} \pmb{Z}_{n,k} \pmb {\hat{a}}_{k}(x_n^{Tr})\\
        s.t. & \frac{1}{N^{Tr}} \sum_{n=1}^{N^{Tr}} \sum_{k=1}^{K} \pmb Z_{n,k} \hat{\pmb c}_{k} \leq  (1-\delta)\hat{b} \\
        & \sum_{k=1}^{K} \pmb{Z}_{n,k} = 1,       \pmb{Z}_{n,k} \in [0,1], \forall n, k
    \end{split}
\end{equation}
Note that this is in the same form of Problem \ref{prob:SFAME:APISelectionLP} except that the data become $\{x_n^{Tr}\}_{n=1}^{N^{Tr}}$ instead of $\{x_n\}_{n=1}^N$.
Using a similar argument in the proof for Theorem \ref{thm:SFAME:dualbound}, ${{s}}^{\hat{p}}(x_n^{Tr})$ is a feasible solution to \begin{equation*}
\begin{split}
   \max_{} &  \frac{1}{N^{Tr}} \sum_{n=1}^{N^{Tr}} r^{{{s}}^{p}}(x_n^{Tr})\\
    s.t. & \frac{1}{N^{Tr}} \sum_{n=1}^{N^{Tr}}\eta^{[{{s}}^{p}]}(x_n^{Tr}, \pmb c) \leq (1-\delta) b,
\end{split}
\end{equation*}
and thus we have
\begin{equation*}
\frac{1}{N^{Tr}} \sum_{n=1}^{N^{Tr}}\eta^{[{{s}}^{p}]}(x_n^{Tr}, \pmb c) \leq (1-\delta) b    
\end{equation*}
Note that training data $x_n^{Tr}$ are i.i.d samples from the true distribution and $0 \leq \eta^{[s]}(x_n^{Tr},\pmb c)\leq \|\pmb c\|_\infty$.
Thus, by Hoeffding's inequality,  with probability $1-\epsilon$, we have
\begin{equation*}
\|  \frac{1}{N^{Tr}} \sum_{n=1}^{N^{Tr}}\eta^{[{{s}}^{p}]}(x_n^{Tr}, \pmb c) - \Exp\left[\eta^{[{{s}}^{p}]}(x, \pmb c)\right] \| \leq \|\pmb c\|_\infty \sqrt{\frac{\log 2 -\log \epsilon}{2 N^{Tr}}}
\end{equation*}
The data stream $x_n$ is also from the same distribution, and thus we also have with probability $1-\epsilon$,
\begin{equation*}
\|  \frac{1}{N} \sum_{n=1}^{N^{}}\eta^{[{{s}}^{p}]}(x_n^{}, \pmb c) - \Exp\left[\eta^{[{{s}}^{p}]}(x, \pmb c)\right] \| \leq \|\pmb c\|_\infty \sqrt{\frac{\log 2 -\log \epsilon}{2 N^{}}}
\end{equation*}
Applying union bound, we have with probability $1-\epsilon$,

\begin{equation*}
\|  \frac{1}{N^{Tr}} \sum_{n=1}^{N^{Tr}}\eta^{[{{s}}^{p}]}(x_n^{Tr}, \pmb c) - \Exp\left[\eta^{[{{s}}^{p}]}(x, \pmb c)\right] \| \leq \|\pmb c\|_\infty \sqrt{\frac{\log 4 -\log \epsilon}{2 N^{Tr}}}
\end{equation*}
and
\begin{equation*}
\|  \frac{1}{N} \sum_{n=1}^{N^{}}\eta^{[{{s}}^{p}]}(x_n^{}, \pmb c) - \Exp\left[\eta^{[{{s}}^{p}]}(x, \pmb c)\right] \| \leq \|\pmb c\|_\infty \sqrt{\frac{\log 4 -\log \epsilon}{2 N^{}}}
\end{equation*}
Using triangle inequality, we have with probability $1-\epsilon$,
\begin{equation*}
\|  \frac{1}{N^{Tr}} \sum_{n=1}^{N^{Tr}}\eta^{[{{s}}^{p}]}(x_n^{Tr}, \pmb c) - \frac{1}{N} \sum_{n=1}^{N^{}}\eta^{[{{s}}^{p}]}(x_n^{}, \pmb c) \| \leq \|\pmb c\|_\infty \sqrt{\frac{\log 4 -\log \epsilon}{2 N^{}}} + \|\pmb c\|_\infty \sqrt{\frac{\log 4 -\log \epsilon}{2 N^{Tr}}}.
\end{equation*}
Thus we have 
\begin{equation*}
\begin{split}
\frac{1}{N} \sum_{n=1}^{N^{}}\eta^{[{{s}}^{p}]}(x_n^{}, \pmb c)  \leq & \frac{1}{N^{Tr}} \sum_{n=1}^{N^{Tr}}\eta^{[{{s}}^{p}]}(x_n^{Tr}, \pmb c) + \|\pmb c\|_\infty \sqrt{\frac{\log 4 -\log \epsilon}{2 N^{}}} + \|\pmb c\|_\infty \sqrt{\frac{\log 4 -\log \epsilon}{2 N^{Tr}}}\\
\leq & (1-\delta) b + \|\pmb c\|_\infty \sqrt{\frac{\log 4 -\log \epsilon}{2 N^{}}} + \|\pmb c\|_\infty \sqrt{\frac{\log 4 -\log \epsilon}{2 N^{Tr}}} \leq b
\end{split}
\end{equation*}
where the last inequality is due to the assumption on $\delta$.
That is to say, with probability $1-\epsilon$, ${{s}}^{\hat{p}}$ is a feasible solution to Problem \ref{prob:SFAME:optimaldefinition}, which completes the proof. 
\end{proof}

\begin{lemma}\label{lemma:SFAME:DiscretizeBound} 
Construct the set $\Omega_{M}\triangleq\{0, \frac{1}{(M-1)\min_{\pmb c_{k}\not=0} \pmb c_{k}}, \frac{2}{(M-1)\min_{\pmb c_{k}\not=0} \pmb c_{k}}, \cdots, \frac{1}{\min_{\pmb c_{k}\not=0} \pmb c_{k}}\}$
and 
\begin{equation*}
\begin{split}
    \hat{p}(\Omega_M) \triangleq \arg \max_{p\in \Omega_{M}} &  \frac{1}{N^{Tr}} \sum_{n=1}^{N^{Tr}} r^{{{s}}^{p}}(x_n^{Tr}) \\ 
    s.t. & \frac{1}{N^{Tr}} \eta^{[{{s}}^{p}]}(x_n, \pmb c) \leq (1-\delta) b.
\end{split}
\end{equation*}
Then with probability $1-\epsilon$,  \begin{equation*}
\begin{split}
\| \frac{1}{N^{}} \sum_{n=1}^{N^{}} r^{{{s}}^{\hat{p}}}(x_n^{}) -     \frac{1}{N^{}} \sum_{n=1}^{N} r^{{{s}}^{\hat{p}(\Omega_M)}}(x_n^{}) \| 
&\leq O(\sqrt{\frac{\log N + \log 8 -\log \epsilon}{2 N^{}}} + \sqrt{\frac{\log N^{Tr} + \log 8 -\log \epsilon}{2 N^{Tr}}}).
\end{split}
\end{equation*}
\end{lemma}
\begin{proof}
Note that $\Omega_M \subseteq \R$.
Consider an element $p \in \R$. 

(i) $p \geq \frac{1}{\min_{\pmb c_{k}\not=0}\pmb c_k}$: This effectively means the API with the smallest cost is always selected. In other words, we always have
\begin{equation*}
    bs =\arg \max \hat{\pmb a}_{k}(x) - p \hat{\pmb c}_{k}
\end{equation*}
To see this, simply note that for any other $k_1$, we have \begin{equation*}
\begin{split}
    \hat{\pmb a}_{bs}(x) - p \hat{\pmb c}_{bs} -      (\hat{\pmb a}_{k_1}(x) - p \hat{\pmb c}_{k_1}) = & \hat{\pmb a}_{bs}(x) - \hat{\pmb a}_{k_1}(x) + p (\hat{\pmb c}_{k_1} - \hat{\pmb c}_{bs}) = \hat{\pmb a}_{bs}(x) - \hat{\pmb a}_{k_1}(x) + p \hat{\pmb c}_{k_1}\\
    \geq & 0-1 + p \hat{\pmb c}_{k_1} \geq -1 + \hat{\pmb c}_{k_1} \cdot \frac{1}{\min_{\pmb c_{k}\not=0}\pmb c_k} \geq 0
    \end{split}
\end{equation*}
Thus, for such $p$, the objective value is the same as that for $\frac{1}{\min_{\pmb c_{k}\not=0}\pmb c_k} \in \Omega_{M}$.

(ii): $0 \leq p \leq \frac{1}{\min_{\pmb c_{k}\not=0}\pmb c_k}$:
By construction of $\Omega_M$, there exists some $m$, such that 
$\frac{m}{(M-1)\min_{\pmb c_{k}\not=0}\pmb c_k}\leq p\leq \frac{m+1}{(M-1)\min_{\pmb c_{k}\not=0}\pmb c_k}$.
Let $p_j \triangleq \frac{j}{(M-1)\min_{\pmb c_{k}\not=0}\pmb c_k}$ for ease of notations.
Clearly, we have $p_m \in \Omega_M$. 

Now let us partition the space of $\hat{\pmb a}(x)$ into $M$ regions, denoted by $A_1, A_2, \cdots, A_{M}$.
Abusing the notation a little bit, let $\phi(p,x)\triangleq \arg \max{\hat{\pmb a}_k(x)- p}\hat{\pmb c}_k$
$A_1$ is the set of all $\hat{\pmb a}(x)$ such that $\phi(p,x)$ is a constant.
$A_2$ is the set of all $\hat{\pmb a}(x)$ such that $\phi(p,x)$ is a constant for $p$ larger than $p_1$.
Generally, 
$A_j$ is the set of all $\hat{\pmb a}(x)$ such that $\phi(p,x)$ is a constant for $p$ larger than $p_{j-1}$ subtracting $A_{j-1}$.
Formally,
\begin{equation*}
A_j=\begin{cases}
\{\hat{\pmb a}(x): \phi(p,x) \textit{is a constant} \}, &j=1\\
\{\hat{\pmb a}(x): \phi(p,x) \textit{is a constant if } p\geq p_{j-1}  \} - A_j,&j> 1
\end{cases}
\end{equation*}
One can easily verify that 
$\{A_j\}$ form a partition of the space of the estimated accuracy, and further more, $\|A_j\| \leq \frac{\|\pmb c\|_1}{M \min_{\pmb c_k\not=0}\pmb c_k}$.
By the assumption of the distribution, there exists some constant $u$, such that 
$Pr(A) \leq u \|A\|$, for any $A$ in the probability space. Thus, we must have
\begin{equation*}
    \Pr[\hat{\pmb a}(x) \in A_j] \leq \|A_j\| u = \frac{u \|\pmb c\|_1}{M \min_{\pmb c_k\not=0}\pmb c_k^2}
\end{equation*}
Now note that, when 
$p_m = \frac{m}{(M-1)\min_{\pmb c_{k}\not=0}\pmb c_k}\leq p\leq \frac{m+1}{(M-1)\min_{\pmb c_{k}\not=0}\pmb c_k}=p_{m+1}$, only elements in $A_m$ may affect the reward.
More precisely, we have
\begin{equation*}
    \frac{1}{N} \sum_{n=1}^{N} r^{{{s}}^{p}}(x_n) -     \frac{1}{N} \sum_{n=1}^{N} r^{{{s}}^{p_m}}(x_n) = \frac{1}{N} \sum_{x_n \in A_m}^{} r^{{{s}}^{p}}(x_n) -     \frac{1}{N} \sum_{x_n \in A_m}^{} r^{{{s}}^{p_m}}(x_n)
\end{equation*}
Note that each estimated accuracy is an i.i.d sample from the true distribution, and its value is from $[0,1]$,
by Hoeffding's inequality,  with probability $1-\epsilon$, we have
\begin{equation*}
\|  \frac{1}{N} \sum_{n=1}^{N} \mathbbm{1}_{x_n \in A_j} - \Pr[x_n \in A_j]\| \leq \sqrt{\frac{\log 2 -\log \epsilon}{2 N^{}}}
\end{equation*}
Applying the union bound, we have for any $j$, with probability $1-\epsilon$,

\begin{equation*}
\|  \frac{1}{N} \sum_{n=1}^{N} \mathbbm{1}_{x_n \in A_j} - \Pr[x_n \in A_j]\| \leq \sqrt{\frac{\log M + \log 2 -\log \epsilon}{2 N^{}}}
\end{equation*}
Therefore, we have with probability $1-\epsilon$,
\begin{equation*}
\begin{split}
\frac{1}{N} \sum_{n=1}^{N} r^{{{s}}^{p}}(x_n) -     \frac{1}{N} \sum_{n=1}^{N} r^{{{s}}^{p_m}}(x_n)&= \frac{1}{N} \sum_{x_n \in A_m}^{} r^{{{s}}^{p}}(x_n) -     \frac{1}{N} \sum_{x_n \in A_m}^{} r^{{{s}}^{p_m}}(x_n) \\
&\geq \sum_{x_n \in A_m}^{} 0 -     \frac{1}{N} \sum_{x_n \in A_m}^{} 1 = \frac{1}{N} \sum_{n=1}^{N} \mathbbm{1}_{x_n \in A_m}\\
&\geq \Pr[x_n \in A_m] - \sqrt{\frac{\log M + \log 2 -\log \epsilon}{2 N^{}}} \\
&\geq - \sqrt{\frac{\log M + \log 2 -\log \epsilon}{2 N^{}}}
\end{split}
\end{equation*}
and similarly
\begin{equation*}
\begin{split}
\frac{1}{N} \sum_{n=1}^{N} r^{{{s}}^{p}}(x_n) -     \frac{1}{N} \sum_{n=1}^{N} r^{{{s}}^{p_m}}(x_n)&= \frac{1}{N} \sum_{x_n \in A_m}^{} r^{{{s}}^{p}}(x_n) -     \frac{1}{N} \sum_{x_n \in A_m}^{} r^{{{s}}^{p_m}}(x_n) \\
&\leq \sum_{x_n \in A_m}^{} 1 -     \frac{1}{N} \sum_{x_n \in A_m}^{} 0 = \frac{1}{N} \sum_{n=1}^{N} \mathbbm{1}_{x_n \in A_m}\\
&\leq \Pr[x_n \in A_m] + \sqrt{\frac{\log M + \log 2 -\log \epsilon}{2 N^{}}} \\
&\leq \frac{u \|\pmb c\|_1}{M\min_{\pmb c_k\not=0}\pmb c_k} + \sqrt{\frac{\log M + \log 2 -\log \epsilon}{2 N^{}}}
\end{split}
\end{equation*}
That is to say,
\begin{equation}\label{equation:SFAME:temp1}
\begin{split}
\| \frac{1}{N} \sum_{n=1}^{N} r^{{{s}}^{p}}(x_n) -     \frac{1}{N} \sum_{n=1}^{N} r^{{{s}}^{p_m}}(x_n) \| 
&\leq \frac{u\|\pmb c\|_1}{M\min_{\pmb c_k\not=0}\pmb c_k} + \sqrt{\frac{\log M + \log 2 -\log \epsilon}{2 N^{}}}
\end{split}
\end{equation}
Similarly, for the training dataset, we can also get,  with probability $1-\epsilon$,
\begin{equation*}
\begin{split}
\| \frac{1}{N^{Tr}} \sum_{n=1}^{N^{Tr}} r^{{{s}}^{p}}(x_n^{Tr}) -     \frac{1}{N^{Tr}} \sum_{n=1}^{N} r^{{{s}}^{p_m}}(x_n^{Tr}) \| 
&\leq \frac{u\|\pmb c\|_1}{M\min_{\pmb c_k\not=0}\pmb c_k} + \sqrt{\frac{\log M + \log 2 -\log \epsilon}{2 N^{}}}
\end{split}
\end{equation*}

Combining case (i) and case (ii), we have just shown that for any $p \in \R$, there exists another $p' \in \Omega_M$, such that 
\begin{equation}\label{equation:SFAME:discretizebound}
\begin{split}
\| \frac{1}{N^{Tr}} \sum_{n=1}^{N^{Tr}} r^{{{s}}^{p}}(x_n^{Tr}) -     \frac{1}{N^{Tr}} \sum_{n=1}^{N} r^{{{s}}^{p'}}(x_n^{Tr}) \| 
&\leq \frac{u \|\pmb c\|_1}{M\min_{\pmb c_k\not=0}\pmb c_k} + \sqrt{\frac{\log M + \log 2 -\log \epsilon}{2 N^{}}}
\end{split}
\end{equation}
Thus, applying Lemma \ref{lemma:SFAME:optimizationdiffdomain}, we have with probability $1-\epsilon$,
\begin{equation*}
\begin{split}
\| \frac{1}{N^{Tr}} \sum_{n=1}^{N^{Tr}} r^{{{s}}^{\hat{p}}}(x_n^{Tr}) -     \frac{1}{N^{Tr}} \sum_{n=1}^{N} r^{{{s}}^{\hat{p}(\Omega_M)}}(x_n^{Tr}) \| 
&\leq \frac{u \|\pmb c\|_1}{M\min_{\pmb c_k\not=0}\pmb c_k} + \sqrt{\frac{\log M + \log 2 -\log \epsilon}{2 N^{}}}
\end{split}
\end{equation*}

Now by Lemma \ref{Lemma:SFAME:traintestmatch}, for each fixed $j$, we have with probability $1-\epsilon$,

\begin{equation*}
    \| \frac{1}{N}\sum_{n=1}^{N} r^{{{s}}^{p_j}}(x_{n}) - \frac{1}{N^{Tr}}\sum_{n=1}^{N^{Tr}} r^{{{s}}^{p_j}}(x_{n}^{Tr})\| \leq 
\left[\sqrt{\frac{\log 4 - \log \epsilon }{2 N}} + \sqrt{\frac{\log 4 - \log \epsilon }{2 N^{Tr}}}\right]
\end{equation*}
Applying union bound, with probability $1-\epsilon$,
\begin{equation*}
    \| \frac{1}{N}\sum_{n=1}^{N} r^{{{s}}^{p_j}}(x_{n}) - \frac{1}{N^{Tr}}\sum_{n=1}^{N^{Tr}} r^{{{s}}^{p_j}}(x_{n}^{Tr})\| \leq 
\left[\sqrt{\frac{\log M + \log 4 - \log \epsilon }{2 N}} + \sqrt{\frac{\log M +  \log 4 - \log \epsilon }{2 N^{Tr}}}\right]
\end{equation*}
for all $j$.
Specifically, we have
\begin{equation}\label{equation:SFAME:temp2}
    \| \frac{1}{N}\sum_{n=1}^{N} r^{{{s}}^{\hat{p}(\Omega_M)}}(x_{n}) - \frac{1}{N^{Tr}}\sum_{n=1}^{N^{Tr}} r^{{{s}}^{\hat{p}(\Omega_M)}}(x_{n}^{Tr})\| \leq 
\left[\sqrt{\frac{\log M + \log 4 - \log \epsilon }{2 N}} + \sqrt{\frac{\log M +  \log 4 - \log \epsilon }{2 N^{Tr}}}\right]
\end{equation}
and 
\begin{equation}\label{equation:SFAME:temp3}
    \| \frac{1}{N}\sum_{n=1}^{N} r^{{{s}}^{p'}}(x_{n}) - \frac{1}{N^{Tr}}\sum_{n=1}^{N^{Tr}} r^{{{s}}^{p'}}(x_{n}^{Tr})\| \leq 
\left[\sqrt{\frac{\log M + \log 4 - \log \epsilon }{2 N}} + \sqrt{\frac{\log M +  \log 4 - \log \epsilon }{2 N^{Tr}}}\right]
\end{equation}
Now combining equations
\ref{equation:SFAME:temp1}, \ref{equation:SFAME:discretizebound}, \ref{equation:SFAME:temp2}, and  \ref{equation:SFAME:temp3} with triangle inequality, we have with probability $1-\epsilon$,
\begin{equation*}
\begin{split}
& \| \frac{1}{N^{}} \sum_{n=1}^{N^{}} r^{{{s}}^{\hat{p}}}(x_n^{}) -     \frac{1}{N^{}} \sum_{n=1}^{N} r^{{{s}}^{\hat{p}(\Omega_M)}}(x_n^{}) \| \\
\leq & \frac{u\|\pmb c\|_1}{4M\min_{\pmb c_k\not=0}\pmb c_k} + 4\sqrt{\frac{\log M + \log 8 -\log \epsilon}{2 N^{}}} + 2\sqrt{\frac{\log M + \log 8 -\log \epsilon}{2 N^{Tr}}}
\end{split}
\end{equation*}
Setting $M=\min \{N_{Tr},N\}$, we have 
\begin{equation*}
\begin{split}
\| \frac{1}{N^{}} \sum_{n=1}^{N^{}} r^{{{s}}^{\hat{p}}}(x_n^{}) -     \frac{1}{N^{}} \sum_{n=1}^{N} r^{{{s}}^{\hat{p}(\Omega_M)}}(x_n^{}) \| 
&\leq O(\sqrt{\frac{\log N + \log 8 -\log \epsilon}{2 N^{}}} + \sqrt{\frac{\log N^{Tr} + \log 8 -\log \epsilon}{2 N^{Tr}}})
\end{split}
\end{equation*}
which completes the proof.
\end{proof}

\begin{lemma}\label{lemma:SFAME:convergencetrain2testbound}
Let
\begin{equation*}
\begin{split}
    p(\Omega_M) \triangleq \arg \max_{p\in \Omega_{M}} &  \frac{1}{N} \sum_{n=1}^{N} r^{{{s}}^{p}}(x_n)\\
    s.t. & \frac{1}{N} \eta^{[{{s}}^{p}]}(x_n, \pmb c) \leq (1-\delta) b - 
\|\pmb c\|_{\infty} \left[\sqrt{\frac{\log 8 - \log \epsilon }{2 N}} + \sqrt{\frac{\log 8- \log \epsilon }{2 N^{Tr}}} \right]
\end{split}
\end{equation*}
Then with probability $1-\epsilon$, \begin{equation*}
    \frac{1}{N}\sum_{n=1}^{N} r^{{{s}}^{\hat{p}(\Omega_M)}}(x_n) \geq  \frac{1}{N}\sum_{n=1}^{N^{}} r^{{{s}}^{p(\Omega_M)}} - 
\sqrt{\frac{\log 8 - \log \epsilon }{2 N}} - \sqrt{\frac{\log 8 - \log \epsilon }{2 N^{Tr}}}
\end{equation*}
and 
\begin{equation*}
\frac{1}{N}\sum_{n=1}^{N} \eta^{[{{s}}^{\hat{p}(\Omega_M)}]}(x_n,\pmb c) \leq  (1-\delta) b + 2 \|\pmb c\|_{\infty} \left[
\sqrt{\frac{\log 8 - \log \epsilon }{2 N}} + \sqrt{\frac{\log 8- \log \epsilon }{2 N^{Tr}}}\right].
\end{equation*}
\end{lemma}
\begin{proof}
By Lemma \ref{Lemma:SFAME:traintestmatch}, with probability $1-\epsilon$, we have
\begin{equation*}
    \| \frac{1}{N}\sum_{n=1}^{N} r^{{{s}}^{p}}(x_n) - \frac{1}{N^{Tr}}\sum_{n=1}^{N^{Tr}} r^{{{s}}^{p}}(x_n^{Tr})\| \leq 
\sqrt{\frac{\log 4 - \log \epsilon }{2 N}} + \sqrt{\frac{\log 4 - \log \epsilon }{2 N^{Tr}}}.
\end{equation*}
and similarly, with probability $1-\epsilon$, 

\begin{equation*}
    \| \frac{1}{N}\sum_{n=1}^{N} \eta^{[{{s}}^{p}]}(x_n,\pmb c) - \frac{1}{N^{Tr}}\sum_{n=1}^{N^{Tr}} \eta^{[{{s}}^{p}]}(x_n^{Tr},\pmb c) \| \leq  \|\pmb c\|_{\infty} \left[
\sqrt{\frac{\log 4 - \log \epsilon }{2 N}} + \sqrt{\frac{\log 4 - \log \epsilon }{2 N^{Tr}}}\right].
\end{equation*}
which is the same as
\begin{equation*}
    \| \frac{1}{N}\sum_{n=1}^{N} \eta^{[{{s}}^{p}]}(x_n,\pmb c) -(1-\delta) b  - \left[\frac{1}{N^{Tr}}\sum_{n=1}^{N^{Tr}} \eta^{[{{s}}^{p}]}(x_n^{Tr},\pmb c) -(1-\delta) b\right]\| \leq  \|\pmb c\|_{\infty} \left[
\sqrt{\frac{\log 4 - \log \epsilon }{2 N}} + \sqrt{\frac{\log 4 - \log \epsilon }{2 N^{Tr}}}\right].
\end{equation*}
Now applying union bound, we have with probability $1-\epsilon$,
\begin{equation*}
    \| \frac{1}{N}\sum_{n=1}^{N} r^{{{s}}^{p}}(x_n) - \frac{1}{N^{Tr}}\sum_{n=1}^{N^{Tr}} r^{{{s}}^{p}}(x_n^{Tr})\| \leq 
\sqrt{\frac{\log 8 - \log \epsilon }{2 N}} + \sqrt{\frac{\log 8 - \log \epsilon }{2 N^{Tr}}}.
\end{equation*}
and 
\begin{equation*}
\| \frac{1}{N}\sum_{n=1}^{N} \eta^{[{{s}}^{p}]}(x_n,\pmb c) -(1-\delta) b  - \left[\frac{1}{N^{Tr}}\sum_{n=1}^{N^{Tr}} \eta^{[{{s}}^{p}]}(x_n^{Tr},\pmb c) -(1-\delta) b\right]\| \leq  \|\pmb c\|_{\infty} \left[
\sqrt{\frac{\log 8 - \log \epsilon }{2 N}} + \sqrt{\frac{\log 8- \log \epsilon }{2 N^{Tr}}}\right].
\end{equation*}
both hold. By Lemma \ref{Lemma:SFAME:OptimiziationDiff}, we can conclude that 
\begin{equation*}
    \frac{1}{N}\sum_{n=1}^{N} r^{{{s}}^{\hat{p}(\Omega_M)}}(x_n) \geq  \frac{1}{N}\sum_{n=1}^{N^{}} r^{{{s}}^{p(\Omega_M)}} - 
\sqrt{\frac{\log 8 - \log \epsilon }{2 N}} - \sqrt{\frac{\log 8 - \log \epsilon }{2 N^{Tr}}}.
\end{equation*}
and 
\begin{equation*}
\frac{1}{N}\sum_{n=1}^{N} \eta^{[{{s}}^{\hat{p}(\Omega_M)}]}(x_n,\pmb c) -(1-\delta) b  \leq  2 \|\pmb c\|_{\infty} \left[
\sqrt{\frac{\log 8 - \log \epsilon }{2 N}} + \sqrt{\frac{\log 8- \log \epsilon }{2 N^{Tr}}}\right].
\end{equation*}
with probability $1-\epsilon$, which completes the proof.

\end{proof}

\begin{lemma}\label{lemma:SFAME:discrete2testingbound}
For $\delta =  \Omega(\sqrt{\frac{\log 4-\log \epsilon}{N}} + \sqrt{\frac{\log 4-\log \epsilon}{N^{Tr}}})$, we have with probability $1-\epsilon$,
\begin{equation*}
\begin{split}
       & \frac{1}{N}\sum_{n=1}^{N} r^{{{s}}^{p(\Omega_M)}}(x_n) -  \frac{1}{N}\sum_{n=1}^{N} r^{{{s}}^{p^*}}(x_n)       \geq  - O(\sqrt{\frac{log N - \log \epsilon }{2 N}} + \sqrt{\frac{\log N  - \log \epsilon }{2 N^{Tr}}} ). 
\end{split}
\end{equation*}
\end{lemma}
\begin{proof}
Let $\tilde{p}({\Delta})$ be the optimal solution to the following problem 
\begin{equation*}
\begin{split}
      \max_{p\in \R} &  \frac{1}{N} \sum_{n=1}^{N} r^{{{s}}^{p}}(x_n)
    s.t.  \frac{1}{N} \eta^{[{{s}}^{p}]}(x_n, \pmb c) \leq b-  \Delta
\end{split}
\end{equation*}
On one hand, $p^*$ apparently is a feasible solution to the above problem with $\Delta=0$, so we must have
\begin{equation}\label{equation:SFAME:temp5}
    \frac{1}{N}\sum_{n=1}^{N}r^{s^{p^*}}(x_n)  \leq  \frac{1}{N}\sum_{n=1}^{N}r^{s^{\tilde{p}(0)}}(x_n) 
\end{equation}
Let $\Delta' = \delta + \| \pmb c\|_{\infty} \left[\sqrt{\frac{\log 8 - \log \epsilon }{2 N}} + \sqrt{\frac{\log 8- \log \epsilon }{2 N^{Tr}}} \right]$. Then $\tilde{p}(\Delta')$ corresponds to the following problem 
\begin{equation*}
\begin{split}
      \max_{p\in \R} &  \frac{1}{N} \sum_{n=1}^{N} r^{{{s}}^{p}}(x_n)\\
    s.t. & \frac{1}{N} \eta^{[{{s}}^{p}]}(x_n, \pmb c) \leq (1-\delta) b - 
\|\pmb c\|_{\infty} \left[\sqrt{\frac{\log 8 - \log \epsilon }{2 N}} + \sqrt{\frac{\log 8- \log \epsilon }{2 N^{Tr}}} \right]
\end{split}
\end{equation*}
Then using the same argument in the proof of Lemma \ref{lemma:SFAME:DiscretizeBound}, we have with probability $1-\epsilon$, 
\begin{equation}\label{equation:SFAME:temp6}
    \| \frac{1}{N}\sum_{n=1}^{N} r^{{{s}}^{\tilde{p}(\Delta')}}(x_n) -  \frac{1}{N}\sum_{n=1}^{N^{}} r^{{{s}}^{p(\Omega_M)}(x_n)}\| \leq  
\sqrt{\frac{log N + \log 8 + \log \epsilon }{2 N}} + \sqrt{\frac{\log N + \log 8 - \log \epsilon }{2 N^{Tr}}}.
\end{equation}
Furthermore, it is clear that $\tilde{p}(\Delta)$ is decreasingly-monotone with respect to  $\Delta$.
In fact, removing the budget by $\Delta_1$, at most $\frac{\Delta_1}{\min_{\pmb c_k>c_j}c_k-c_j}$ data's APIs need be changed, and thus incurs at most $\frac{\Delta_1}{\min_{\pmb c_k>c_j}c_k-c_j}$ accuracy decrease.
That is to say, we must have
\begin{equation}\label{equation:SFAME:temp7}
    \| \frac{1}{N}\sum_{n=1}^{N} r^{{{s}}^{\tilde{p}(\Delta')}}(x_n) -  \frac{1}{N}\sum_{n=1}^{N} r^{{{s}}^{\tilde{p}(0)}}(x_n)\| \leq  \frac{\Delta_1}{\min_{\pmb c_k> \pmb c_j} \pmb c_k-\pmb c_j}
\end{equation}
Now combining equations \ref{equation:SFAME:temp5}, \ref{equation:SFAME:temp6}, \ref{equation:SFAME:temp7}, we can obtain
\begin{equation*}
\begin{split}
       & \frac{1}{N}\sum_{n=1}^{N} r^{{{s}}^{p(\Omega_M)}}(x_n) -  \frac{1}{N}\sum_{n=1}^{N} r^{{{s}}^{p^*}}(x_n)\\
       = & \frac{1}{N}\sum_{n=1}^{N^{}} r^{{{s}}^{p(\Omega_M)}(x_n)} - \frac{1}{N}\sum_{n=1}^{N} r^{{{s}}^{\tilde{p}(\Delta')}}(x_n) \\
       + & \frac{1}{N}\sum_{n=1}^{N} r^{{{s}}^{\tilde{p}(\Delta')}}(x_n) -  \frac{1}{N}\sum_{n=1}^{N} r^{{{s}}^{\tilde{p}(0)}}(x_n) +   \frac{1}{N}\sum_{n=1}^{N}r^{s^{\tilde{p}(0)}}(x_n)  - \frac{1}{N}\sum_{n=1}^{N}r^{s^{p^*}}(x_n) \\
       \geq & - \sqrt{\frac{log N + \log 8 + \log \epsilon }{2 N}} - \sqrt{\frac{\log N^{Tr} + \log 8 - \log \epsilon }{2 N^{Tr}}} -  \frac{\Delta_1}{\min_{\pmb c_k> \pmb c_j} \pmb c_k-\pmb c_j} - 0 
\end{split}
\end{equation*}
When $\delta =  \Omega(\left[\sqrt{\frac{\log 4-\log \epsilon}{N}} + \sqrt{\frac{\log 4-\log \epsilon}{N^{Tr}}} \right])$, we have 
\begin{equation*}
\begin{split}
       & \frac{1}{N}\sum_{n=1}^{N} r^{{{s}}^{p(\Omega_M)}}(x_n) -  \frac{1}{N}\sum_{n=1}^{N} r^{{{s}}^{p^*}}(x_n)\\
       \geq & - O(\sqrt{\frac{log N - \log \epsilon }{2 N}} + \sqrt{\frac{\log N^{Tr}  - \log \epsilon }{2 N^{Tr}}}  ) 
\end{split}
\end{equation*}
which completes the proof.
\end{proof}

Now we are ready to prove the main theorem. We start by showing the bound on the reward.
Suppose $\delta = \Theta\left(\sqrt{\frac{\log N-\log \epsilon}{N}} + \sqrt{\frac{\log N^{Tr} -\log \epsilon}{N^{Tr}}}\right)$. By union bound, with probability $1-\epsilon$, Lemma \ref{lemma:SFAME:DiscretizeBound}, Lemma \ref{lemma:SFAME:convergencetrain2testbound}, and Lemma \ref{lemma:SFAME:discrete2testingbound}
 all hold, and we have 

\begin{equation*}
    \begin{split}
    & \frac{1}{N^{}} \sum_{n=1}^{N^{}} r^{{{s}}^{\hat{p}}}(x_n^{}) -  \frac{1}{N}\sum_{n=1}^{N} r^{{{s}}^{p^*}}(x_n) \\
    = &    \frac{1}{N^{}} \sum_{n=1}^{N^{}} r^{{{s}}^{\hat{p}}}(x_n^{}) -     \frac{1}{N^{}} \sum_{n=1}^{N} r^{{{s}}^{\hat{p}(\Omega_M)}}(x_n^{}) +  \frac{1}{N}\sum_{n=1}^{N} r^{{{s}}^{\hat{p}(\Omega_M)}}(x_n) \\
    + & \frac{1}{N}\sum_{n=1}^{N} r^{{{s}}^{p(\Omega_M)}}(x_n) -  \frac{1}{N}\sum_{n=1}^{N} r^{{{s}}^{p^*}}(x_n) \\
    \geq  & -O\left(\sqrt{\frac{log N - \log \epsilon }{ N}} + \sqrt{\frac{\log N^{Tr}  - \log \epsilon }{ N^{Tr}}}  \right)
    \end{split}
\end{equation*}
Now note that by Theorem \ref{thm:SFAME:dualbound}, we have with probability 1,  
\begin{equation*}
    \frac{1}{N} \sum_{n=1}^{N}  {r}^{{{s}}^{p^*}}(x_n) \geq \frac{1}{N} \sum_{n=1}^{N}  {r}^{s^*}(x_n) - \frac{1}{N}
\end{equation*}
Combing the above two inequalities, we have 
\begin{equation*}
    \begin{split}
    & \frac{1}{N^{}} \sum_{n=1}^{N^{}} r^{{{s}}^{\hat{p}}}(x_n^{}) -  \frac{1}{N} \sum_{n=1}^{N}  {r}^{s^*}(x_n)
    \geq  -O\left(\sqrt{\frac{log N - \log \epsilon }{ N}} + \sqrt{\frac{\log N^{Tr}  - \log \epsilon }{ N^{Tr}}}  \right)
    \end{split}
\end{equation*}
Next we consider the feasibility requirement. By Lemma \ref{lemma:SFAME:convergencefeasibility}, with probability $1-\epsilon$, ${{s}}^{\hat{p}}$ is a feasible solution to Problem \ref{prob:SFAME:optimaldefinition}.
That is to say, ${{s}}^{\hat{p}}$ with probability $1-\epsilon$ is a feasible solution.
Applying union bound  completes the proof.
\end{proof}

\newpage
\section{Experimental Details}\label{sec:SFAME:experimentdetails}
We provide missing experimental details in this part.

\paragraph{Experimental setup}
 All experiments were run on a machine with 8 Intel Xeon Platinum  2.5 GHz cores, 32 GB
RAM, and 500GB disk with Ubuntu 16.04 LTS as the OS.
Our code is implemented in Python 3.7.
Each experiments, except the case study, were run for five times to mitigate the randomness introduced by training-testing splitting. 

 \paragraph{ML tasks and services}
Recall that We focus on three multi-label classification tasks, multi-label image classification (\textit{MIC}), scene text recognition (\textit{STR}), and named entity recognition (\textit{NER}). 

\textit{MIC} is a computer vision task, where the goal is to assign a set of labels to a given image.
For \textit{MIC}, we use 3 different commercial ML cloud services, Google Vision \cite{GoogleAPI}, Microsoft Vision \cite{MicrosoftAPI}, and Everypixel\cite{EverypixelAPI}.
We also use a  single shot detector model (SSD) pretrained on OpenImageV4 \cite{OpenImagesV4IJCV2020}, which is freely available from GitHub \cite{SSD_MIC_github}.
All of those APIs produce labels from a large (and unknown) set, but  the datasets we consider have bounded number of  labels. 
For example, there are only 80 distinct labels in COCO dataset.
Thus, we remove the predicted labels which are not in the full label set.
For example, if Google API gives label \textit{\{person, car, man\}} for an image in COCO, but \textit{man} is not  in the full label set of COCO, then we will use \textit{\{person, car\}} as the label set produced from Google.

\textit{STR} is a computer vision task, where the goal is to predict all texts in a natural scene image.
In the context of multi-label classification, we view each predicted word as a label, and all possible words as the label set.  
For \textit{STR}, the ML services used in the experiments are Google Vision \cite{GoogleAPI}, iFLYTEK API \cite{IflytekAPI}, and Tencent API  \cite{TencentAPI}. 
We also use PP-OCR \cite{PaddleOCR_github}, an open source model from GitHub.

\textit{NER} is a natural language processing task where the goal is to extract all possible entities from a given text. 
For example, for the sentence \textit{ICML was held in Long Beach in 2019}, \textit{ICML} should be extracted as an organization, and \textit{Long Beach} should be identified as a location. In this paper, we consider three common types of entities, \textit{person, location}, and \textit{organization}.
For any given text, each possible entity is viewed as a label, and the label set is the number of unique entities in the entire dataset.
For \textit{NER}, we use three common APIs: Amazon Comprehend \cite{AmazonAPI}, Google NLP \cite{GoNLPAPI}, and IBM natural language API \cite{IBMNLPAPI}.
 a  multi-task convolutional neural network  model\cite{Spacy_github} from GitHub is also used.

\paragraph{Datasets.}

The experiments were conducted on 9 datasets.
For \textit{MIC}, we use three popular datasets 
including PASCAL~\cite{Dataset_Pascal_2015}, MIR~\cite{Dataset_MIR_2008} and COCO~\cite{Dataset_COCO_2014}.
PASCAL is a standard object recognition dataset with 20 distinct labels, and COCO is another one with 80 unique labels. 
PASCAL's label set contains 20 common objects: \textit{person,
bird, cat, cow, dog, horse, sheep, aeroplane, bicycle, boat, bus, car, motorbike, train, bottle, chair, dining table, potted plant, sofa, tv/monitor}.
The 80 objects in COCO include: \textit{person, bicycle, car, motorcycle, airplane, bus, train, truck, boat, traffic light, fire hydrant, stop sign, parking meter, bench, bird, cat, dog, horse, sheep, cow, elephant, bear, zebra, giraffe, backpack, umbrella, handbag, tie, suitcase, frisbee, skis, snowboard, sports ball, kite, baseball bat, baseball glove, skateboard, surfboard, tennis racket, bottle, wine glass, cup, fork, knife, spoon, bowl, banana, apple, sandwich, orange, broccoli,  carrot,  hot dog,  pizza,  donut,  cake,  chair,  couch,  potted plant,  bed,  dining table,  toilet,  tv,  laptop,  mouse,  remote,  keyboard,  cell phone,  microwave,  oven,  toaster,  sink,  refrigerator,  book,  clock,  vase,  scissors,  teddy bear,  hair drier,  toothbrush}.
For those two datasets, we use their original associated labels as the label set.
MIR is a dataset designed for image retrieval.
There are originally 25 labels:   \textit{animals,  baby,  bird,  car,  clouds,  dog,  female,  flower,  food,  indoor,  lake,  male,  night,  people,  plant\_life,  portrait,  river,  sea,  sky,  structures,  sunset,  transport,  tree,  water}.     
We remove the label \textit{night} since it is not in the label set of any of the APIs or the GitHub model.
On average, there are 1.44 labels per image for PASCAL, 3.71 labels per image for MIR, and 2.91 labels per image for COCO.
The dataset statistic is summarized in Table \ref{tab:SFAME:DatasetStats}. 
Most of the datasets are open and under Creative Commons license (e.g., the dataset COCO \cite{Dataset_COCO_2014}).
The details can be found in their corresponding paper and repository. 
As those datasets are actually open, they do not require an in-person consent from the authors/developers. 
The datasets themselves may contain personal information (e.g., there are personal images in COCO). 
Though, they have been render anonymous.
For the purpose of deciding which API to call, we also do not use personally identifiable  information.

For \textit{STR}, we use three large scale Chinese text recognition datasets, 
MTWI~\cite{Dataset_MTWI_2018}, ReCTS~\cite{Dataset_ReCTS_2019} and LSVT~\cite{Dataset_LSVT_2019}.
The label set contains all possible Chinese characters as well as digits (0-9).
MTWI contains images from the internet mainly targeting at advertisements. 
Thus, most of its images have dense texts.
ReCTS includes photos taken on sign boards and thus has relatively fewer words.
The images from LSVT are typically street view images and hence have medium number of words.
All images in MTWI and ReCTS are fully annotated and used in our experiments.
LSVT contains both fully and partially annotated images, and we only use the subset with full annotations. 


The other datasets,  CONLL~\cite{Dataset_CONLL_2003}, ZHNER~\cite{ZHNER_dataset_github} and GMB~\cite{Dataset_GMB_2013}, are used for \textit{NER} task.
CONLL contains English sentences from newspapers, and texts from GMB are also English and from a wider range of sources. On the other hand, ZHNER is a Chinese text dataset.
We consider four common types of entities: organization, person, and location.
In this paper, we focus on three common types of entities that all datasets contain: \textit{persons},\textit{locations}, and \textit{organizations}. 
Each sentence from those datasets is extracted as a data point, and the associated label set is simply all entities in this sentence. 
An entity is considered correctly extracted if and only if it is labeled as an entity and its entity type is correct. 

\paragraph{GitHub model cost} We evaluate the inference time of all GitHub models on an Amazon EC2 p2.x instance, which is \$0.90 per hour. For multi-label image classification, the GitHub model \cite{SSD_MIC_github} takes 6s to classify each image, resulting in an equivalent cost of \$0.0015 per image.
For the named entity task, the GitHub model \cite{Spacy_github} can extract the entities from a sentence in 0.015s, leading to \$ 0.00000375 per sentence.
The GitHub model \cite{PaddleOCR_github} with the mobile version 3.0 text detector and recognizer requires 
1.5 on average to extract text from an image, causing a cost of \$ 0.000375 per image.
Compared to the commercial APIs, this cost is much cheaper.

\paragraph{Case study on COCO}
Now we provide more details about the case study on the multi-label image classification dataset, COCO.
\begin{figure}[t]
	\centering
	\includegraphics[angle=90,width=1.0\linewidth]{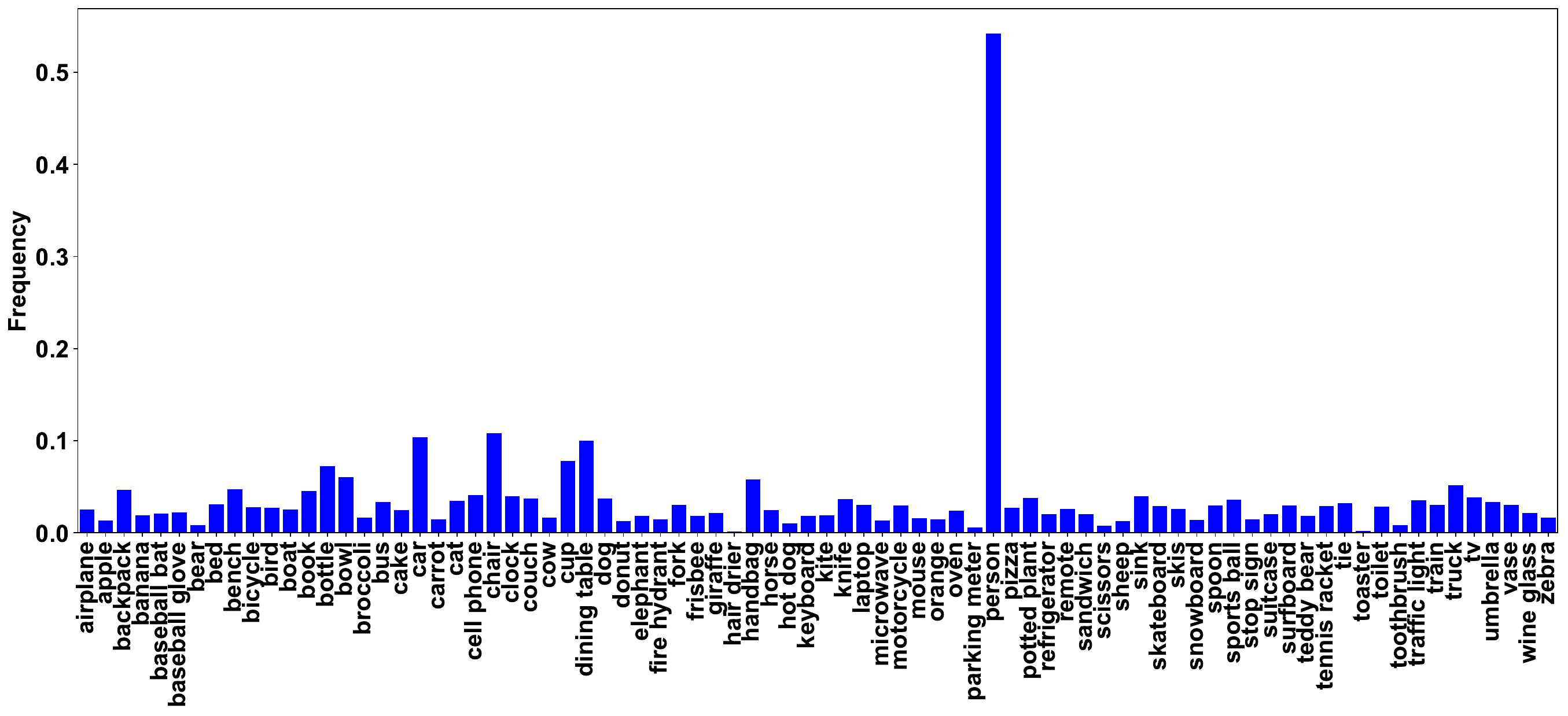}
	\caption{Label Distribution on COCO.}
	\label{fig:SFAME:COCOLabelDist}
\end{figure}
There are in total 123,287 images containing labels from 80 different categories in COCO.
Figure \ref{fig:SFAME:COCOLabelDist} gives the label distribution.
First note that the label distribution is quite skewed.
overall, the label person is the most frequent: more than 50\% of the images contain the person label.
Among others, car, chair, and dining tables are also quite common labels in this dataset with more than 10\% occurrence.
On the other hand, there are also quite some rare labels. 
For example, half driver and toaster appear in less than 1\% of the images. 
Such imbalance between different labels imposes a high data and computational complexity to directly apply previous approach that learns a decision rule per label, and thus verifies the necessity of the proposed framework,  \systemnamesecond{}.

\begin{figure}[t] \centering
\begin{subfigure}[GitHub]{\label{fig:SFAME:cocoperclassGitHub}\includegraphics[angle=90,width=0.85\linewidth]{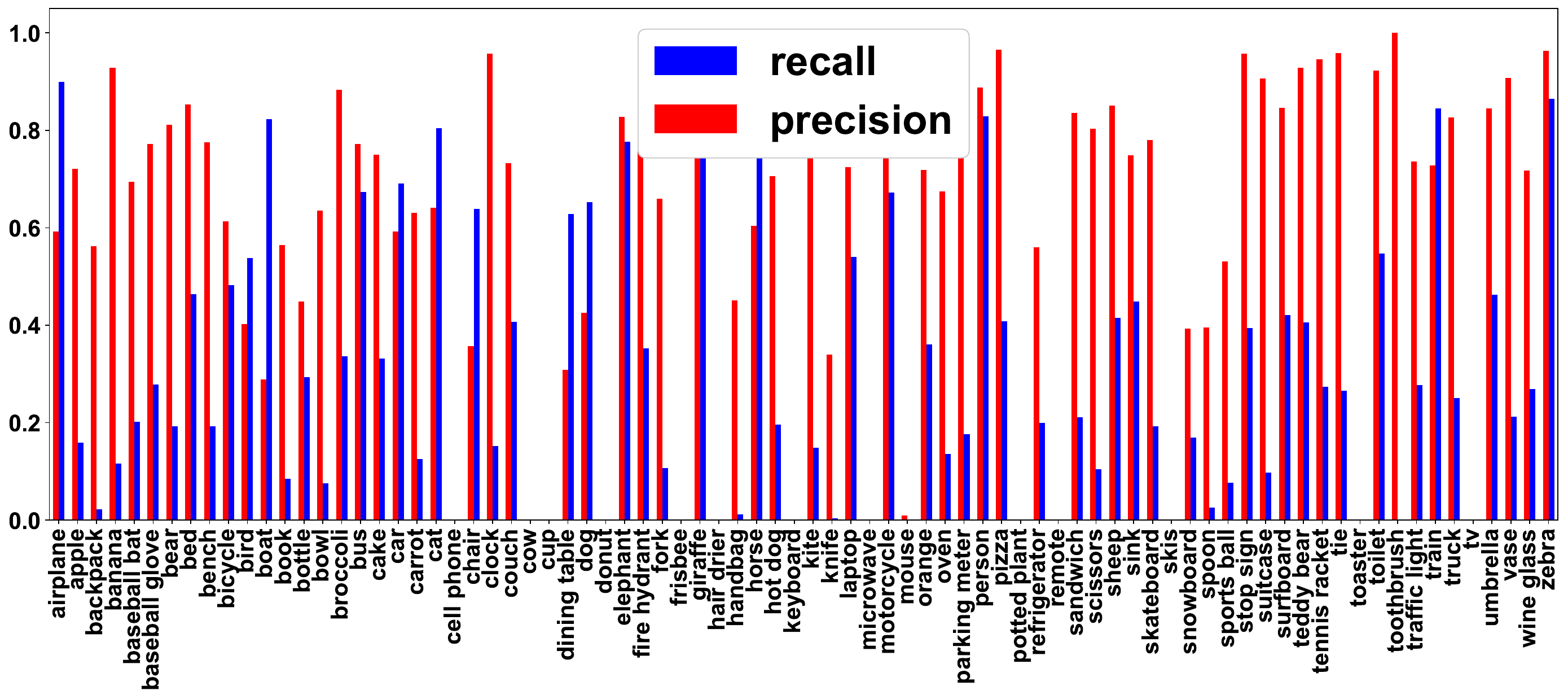}}
\end{subfigure}
\begin{subfigure}[Everypixel]{\label{fig:SFAME:cocoperclassEP}\includegraphics[angle=90,width=0.85\linewidth]{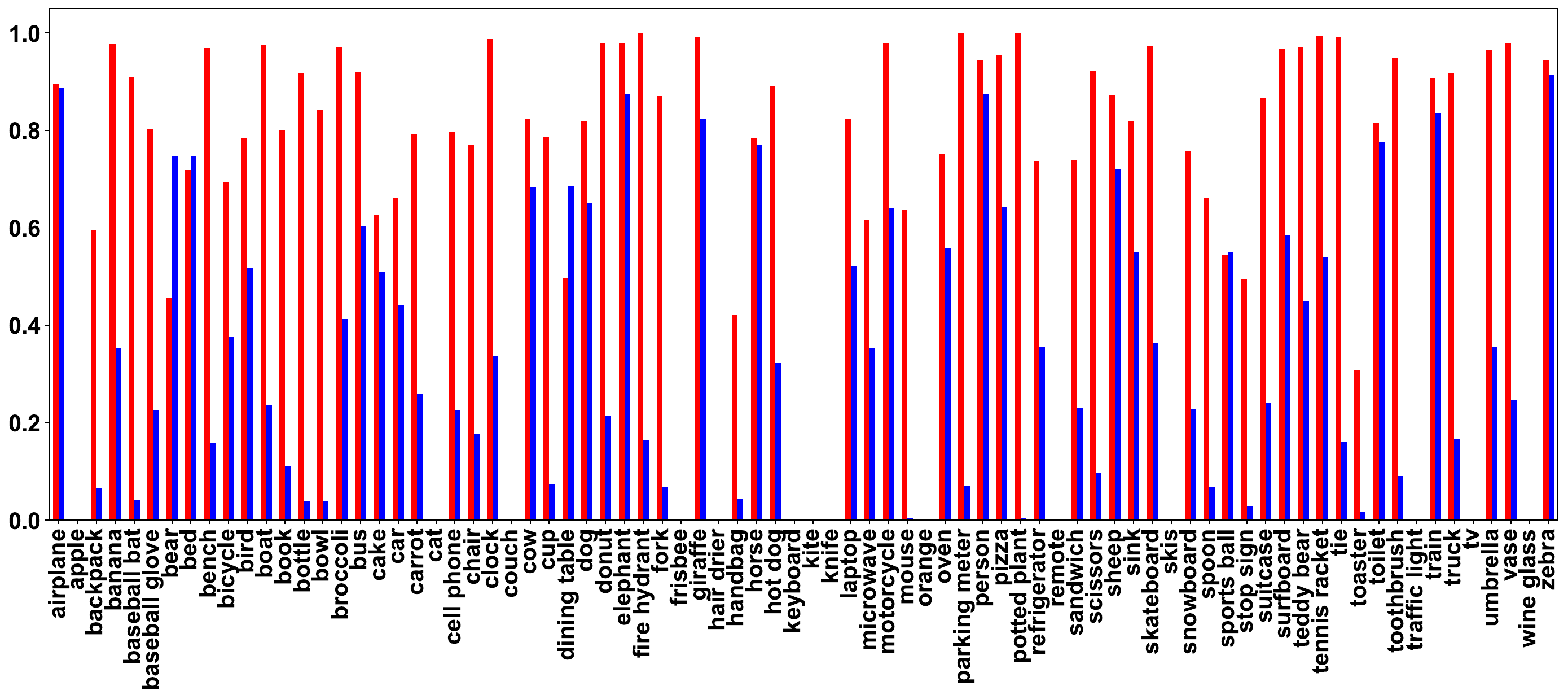}}
\end{subfigure}
\begin{subfigure}[Microsoft]{\label{fig:SFAME:cocoperclassMS}\includegraphics[angle=90,width=0.85\linewidth]{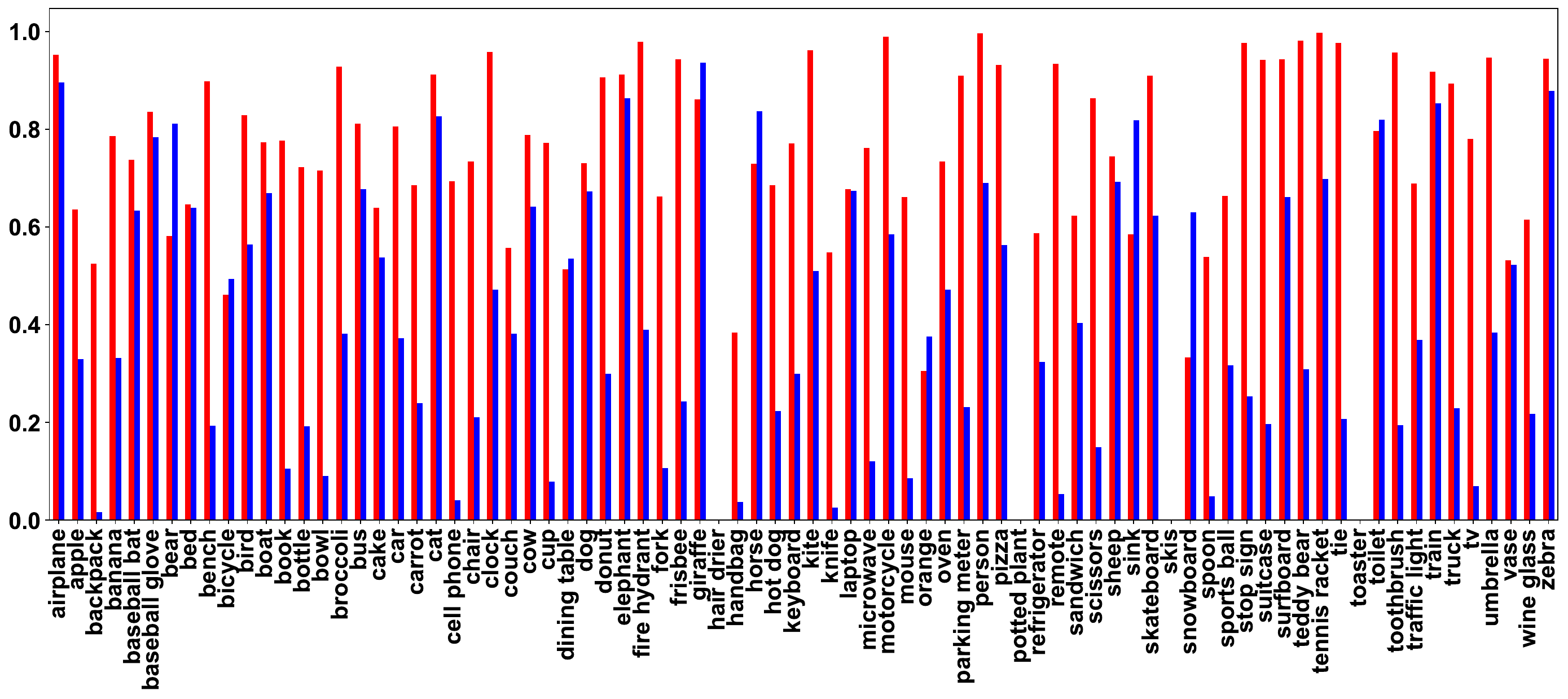}}
\end{subfigure}

	\caption{{The per class precision and recall of different APIs .}}\label{fig:SFAME:precisionrecall1}
\end{figure}

\begin{figure}[t] \centering
\begin{subfigure}[Google]{\label{fig:SFAME:cocoperclassGoogle}\includegraphics[angle=90,width=0.85\linewidth]{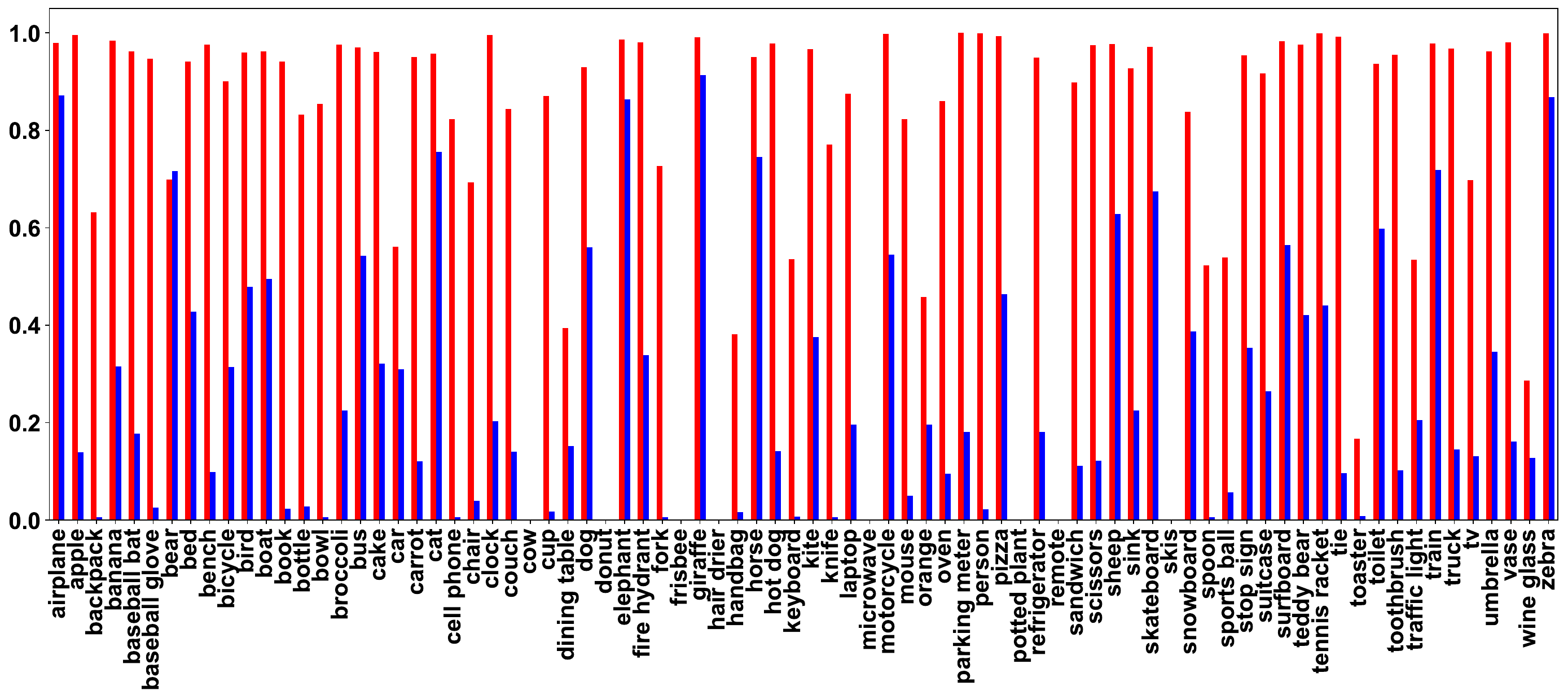}}
\end{subfigure}
\begin{subfigure}[Majority Vote]{\label{fig:SFAME:cocoperclassMajority}\includegraphics[angle=90,width=0.85\linewidth]{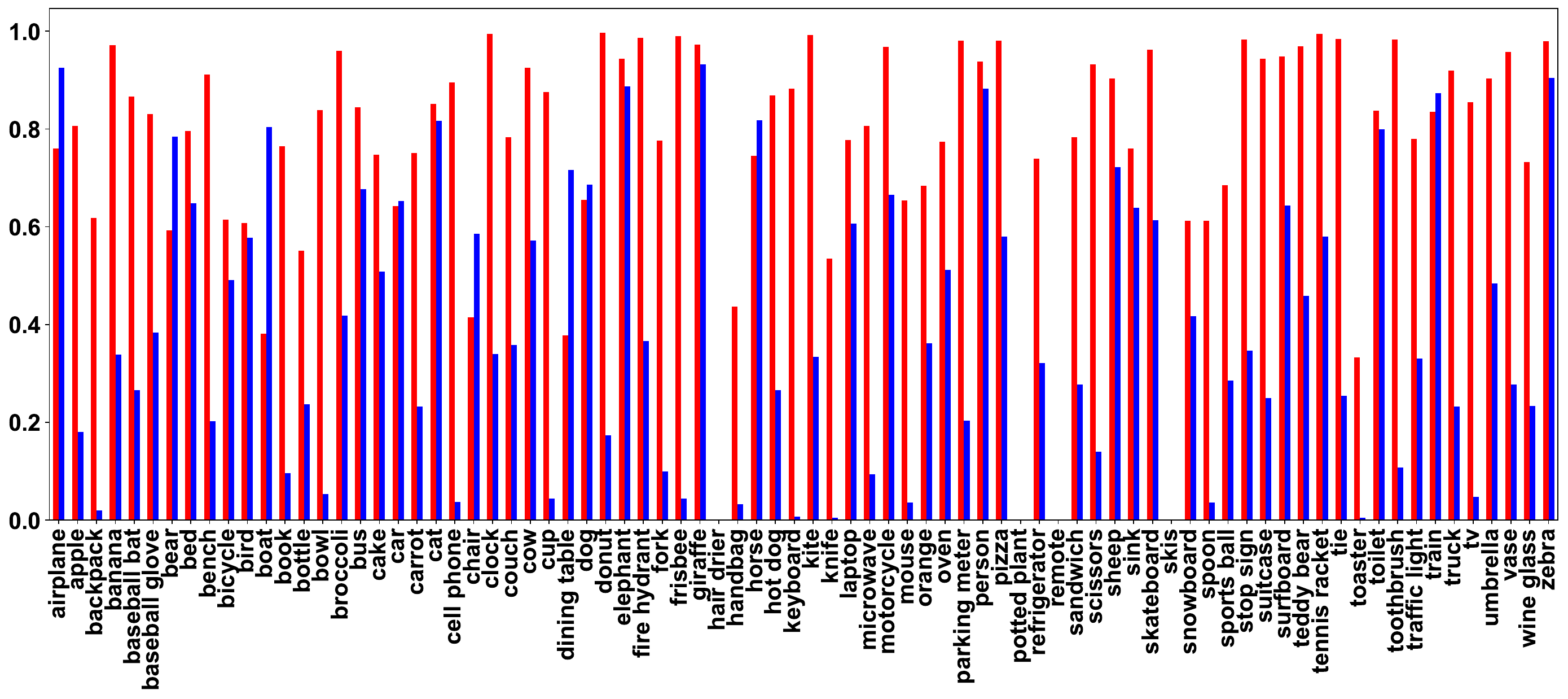}}
\end{subfigure}
\begin{subfigure}[\systemnamesecond{}]{\label{fig:SFAME:cocoperclassFrugalMCT}\includegraphics[angle=90,width=0.85\linewidth]{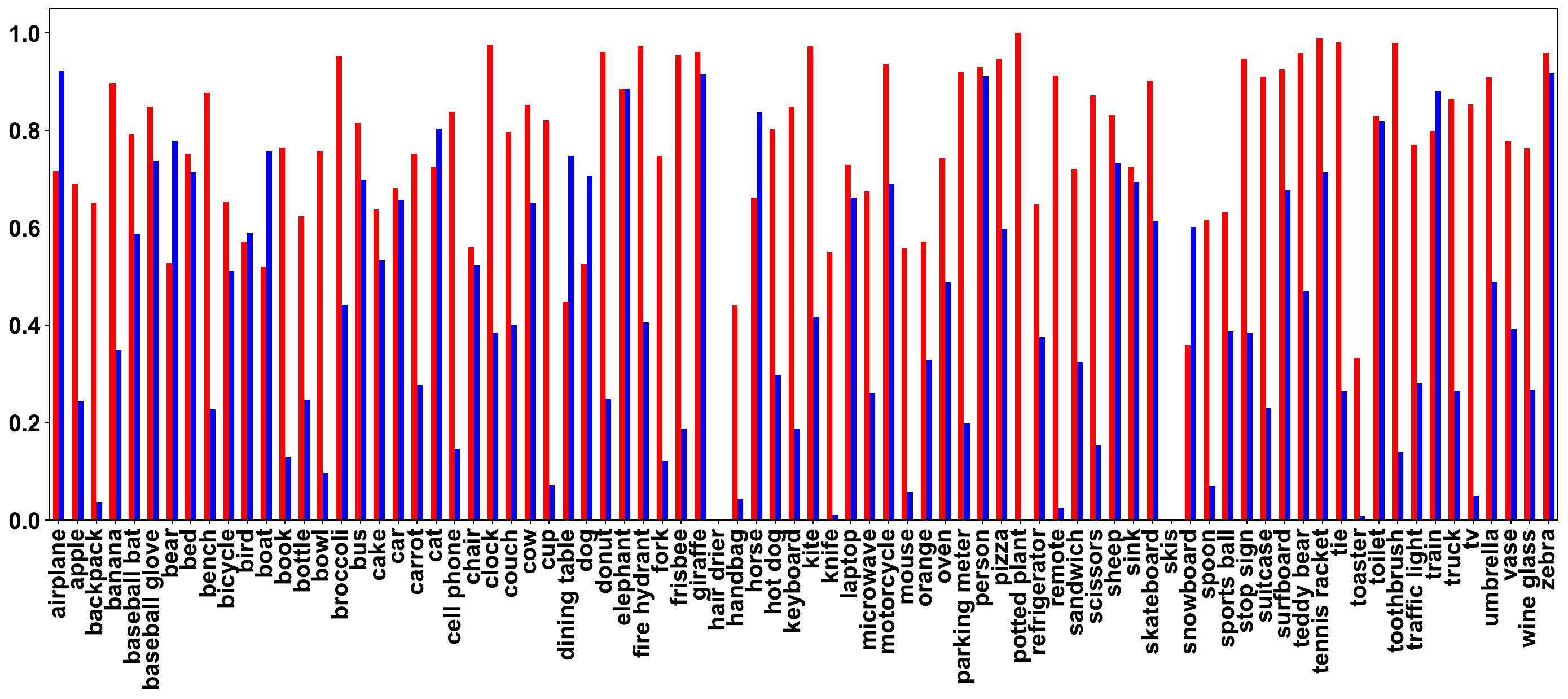}}
\end{subfigure}

	\caption{{The per class precision and recall of different APIs .}}\label{fig:SFAME:precisionrecall2}
\end{figure}

To further understand when and why  \systemnamesecond{} gives a better performance than single API, we present the precision and recall per class for each API, majority vote, and \systemnamesecond{} in Figure \ref{fig:SFAME:precisionrecall1} and Figure \ref{fig:SFAME:precisionrecall2}.
We first note that there is no API universally better than other APIs for each label.
For example, GitHub and Microsoft APIs can hardly correctly predict the label ``toaster'',  but Everypixel and Google APIs have a relatively high accuracy on label ``toaster''.
On the other hand, Everypixel has a low accuracy on label ``kite'' and ``knite'', while Microsoft, Google, GitHub APIs can usually predict those labels with higher accuracy.
This implies that combining different APIs may produce an accuracy better than any single one of them.
There are also some easy labels on which all APIs give a high accuracy.
For example, on the label ``zebra'', all APIs give a 90\% precision and recall.
This actually suggests that it is not always necessary to use all API. 
For example, if GitHub predicts an image has the label ``zebra'', and we know there is no other labels in this image, then probably there is no need to call any other APIs.

Another interesting observation is that \systemnamesecond{} improves the precision and recall for almost every label compared to any single API.
This is primarily because \systemnamesecond{} appropriately utilizes the predicted label information from GitHub model to infer which API is better on certain input, and combine its performance with the base API aptly.
Yet, the precision and recall difference can be quite different for different APIs.
For example, as shown in Figure \ref{fig:SFAME:cocoperclassFrugalMCT}, the recall for ``airplane'' is much higher than its precision, but banana's precision is much higher than its recall. 
For applications that have specific precision and recall requirements, we may adopt different accuracy metrics in \systemnamesecond{}. 
Another interesting observation is that the precision and recall for some labels is extremely.
For example, ``hair drier'' cannot be predicted by \systemnamesecond{}, which is due to that no API actually predicts this label correctly.
How to extend \systemnamesecond{} to recognize unseen labels remains an open question.

\paragraph{Ensemble method comparison}
\eat{
\begin{table}[t]
  \centering
  \caption{Performance of \systemnamesecond{}'s accuracy predictor. Root mean square error (RMSE) quantifies the standard deviation of the differences between predicted and true accuracy. 
  }
    \begin{tabular}{|c|c|c|c|c|c|}
    \hline
    Data  & RMSE  & Data  & RMSE  & Data  & RMSE \bigstrut\\
    \hline
    \hline
    PASCAL & 0.28  & MIR   & 0.22  & COCO  & 0.24 \bigstrut\\
    \hline
    MTWI  & 0.17  & ReCTS & 0.22  & LSVT  & 0.19 \bigstrut\\
    \hline
    CONLL & 0.29  & ZHNER & 0.31  & GMB   & 0.28 \bigstrut\\
    \hline
    \end{tabular}%
  \label{tab:SFAME:AccPredMeasure}%
\end{table}%
}

\begin{table}[t]
  \centering
  \caption{Comparison of ensemble methods as well as cost-aware approaches. For \systemnamesecond{} and \systemname{}, we pick their corresponding strategies that minimize the cost while ensures that the accuracy reaches the highest possible.  }
    \begin{tabular}{|c||c|c|c|c|c|c|c|c|c|c|}
    \hline
          & \multicolumn{2}{c|}{best single API} & \multicolumn{2}{c|}{FrugalML} & \multicolumn{2}{c|}{FrugalMCT} & \multicolumn{2}{c|}{majority vote} & \multicolumn{2}{c|}{weighted maj vote} \bigstrut\\
    \hline
          & acc   & cost  & acc   & cost  & acc   & cost  & acc   & cost  & acc   & cost \bigstrut\\
    \hline
    \hline
    PASCAL & 74.8  & 10    & 76.9  & 11    & 78.5  & 8     & 77.8  & 31.01 & 77.8  & 31.01 \bigstrut\\
    \hline
    MIR   & 41.2  & 10    & 43.8  & 8     & 49.2  & 14    & 41.4  & 31.01 & 48.7  & 31.01 \bigstrut\\
    \hline
    COCO  & 47.5  & 10    & 49.3  & 8     & 54    & 12    & 50.1  & 31.01 & 52.8  & 31.01 \bigstrut\\
    \hline
    MTWI  & 67.9  & 210   & 68.1  & 213   & 71.1  & 208   & 75.4  & 275.01 & 75.4  & 275.01 \bigstrut\\
    \hline
    ReCTS & 61.3  & 210   & 63.4  & 213   & 64.7  & 208   & 70.2  & 275.01 & 70.2  & 275.01 \bigstrut\\
    \hline
    LSVT  & 53.8  & 210   & 56.2  & 213   & 57.2  & 208   & 62.8  & 275.01 & 62.8  & 275.01 \bigstrut\\
    \hline
    CONLL & 52.6  & 3     & 55.7  & 32    & 56.8  & 36.8  & 58.5  & 43.01 & 58.5  & 43.01 \bigstrut\\
    \hline
    ZHNER & 61.3  & 30    & 67.4  & 31.2  & 71.8  & 36.8  & 66    & 43.01 & 66    & 43.01 \bigstrut\\
    \hline
    GMB   & 50.1  & 30    & 52.6  & 30.1  & 53.1  & 20.5  & 51.3  & 43.01 & 51.5  & 43.01 \bigstrut\\
    \hline
    \end{tabular}%
  \label{tab:SFAME:ensemblecompare}%
\end{table}%

\begin{table}[h]
  \centering
  \caption{Accuracy predictor performance. RMSE and PCC stand for root mean square error and Pearson’s correlation coefficient.}
    \begin{tabular}{|c||c|c|c|c|}
    \hline
    \multirow{2}[4]{*}{Data} & \multicolumn{2}{c|}{RMSE} & \multicolumn{2}{c|}{PCC} \bigstrut\\
\cline{2-5}          & \systemnamesecond{} & DAP   & \systemnamesecond{} & DAP \bigstrut\\
    \hline
    \hline
    PASCAL & 0.28  & 0.35  & 0.55  & 0.012 \bigstrut\\
    \hline
    MIR   & 0.22  & 0.31   & 0.55  & -0.013 \bigstrut\\
    \hline
    COCO  & 0.24  & 0.31  & 0.63  & 0.001 \bigstrut\\
    \hline
    MTWI  & 0.17  & 0.21  & 0.57  & 0.004 \bigstrut\\
    \hline
    ReCTS & 0.22  & 0.27  & 0.57  & 0.001 \bigstrut\\
    \hline
    LSVT  & 0.19  & 0.24  & 0.61  & -0.003 \bigstrut\\
    \hline
    CONLL & 0.29  & 0.41  & 0.72  & -0.003 \bigstrut\\
    \hline
    ZHNER & 0.31  & 0.36  & 0.48  & -0.005 \bigstrut\\
    \hline
    GMB   & 0.28  & 0.40   & 0.69  & -0.006 \bigstrut\\
    \hline
    \end{tabular}%
  \label{tab:SFAME:AccPredMeasureFull}%
\end{table}%
For comparison, we compare \systemnamesecond{} against \systemname{} as well as two ensemble methods, majority vote and weighted majority vote.
In majority vote, for each label, we accept it if at least half of the APIs predict it.
In weighted majority vote, we assign each API's accuracy as its weight. 
Next, for each label,  we compute a label score, which is equal to the sum of each API's  confidence score on this label weighted by its corresponding weight.
If an API does not predict a label, then its confidence score is viewed as 0.
Finally, we only accept the label if its label score is larger than a threshold.
We pick a threshold that maximizes the overall accuracy by grid search.

The results are summarized in Table \ref{tab:SFAME:ensemblecompare}.
Overall, we observe that \systemnamesecond{} and ensemble methods have similar performance across different tasks and datasets, but with a much lower cost.
In fact, for datasets including COCO and ZHNER, \systemnamesecond{} can achieve an accuracy enven higher than ensemble methods.

\paragraph{Accuracy predictor performance} 
Note that \systemnamesecond{}'s performance highly depends on its accuracy predictors' performance.
Tn obtain a quantitative sense of the accuracy predictors, we evaluate the accuracy predictors' performance in Table \ref{tab:SFAME:AccPredMeasureFull}.
RMSE measures the standard deviation of the difference between accuracy predictor's output and the corresponding true accuracy.
PCC stands for Pearson correlation coefficient, which roughly measures the linear correlation between the true accuracy and the predicted value from the accuracy predictors. 
Overall, \systemnamesecond{}' random forest predictors enjoy a much smaller RMSE and higher PCC than DAP (the dummy accuracy predictors), which matches the fact that \systemnamesecond{} gives a higher end to end performance than using the DAP.

\end{document}